\documentclass{article}
\PassOptionsToPackage{sort&compress,numbers}{natbib}
\usepackage[final]{neurips_2021} 
\usepackage[utf8]{inputenc}
\usepackage[T1]{fontenc}
\usepackage{graphics,graphicx,caption,float,subcaption,booktabs,xcolor,multirow,array,color,ifthen,tabu,colortbl,dblfloatfix,url,xparse,mathtools,patchcmd,algorithm,algorithmic,amssymb,xspace,nicefrac,microtype,amsmath,amsthm,amsfonts,bm,ragged2e,tikz,stackengine,etoolbox,xpatch,enumerate,xstring,setspace,tabularx,makecell,changepage,cuted,titlesec,enumitem,wrapfig,tcolorbox}
\usepackage[pagebackref=true,breaklinks=true,colorlinks=true,bookmarks=false,citecolor=blue]{hyperref}
\usepackage[capitalise,noabbrev,nameinlink]{cleveref}
\usepackage[english]{babel}
\usepackage{listings}
\usepackage[export]{adjustbox}

\DeclareMathOperator*{\argmax}{arg\,max}
\DeclareMathOperator*{\argmin}{arg\,min}
\newtheorem{lemma}{Lemma}
\newtheorem{theorem}{Theorem}

\usepackage{etoolbox}
\makeatletter
\AfterEndEnvironment{algorithm}{\let\@algcomment\relax}
\AtEndEnvironment{algorithm}{\kern2pt\hrule\relax\vskip3pt\@algcomment}
\let\@algcomment\relax
\newcommand\algcomment[1]{\def\@algcomment{\footnotesize#1}}
\renewcommand\fs@ruled{\def\@fs@cfont{\bfseries}\let\@fs@capt\floatc@ruled
  \def\@fs@pre{\hrule height.8pt depth0pt \kern2pt}%
  \def\@fs@post{}%
  \def\@fs@mid{\kern2pt\hrule\kern2pt}%
  \let\@fs@iftopcapt\iftrue}
\makeatother

\title{Functional Regularization for Reinforcement Learning via Learned Fourier Features}

\author{%
  Alexander C. Li \\
  Carnegie Mellon University\\
  \texttt{alexanderli@cmu.edu} \\
   \And
   Deepak Pathak \\
   Carnegie Mellon University \\
   \texttt{dpathak@cs.cmu.edu} \\
}

\begin{document}

\maketitle

\begin{abstract}
We propose a simple architecture for deep reinforcement learning by embedding inputs into a learned Fourier basis and show that it improves the sample efficiency of both state-based and image-based RL. We perform infinite-width analysis of our architecture using the Neural Tangent Kernel and theoretically show that tuning the initial variance of the Fourier basis is equivalent to \textit{functional} regularization of the learned deep network. That is, these learned Fourier features allow for adjusting the degree to which networks underfit or overfit different frequencies in the training data, and hence provide a controlled mechanism to improve the stability and performance of RL optimization. Empirically, this allows us to prioritize learning low-frequency functions and speed up learning by reducing networks' susceptibility to noise in the optimization process, such as during Bellman updates. Experiments on standard state-based and image-based RL benchmarks show clear benefits of our architecture over the baselines\footnote{Code available at \href{https://github.com/alexlioralexli/learned-fourier-features}{https://github.com/alexlioralexli/learned-fourier-features}}.
\end{abstract}

\section{Introduction}
\label{introduction}
Most popular deep reinforcement learning (RL) approaches estimate either a value or Q-value function under the agent's learned policy. These functions map points in the state or state-action space to expected returns, and provide crucial information that is used for improving the policy. However, optimizing these functions can be difficult, since there are no ground-truth labels to predict. Instead, they are trained through bootstrapping: the networks are updated towards target values calculated with the same networks being optimized. These updates introduce noise that accumulates over repeated iterations of bootstrapping, which can result in highly inaccurate value or Q-value estimates \citep{thrun1993issues,tsitsiklis1997analysis}. As a result, these RL algorithms may suffer from lower asymptotic performance or sample efficiency.

Most prior work has focused on making the estimation of target values more accurate. 
Some examples include double Q-learning for unbiased target values~\citep{hasselt2010double,vanhasselt2015deep}, or reducing the reliance on the bootstrapped Q-values by calculating multi-step returns with TD($\lambda$) \citep{sutton2018reinforcement}.
However, it is impossible to hope that the noise in the target values estimated via bootstrapping will go to zero, because we cannot estimate the true expectation over infinite rollouts in practical setups. Hence, we argue that it is equally important to also regularize the function (in this case, the deep Q-network) that is fitting these noisy target values.

Conventional regularization methods in supervised learning can be associated with drawbacks in RL.
Stochastic methods like dropout \citep{srivastava2014dropout} introduce more noise into the process, which is tolerable when ground-truth labels are present (in supervised learning) but counterproductive when bootstrapping. An alternative approach is early stopping \citep{yao2007early}, which hurts sample efficiency in reinforcement learning because a minimum number of gradient steps is required to propagate value information backwards to early states. Finally, penalty-based methods like $L_1$ \citep{tibshirani1996regression} and $L_2$ regularization \citep{hoerl1970ridge,krogh1992simple} can help in RL \citep{liu2019regularization}, but regularizing the network in weight space does not disentangle noise from reward signal and could make it difficult to learn the true Q-function. This leads us to ask: what is the right way to regularize the RL bootstrapping process?

We suggest that the impact of target value noise can be better reduced by \textit{frequency-based functional regularization}: direct control over the frequencies that the network tends to learn first. If the target noise consists of higher frequencies than the true target values, discouraging high-frequency learning can help networks efficiently learn the underlying Q-function while fitting minimal amounts of noise.
In this work, we propose an architecture that achieves this by encoding the inputs with learned Fourier features, which we abbreviate as LFF.
In contrast to using fixed Fourier features \citep{rahimi2007random,tancik2020fourier}, we train the Fourier features, which helps them find an appropriate basis even in high dimensional domains.
We analyze our architecture using the Neural Tangent Kernel \citep{jacot2018neural} and theoretically show that tuning the initial variance of the Fourier basis controls the rate at which networks fit different frequencies in the training data. Thus, LFF's initial variance provides a controlled mechanism to improve the stability and performance of RL optimization (see Figure~\ref{fig:architecture}).
Tuned to prioritize learning low frequencies, LFF filters out bootstrapping noise while learning the underlying Q-function.

We evaluate LFF, which only requires changing a few lines of code, on state-space and image-space DeepMind Control Suite environments~\citep{tassa2018deepmind}. We find that LFF produces moderate gains in sample efficiency on state-based RL and dramatic gains on image-based RL. 
In addition, we empirically demonstrate that LFF makes the value function bootstrapping stable even in absence of target networks, and confirm that most of LFF's benefit comes through regularizing the Q-network. Finally, we provide a thorough ablation of our architectural design choices.

\begin{figure}
    \centering
    \includegraphics[width=0.47\linewidth]{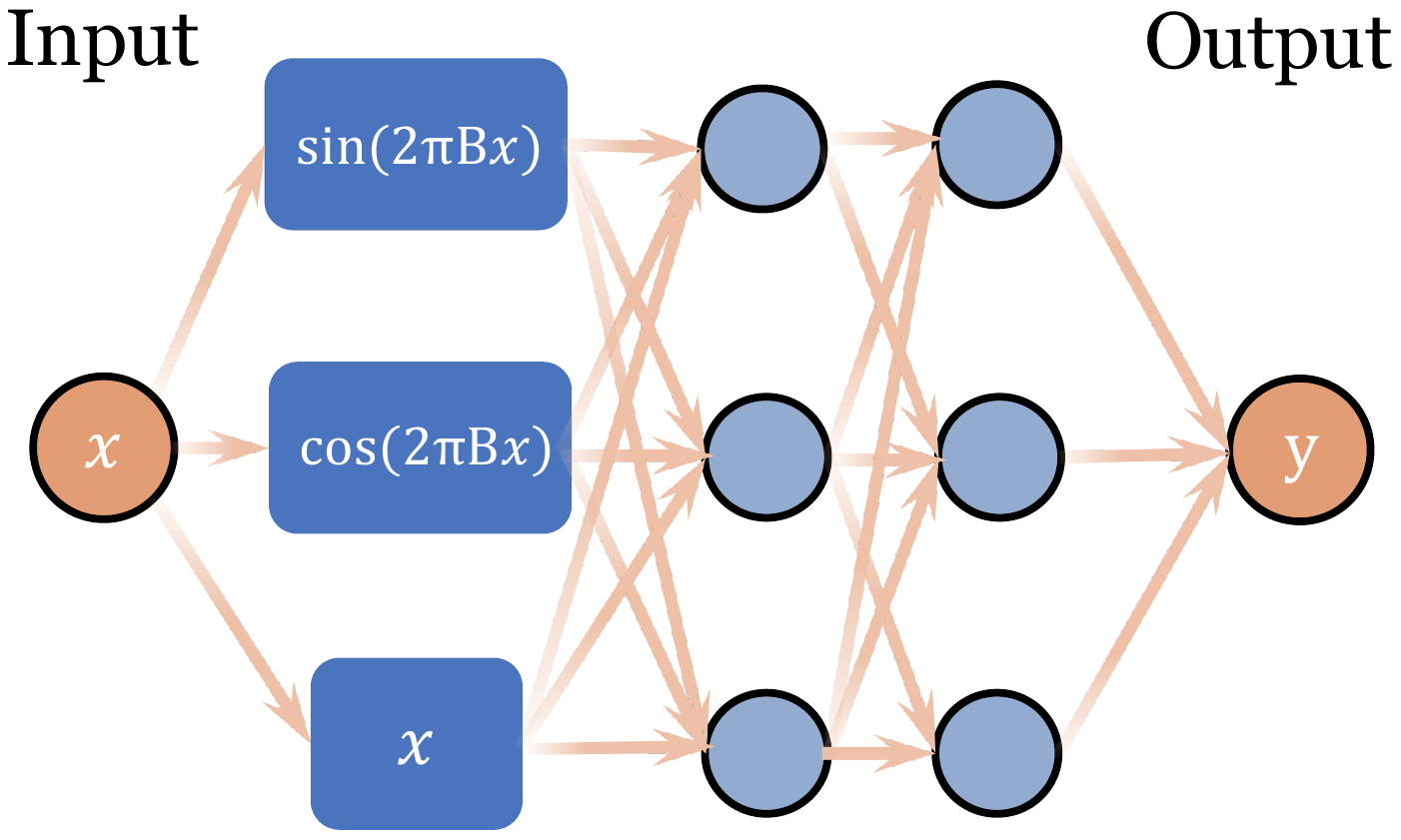}
    \hfill
    \includegraphics[width=0.47\linewidth]{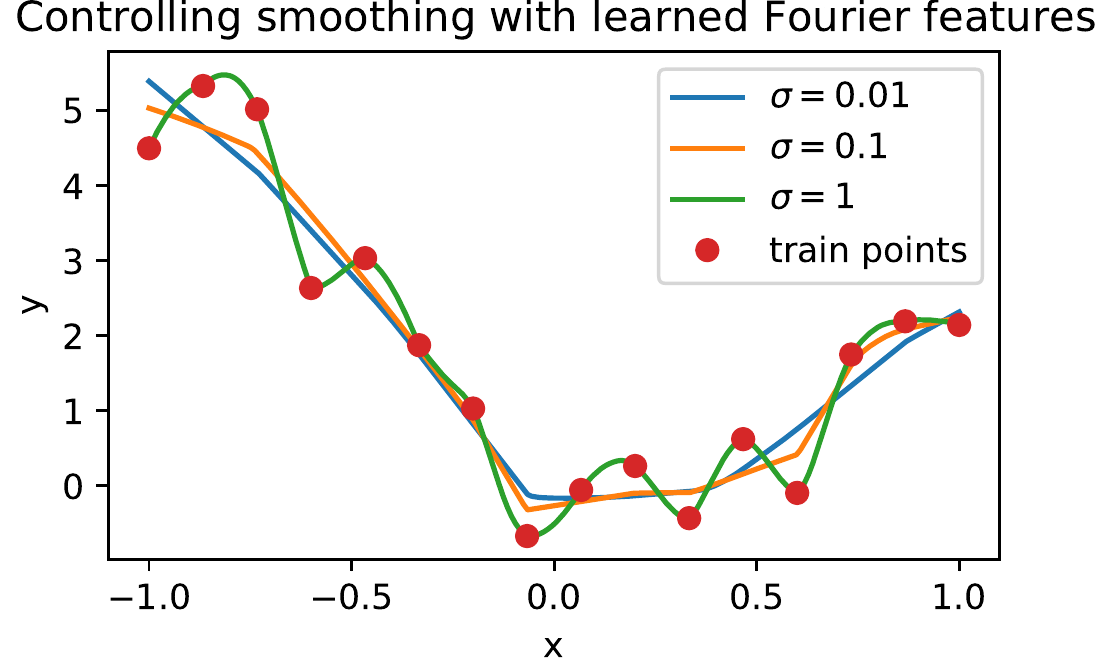}
    \caption{\textbf{Left:} the proposed learned Fourier feature (LFF) architecture. $B$ is a matrix trained through backpropagation, and is used to create a learned set of Fourier features. The network then passes the Fourier features through alternating linear layers and ReLU nonlinearities, as is done in vanilla MLPs. \textbf{Right:} tuning the initialization variance of $B_{ij} \sim \mathcal N(0, \sigma^2)$ controls the rate at which target frequencies are learned. Higher $\sigma$ fits higher frequencies faster, while lower $\sigma$ smooths out noise. This architecture can be used as functional regularization for a Q-function $Q: \mathcal S \times \mathcal A \rightarrow \mathbb R$, value function $V: \mathcal S \rightarrow \mathbb R$, policy $\pi: \mathcal S \rightarrow \mathbb{R}^{\dim(\mathcal A)}$, or model $T: \mathcal S \times \mathcal A \rightarrow \mathcal S$.}
    \label{fig:architecture}
\end{figure}

\section{Preliminaries}
The reinforcement learning objective is to solve a Markov Decision Process (MDP), which is defined as a tuple $(\mathcal S, \mathcal A, P, R, \gamma)$. $\mathcal S$ and $\mathcal A$ denote the state and action spaces. $P(s'|s, a)$ is the transition function, $R(s, a)$ is the reward function, and $\gamma \in [0, 1]$ is the discount factor. We aim to find an optimal policy $\pi^*(a|s)$ that maximizes the expected sum of discounted rewards.
Q-learning finds the optimal policy by first learning a function $Q^*(s, a)$ such that 
    $Q^*(s, a) = \mathbb{E}_{s' \sim P(\cdot|s, a)}[R(s, a) + \gamma \max_{a'}Q^*(s', a')]$.
The optimal policy is then $\pi^*(a|s) = \argmax_a Q^*(s, a)$.
To find $Q^*$, we repeatedly perform Bellman updates to $Q$, which uses the Q-network itself to bootstrap target values on observed transitions $(s, a, r, s')$ \citep{bellman1966dynamic}. The most basic way to estimate the target values is $target = r + \gamma \max_{a'} Q(s', a')$, but a popular line of work in RL aims to find more accurate target value estimation methods~\citep{vanhasselt2015deep,hasselt2010double,sutton2018reinforcement,fujimoto2018addressing,precup2001off}. Once we have these target value estimates, we update our Q-network's parameters $\theta$ via gradient descent on temporal difference error: 
    $\Delta \theta = - \eta \nabla_\theta \left(Q_\theta(s, a) - target\right)^2$.
Our focus is on using the LFF architecture to prevent Q-networks from fitting irreducible, high-frequency noise in the target values during these updates.


\section{Reinforcement Learning with Learned Fourier Features}
\label{sec:fouriermlp}
\begin{figure}
    \centering
    \includegraphics[width=\textwidth]{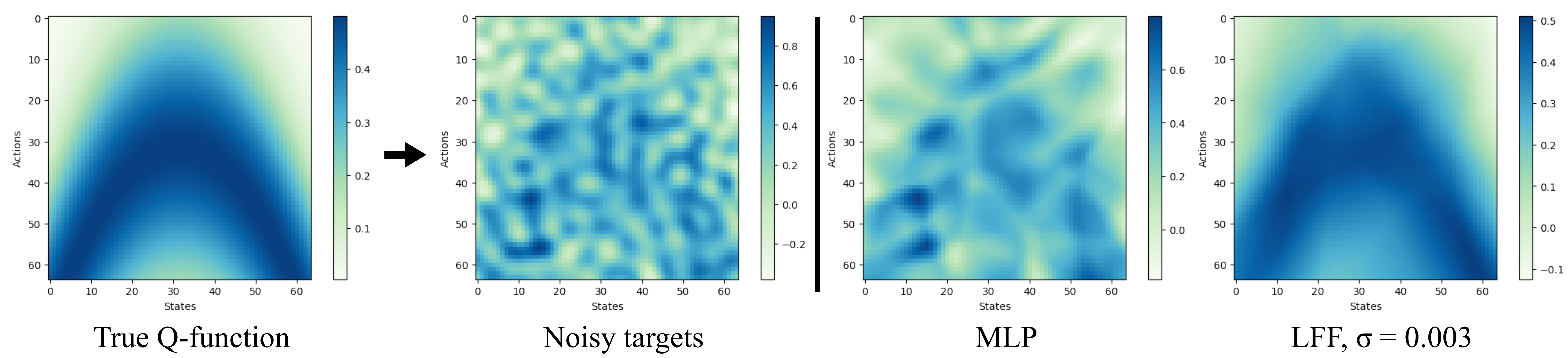}
    \caption{\textbf{Filtering noise with LFF input embedding.} The bootstrapped targets in RL are a mix between signal and noise. 
    Right side: we fit different networks to these noisy targets and display their predictions.
    While the MLP overfits, LFF learns the Q-function and ignores almost all of the noise.}
    \label{fig:toy_qf_noise}
    \vspace{-1em}
\end{figure}
We present a visualization of the noisy target value problem in Figure~\ref{fig:toy_qf_noise}. The target value estimates can be noisy due to stochastic transitions, replay buffer sampling, or unintended generalization across state-action pairs due to function approximation ~\citep{van2018deep}. We simulate this in Figure~\ref{fig:toy_qf_noise} by adding noise to the optimal Q-function of a small gridworld. 
MLPs are susceptible to fitting the noise, resulting in inaccurate Q-values that could diverge after repeated bootstrapping. In contrast, our LFF architecture controls how quickly low- and high-frequency signals are learned. Tuned properly, LFF filters out the noise and learns the ground truth Q-function almost perfectly (Figure~\ref{fig:toy_qf_noise}, right). 

The problem is that MLPs provide no control over how quickly they learn signals of different frequencies. Prior work in computer vision found this a problem when MLPs blurred desired high frequencies in low-dimensional (3-5 dimensions) graphics regression problems \citep{mildenhall2020nerf, tancik2020fourier}. They fixed this blurring problem by transforming the input using a random Fourier feature embedding of the input \citep{rahimi2007random,tancik2020fourier}. Specifically, the idea is to map a low-dimensional input $x$ to an embedding $\gamma(x) = \sin(2 \pi Bx) ||\cos(2 \pi Bx)$, where $B$ is a $d_{\text{fourier}} / 2 \times d_{\text{input}}$ matrix, $||$ denotes concatenation, and $\sin$ and $\cos$ act elementwise. The embedding $\gamma(x)$ directly provides a mix of low- and high-frequnecy functions a MLP $f_\theta$ can use to learn a desired function. \citet{tancik2020fourier} use this to improve fidelity in coordinate-based graphics problems. 
    
Intuitively, the row vectors $b_i$ are responsible for capturing desired frequencies of the data. If they capture only low frequencies, then the MLP will be biased towards only learning the low-frequency signals in the data. Conversely, if the Fourier features capture sufficient high-frequency features, then a MLP can fit high frequency functions by computing simple nonlinear combinations of the features. 
In these low-dimensional graphics problems, 
initializing fixed entries $B_{ij} \sim \mathcal N(0, \sigma^2)$ with large $\sigma$ was enough to learn desired high frequency functions; training $B$ did not improve performance.

In the following section, we propose our learned Fourier feature architecture for deep RL, which allows practitioners to \textit{tune the range of frequencies that the network should be biased towards learning}. We propose several key enhancements that help Fourier features learn in high-dimensional environments. Our work uses learned Fourier features to improve deep RL, in contrast to prior work focused on simple environments with fixed, hand-designed Fourier features and linear function approximation \citep{kolter2007learning,konidaris2011value}. 
Although we focus on prioritizing low-frequency signals to reduce bootstrap noise, we also present cases in Appendix \ref{sec:toy}
where biasing networks towards high-frequency learning with learned Fourier features enables fast convergence and high asymptotic performance in RL. 

\subsection{Learned Fourier Feature Architecture}
Standard MLPs can be written as the repeated, alternating composition of affine transformations $L_i(x) = W_ix + b_i$
and nonlinearity $\tau$, which is usually the ReLU $\tau(x) = \max(0, x)$:
\begin{align}
    f_\theta(x) = L_n \circ \tau \circ L_{n-1} \circ \tau \circ \dots \circ L_1(x)
\end{align}
We propose a novel architecture based on Fourier features, shown in Figure~\ref{fig:architecture}. We define a new layer:
\begin{align}
    F_B(x) = \sin(2 \pi Bx) || \cos (2 \pi Bx) || x
\end{align}
where $B$ is a $ d_{\text{fourier}} /2 \times d_{\text{input}}$ matrix and $||$ denotes concatenation. $d_\text{fourier}$ is a hyperparameter that controls the number of Fourier features we can learn; increasing $d_\text{fourier}$ increases the degree to which the model relies on the Fourier features. Following \citet{tancik2020fourier}, we initialize the entries $B_{ij} \sim N(0, \sigma^2)$, where $\sigma^2$ is a hyperparameter. Contrary to prior work, $B$ is a trainable parameter.

\begin{wrapfigure}{r}{0.5\textwidth}
\begin{minipage}{0.47\textwidth}
\vspace{-1.3em}
\begin{algorithm}[H]
\caption{LFF PyTorch-like pseudocode.}
\label{alg:code}
\algcomment{\fontsize{7.2pt}{0em}\selectfont \texttt{normal}: sample from Gaussian with specified mean and std dev; \texttt{matmul}: matrix multiplication; \texttt{cat}: concatenation.
}
\definecolor{codeblue}{rgb}{0.25,0.5,0.5}
\lstset{
  backgroundcolor=\color{white},
  basicstyle=\fontsize{7.2pt}{7.2pt}\ttfamily\selectfont,
  columns=fullflexible,
  breaklines=true,
  captionpos=b,
  commentstyle=\fontsize{7.2pt}{7.2pt}\color{codeblue},
  keywordstyle=\fontsize{7.2pt}{7.2pt},
}
\begin{lstlisting}[language=python]
class LFF():
    def __init__(self, input_size, output_size, 
                 n_hidden=1, hidden_dim=256, 
                 sigma=1.0, f_dim=256):
        # create B
        b_shape = (input_size, f_dim // 2)
        self.B = Parameter(normal(zeros(*b_shape), 
                            sigma * ones(*b_shape)))
        # create rest of network
        self.mlp = MLP(in_dims=f_dim + input_size,
                      out_dims=output_size,
                      n_hidden=n_hidden,
                      hidden_dim=hidden_dim)
    def forward(self, x):
        proj = (2 * np.pi) * matmul(x, self.B)
        ff = cat([sin(proj), cos(proj), x], dim=-1)
        return self.mlp.forward(ff)
\end{lstlisting}
\end{algorithm}
\end{minipage}
\end{wrapfigure}

 The resulting LFF MLP can be written: 
\begin{align}
    f_\theta = L_n \circ \tau \circ \dots \circ L_1 \circ \textcolor{red}{F_B}(x)
\end{align}
We can optimize this the same way we optimize a standard MLP, e.g. regression would be: 
\begin{align}
    \argmin_{\theta, \textcolor{red}{B}} \sum_{i=1}^N (L_n \circ \dots \circ \textcolor{red}{F_B}(x_i) - y_i)^2
\end{align}
We propose two key improvements to random Fourier feature input embeddings \citep{rahimi2007random,tancik2020fourier}: training $B$ and concatenating  the input $x$ to the Fourier features. We hypothesize that these changes are necessary to preserve information in high-dimensional RL problems,
where it is increasingly unlikely that randomly initialized $B$ produces Fourier features well-suited for the task. Training $B$ alleviates this problem by allowing the network to discover them on its own. Appendix ~\ref{sec:variance_after_training} 
shows that training $B$ does change its values, but its variance remains quite similar to the initial $\sigma^2$. This indicates that $\sigma$ remains an important knob that controls the network behavior throughout training. Concatenating $x$ to the Fourier features is another key improvement for high dimensional settings. It preserves all the input information, which has been shown to help in RL \citep{sinha2020d2rl}. We further analyze these improvements in Section~\ref{sec:ablations}.


\section{Theoretical Analysis}

By providing periodic features (which are controlled at initialization by the variance $\sigma^2$), LFF biases the network towards fitting a desired range of frequencies.
In this section, we hope to understand why the LFF architecture provably controls the rate at which various frequencies are learned. We draw upon linear neural network approximations in the infinite width limit, which is known as the neural tangent kernel approach \citep{jacot2018neural}. This approximation allows us to understand the training dynamics of the learned neural network output function. 
While NTK analysis has been found to diverge from real-world behavior in certain cases, particularly for deeper convolutional neural networks \citep{fort2020deep,allen2019can}, it has also been remarkably accurate in predicting directional phenomena \citep{tancik2020fourier,basri2019convergence,basri2020frequency}. 

We provide background on the NTK in Section~\ref{sec:ntk}, and discuss the connection between the eigenvalues of the NTK kernel matrix and the rate at which different frequencies are learned in Section~\ref{sec:eigenvalues}. We then analyze the NTK and frequency learning rate for Fourier features in networks with 2 layers (Section \ref{subsec:2layer_ntk}) or more (Section~\ref{subsec:deeper_ntk}). 

\subsection{Neural Tangent Kernel}
\label{sec:ntk}
We can approximate a neural network using a first-order Taylor expansion around its initialization $\theta_0$: 
\begin{align}
    f_\theta(x) \approx f_{\theta_0}(x) + \nabla_\theta f_{\theta_0}(x)^\top ( \theta - \theta_0)
\end{align}
The Neural Tangent Kernel \citep{jacot2018neural} line of analysis makes two further assumptions: $f_\theta$ is an infinitely wide neural network, and it is trained via gradient flow. Under the first condition, a trained network stays very close to its initialization, so the Taylor approximation is good (the so-called ``lazy training'' regime \citep{chizat2018lazy}). Furthermore, $f_{\theta_0}(x) = 0$ in the infinite width limit, so our neural network is simply a linear model over features $\phi(x) = \nabla_\theta f_{\theta_0}(x)$. This gives rise to the kernel function: 
\begin{align}
k(x_i, x_j) = \langle  \nabla_\theta f_{\theta_0}(x_i),  \nabla_\theta f_{\theta_0}(x_j) \rangle
\end{align}
The kernel function $k$ is a similarity function: if $k(x_i, x_j)$ is large, then the predictions $f_\theta(x_i)$ and $f_\theta(x_j)$ will tend to be close. 
This kernel function is deterministic and does not change over the course of training, due to the infinite width assumption. If we have $n$ training points $(x_i, y_i)$, $k$ defines a PSD kernel matrix $K \in \mathbb{R}^{n \times n}_+$ where each entry $K_{ij} = k(x_i, x_j)$. Fascinatingly, when we train this infinite width neural network with gradient flow on the squared error, we precisely know the model output at any point in training. At time $t$, we have training residual: 
\begin{align}
\label{eq:main_convergence}
    f_{\theta_t}(x) - y = e^{-\eta Kt} (f_{\theta_0}(x) - y)
\end{align}
where $ f_{\theta_t}(x)$ is the column vector of model predictions for all $x_i$, and $y$ is the column vector of stacked training labels (see Appendix ~\ref{subsec:convergence_sketch} 
for proof sketch). Eq.~\ref{eq:main_convergence} is critical because it describes how different components of the training loss decrease over time. Section~\ref{sec:eigenvalues} will build on this result and examine the training residual in the eigenbasis of the kernel matrix $K$. This analysis will reveal that each frequency present in the labels $y$ will be learned at its own rate, determined by $K$'s eigenvalues. 

\subsection{Eigenvalues of the LFF NTK} 
\label{sec:eigenvalues}
Consider applying the eigendecomposition of $K = Q\Lambda Q^*$ to Eq.~\ref{eq:main_convergence}, noting that $e^{-\eta Kt} = Q e^{-\eta \Lambda} Q^*$ as the matrix exponential is defined $e^X \coloneqq \sum_{k=0}^\infty \frac{1}{k!} X^k$, so $Q$ and $Q^*$ will repeatedly cancel in the middle of $X^k$ since $Q$ is unitary.
\begin{align}
\label{eq:rate}
    Q^*(f_{\theta_t}(x) - y) = e^{-\eta \Lambda t}Q^* (f_{\theta_0}(x) - y)
\end{align}
Note that the $i$th component of the residual $Q^*(f_{\theta_t}(x) - y)$ decreases with rate $e^{-\eta \lambda_i}$. 

Consider the scenario where the training inputs $x_i$ are evenly spaced on the $d$-dimensional sphere $\mathbb S^{d-1}$. When $k$ is isotropic, which is true for most networks whose weights are sampled isotropically, the kernel matrix $K$ is circulant (each row is a shifted version of the row above). In this special case, $K$'s eigenvectors correspond to frequencies from the discrete Fourier transform (DFT), and the corresponding eigenvalues are the DFT values of the first row of $K$ (see Appendix \ref{subsec:dft_proof}).
Combining this fact with Eq.~\ref{eq:rate}, where we looked at the residual in $K$'s eigenbasis, shows that each frequency in the targets is learned at its own rate, determined by the eigenvalues of $K$. For a ReLU MLP, these eigenvalues decay approximately quadratically with the frequency \citep{basri2019convergence}. This decay rate is rather slow, so MLPs often fit undesirable medium and high frequency signals. \textbf{We hypothesize that LFF controls the frequency-dependent learning rate by tuning the kernel matrix $K$'s eigenvalues.} We examine LFF's kernel matrix eigenvalues in  Sections~\ref{subsec:2layer_ntk} and \ref{subsec:deeper_ntk} and verify this hypothesis.

\subsection{NTK Analysis of 2-layer Network with Fourier Features}
\label{subsec:2layer_ntk}

\begin{figure}
    \centering
    \begin{subfigure}{0.325\textwidth}
        \centering
        \includegraphics[width=\linewidth]{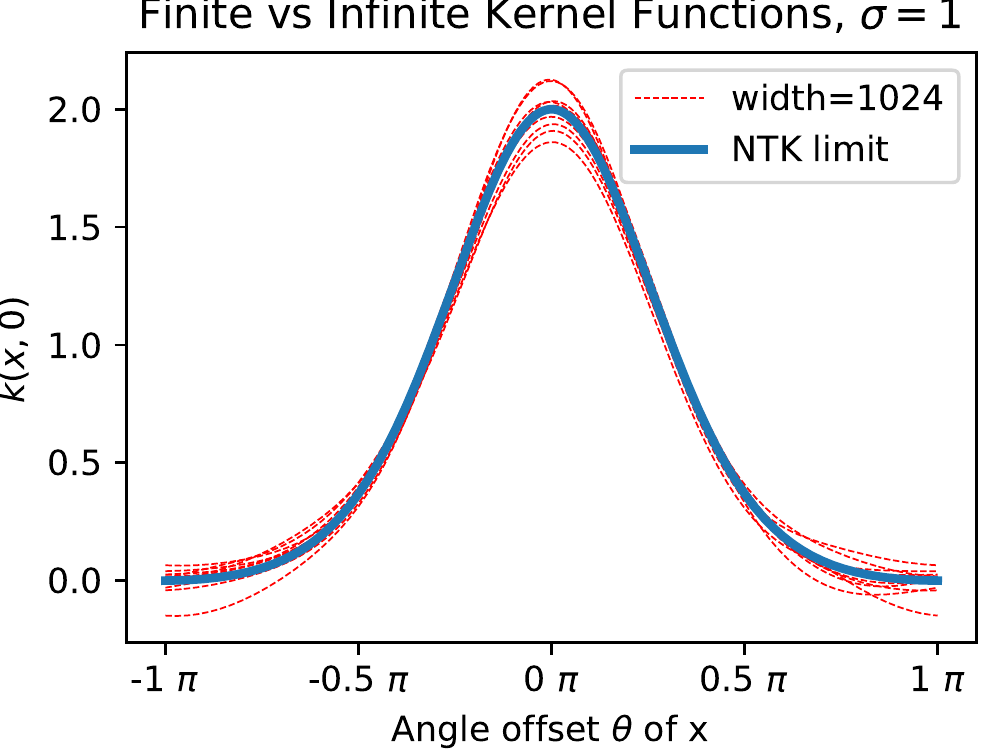}
        \caption{Quality of NTK approximation}
    \end{subfigure}
    \begin{subfigure}{0.325\textwidth}
        \centering
        \includegraphics[width=\linewidth]{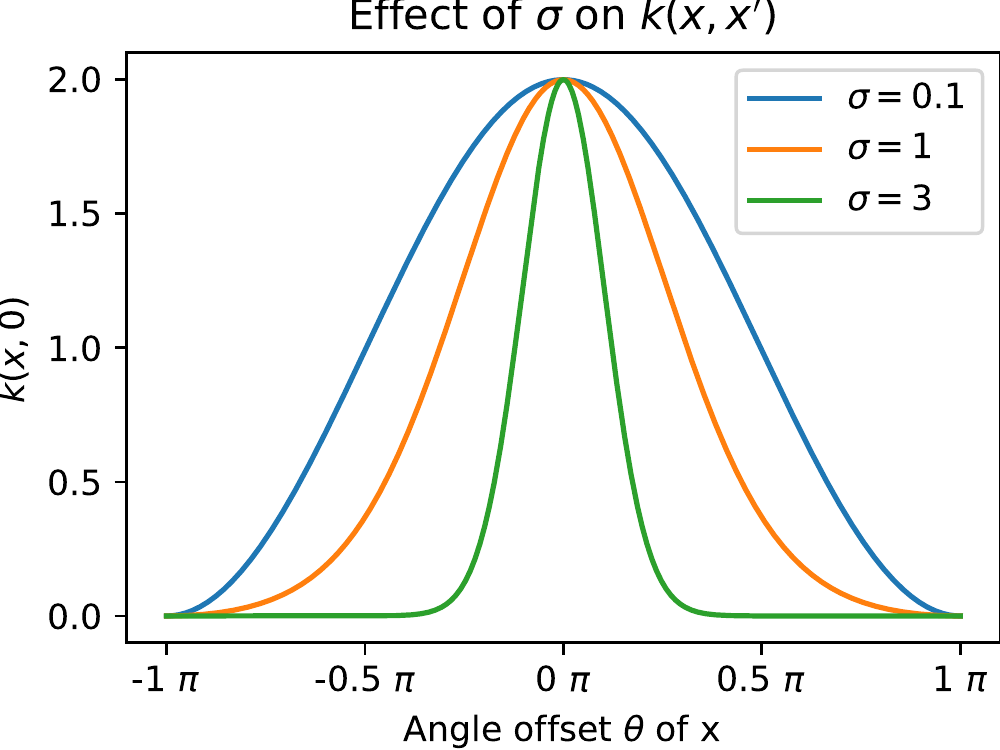}
        \caption{NTK kernel function}
    \end{subfigure}
    \begin{subfigure}{0.325\textwidth}
        \centering
        \includegraphics[width=\linewidth]{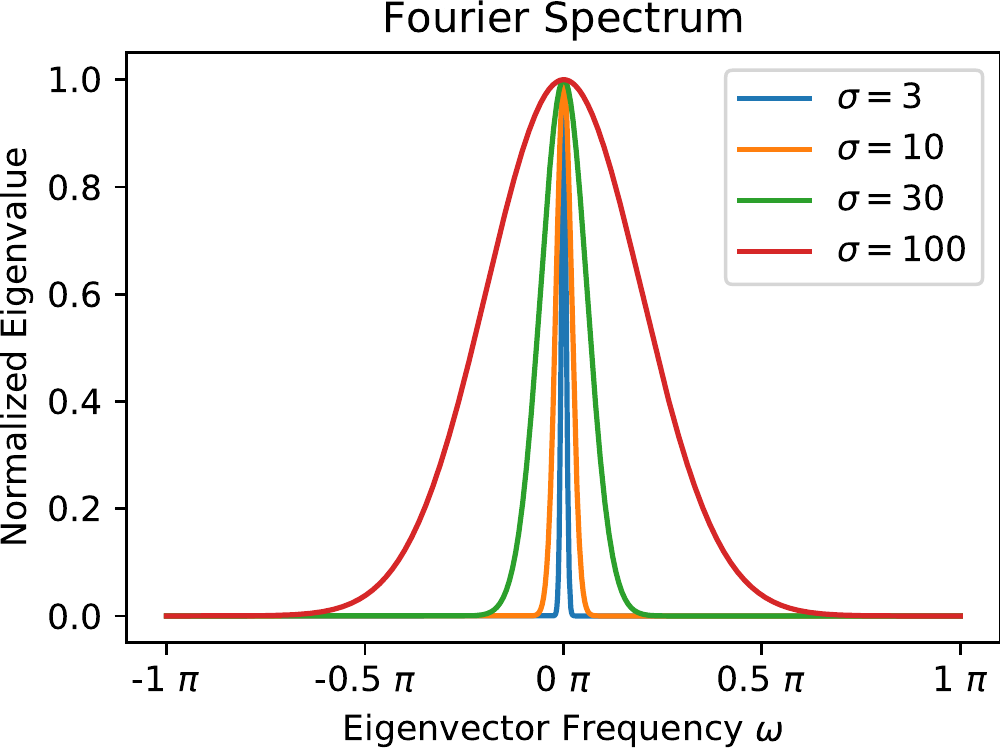}
        \caption{Eigenvalue spectrum}
    \end{subfigure}
    \caption{\textbf{NTK of the 2-layer Fourier feature model.} We plot the NTK $k(x, 0)$ for points $x$ that lie on the unit circle. In (a) and (b), $\theta \in [-\pi, \pi]$ denotes the offset from a reference point with $\theta_0 = 0$. Note that our NTK is shift invariant, so these figures are valid for any reference point $\theta_0$. 
    \textbf{Left:} we compare the NTK infinite-width limit to the kernel function of 10 randomly initialized 2-layer Fourier feature networks with width 1024. The NTK limit is quite accurate for realistically wide networks. 
    \textbf{Middle:} NTK kernel function $k$ for varying settings of the Fourier feature variance $\sigma^2$. Larger $\sigma$ enable sharp, local learning, while smaller $\sigma$ induce smoother function learning. 
    \textbf{Right:} we compute the NTK kernel matrix $K$, then find its eigenvalues with the discrete Fourier transform. The y-axis shows eigenvalues, and the x-axis indicates the corresponding frequencies.
    With low $\sigma$, only the lowest frequencies have nonvanishing eigenvalues, so they are the only ones learned through training. 
    Increasing $\sigma$ here increases the higher frequencies' eigenvalues, so they can be learned faster.}
    \label{fig:ntk}
\end{figure}
To simplify our LFF NTK analysis, we consider a two layer neural network $f: \mathbb{R}^d \rightarrow \mathbb R$: 
\begin{align}
    f(x) = \sqrt{\frac{2}{m}} W^\top \begin{bmatrix}
    \sin(Bx) \\
    \cos(Bx)
    \end{bmatrix}
\end{align}
where each row of $B$ is a vector $b_i^\top \in \mathbb{R}^{1 \times d}$, and there are $m$ rows of $B$. $W_i \sim \mathcal N(0,1)$ and $B_{ij} \sim \mathcal N(0, \sigma^2)$, where $\sigma$ is a hyperparameter. Note that concatenating $x$ is omitted for this two-layer model. This is because any contribution from concatenation goes to zero as we increase the layer width $m$.  Lemma~\ref{lemma:ntk} determines an analytical expression for the LFF kernel function $k(x, x')$.
\begin{lemma}
\label{lemma:ntk}
For $x, x' \in \mathbb S^{d-1}$ with angle $\theta = \cos^{-1}(x^\top x')$, we have the NTK kernel function:
\begin{align}
    k(x, x') &= \left(2 - \frac{\|x - x'\|_2^2}{2} \right) \exp\left\{-\frac{\sigma^2}{2}\|x - x'\|_2^2\right\}
\end{align}
\end{lemma}

Proof: see Appendix \ref{subsec:2layer_proof}.
This closed form expression for $k(x, x')$ elucidates several desirable properties of Fourier features. $\sigma$ directly controls the rate of the exponential decay of $k$, which is the similarity function for points $x$ and $x'$. For large $\sigma$, $k(x, x')$ rapidly goes to 0 as $x$ and $x'$ get farther apart, so their labels only affect the learned function output in a small local neighborhood. This intuitively corresponds to high-frequency learning. In contrast, small $\sigma$ ensures $k(x, x')$ is large, even when $x$ and $x'$ are relatively far apart. This induces smoothing behavior, inhibiting high-frequency learning. We plot the NTK for varying levels of $\sigma$ in Figure~\ref{fig:ntk}(b) and show that $\sigma$ directly controls the frequency learning speed in Figure~\ref{fig:ntk}(c). Figure~\ref{fig:ntk}(a) also verifies that the NTK limit closely matches the empirical behavior of realistically sized networks at initialization.

\subsection{NTK of Deeper Networks}
\label{subsec:deeper_ntk}

\begin{figure}
    \centering
    \includegraphics[width=0.49\textwidth]{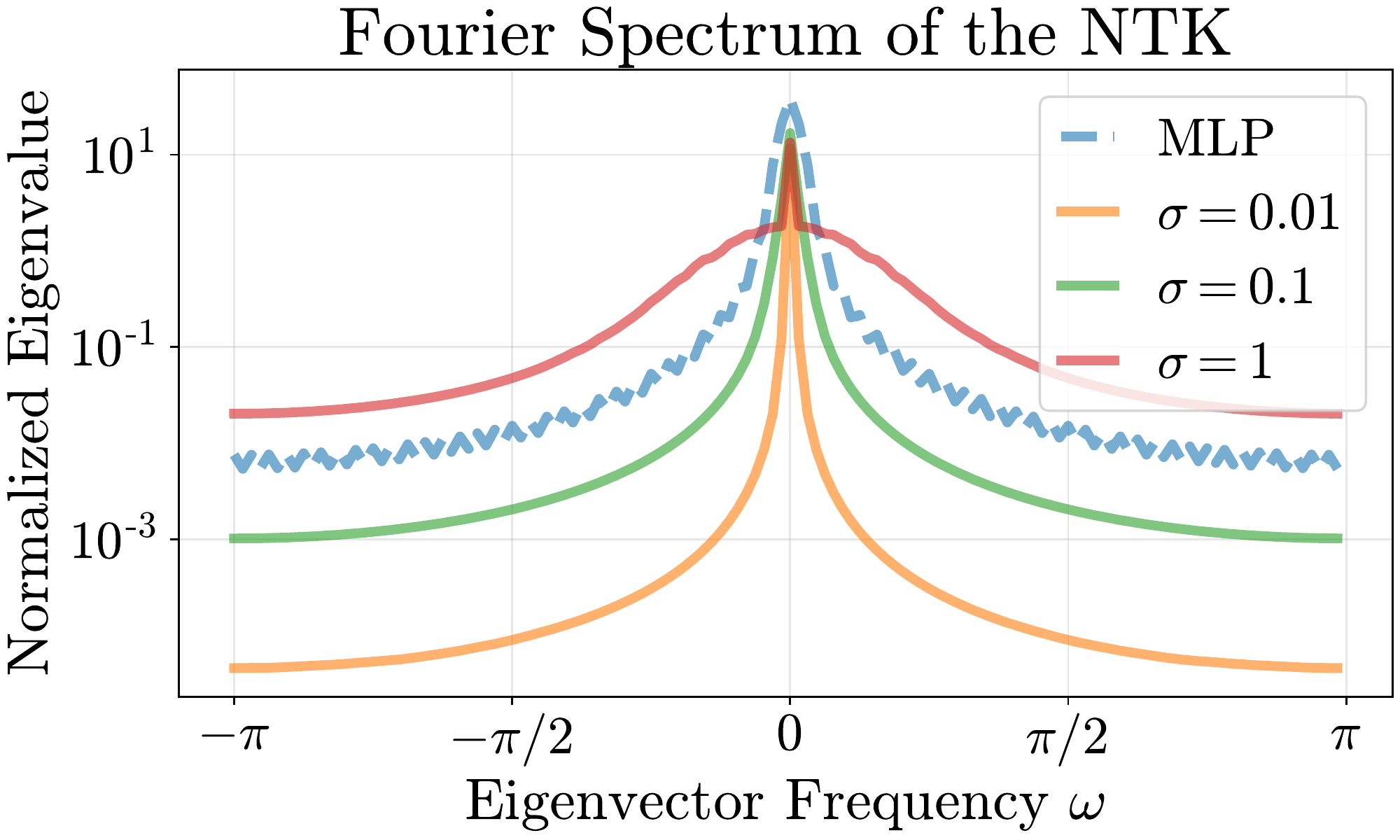}
    \includegraphics[width=0.49\textwidth]{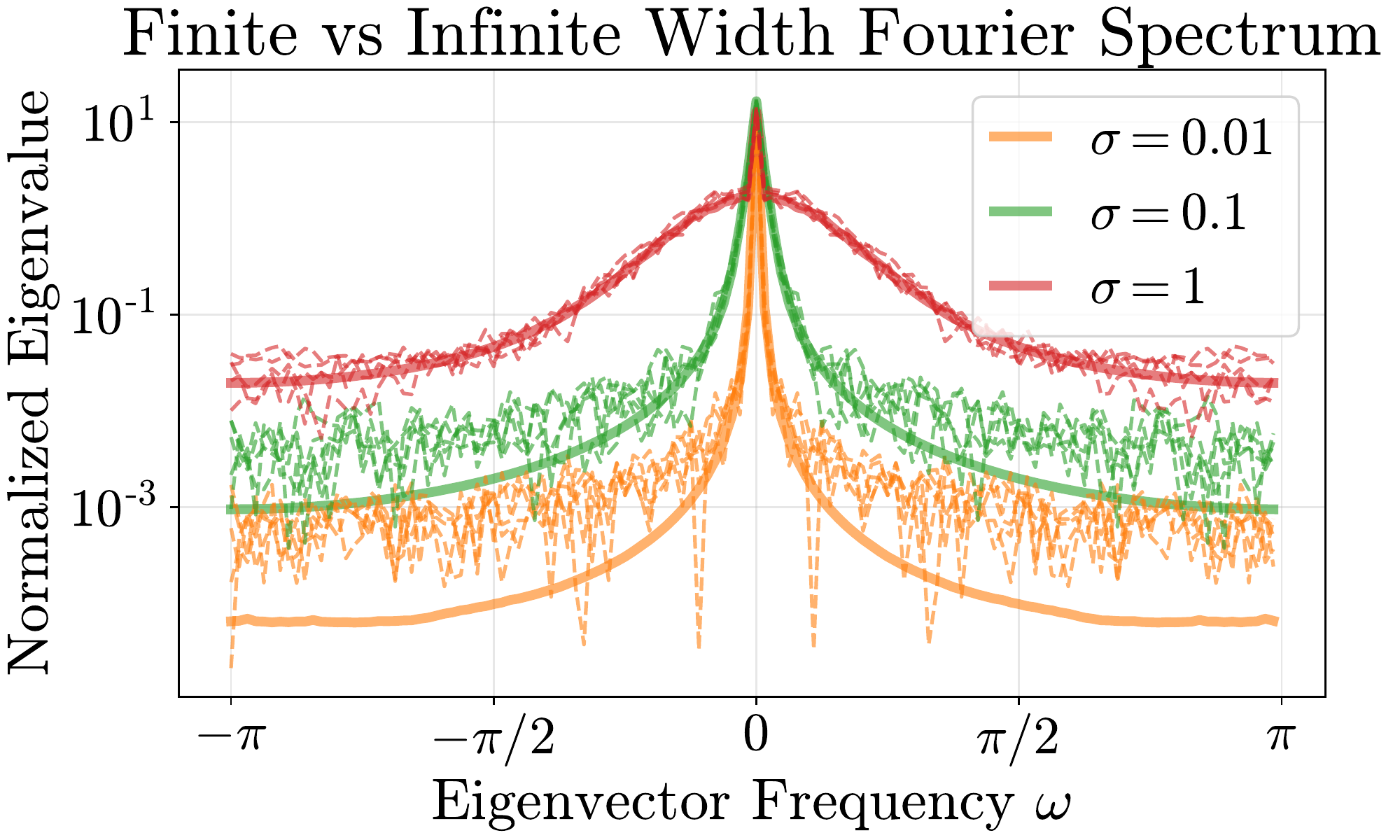}
    \caption{Left: we compare the NTK eigenvalue spectrum (which determines the frequency-specific learning rate) of deep networks with Fourier features to those of a vanilla MLP. Right: we initialize finite-width LFF networks with 2 hidden layers of 1024 units each and compare their kernels (dashed) to the corresponding NTK infinite-width limit (solid). We find that the NTK is accurate (note the log scale) and that decreasing $\sigma$ indeed results in lower convergence rates for higher frequencies. }
    \label{fig:deeper_ntk}
\end{figure}

Figure~\ref{fig:ntk}(c) shows that larger initialization variance $\sigma^2$ corresponds to larger eigenvalues for high frequencies in the 2-layer model. This matches empirical results that small $\sigma$ leads to underfitting and large $\sigma$ leads to overfitting \citep{tancik2020fourier}. However, Figure~\ref{fig:ntk}(c) indicates that only extremely large $\sigma$, on the order of $10^2 - 10^3$, result in coverage of the high frequencies. This contradicts \citet{tancik2020fourier}, who fit fine-grained image details with $\sigma \in [1, 10]$. 
We suggest that the 2-layer model, even though it accurately predicts the directional effects of increasing or decreasing $\sigma$, fails to accurately model learning in realistically sized networks. Manually computing the kernel functions of deeper MLPs with Fourier feature input embeddings is difficult. Thus, we turn to Neural Tangents \citep{novak2019neural}, a library that can compute the kernel function of any architecture expressible in its API . 

We initialize random Fourier features of size 1024 with different variances $\sigma^2$ and build an infinite-width ReLU MLP on top with 3 hidden layers using the Neural Tangents library. As we did in Figure~\ref{fig:ntk}, we take input data $x$ that is evenly spaced on the 2D unit circle and evaluate the corresponding kernel function $k(x, 0)$ between the point $(1, 0)$ and the point $x = (\cos \theta, \sin \theta)$. Figure~\ref{fig:deeper_ntk} shows the eigenvalues of Fourier features and vanilla MLPs in this scenario. We see the same general trend, where increasing $\sigma$ leads to larger eigenvalues for higher frequencies, as we did in Figure~\ref{fig:ntk}. Furthermore, Figure~\ref{fig:deeper_ntk} also shows that this trend also holds for the exact finite-width architectures that we use in our experiments (Section \ref{sec:results}). These results now reflect the empirical behavior where $\sigma \in [1, 10]$ results in high frequency learning. This indicates that deeper networks are crucial for understanding the behavior of Fourier features.


\section{Experimental Setup}
\label{sec:setup}
\begin{figure*}[t]
    \centering
        
    \includegraphics[width=0.24\linewidth]{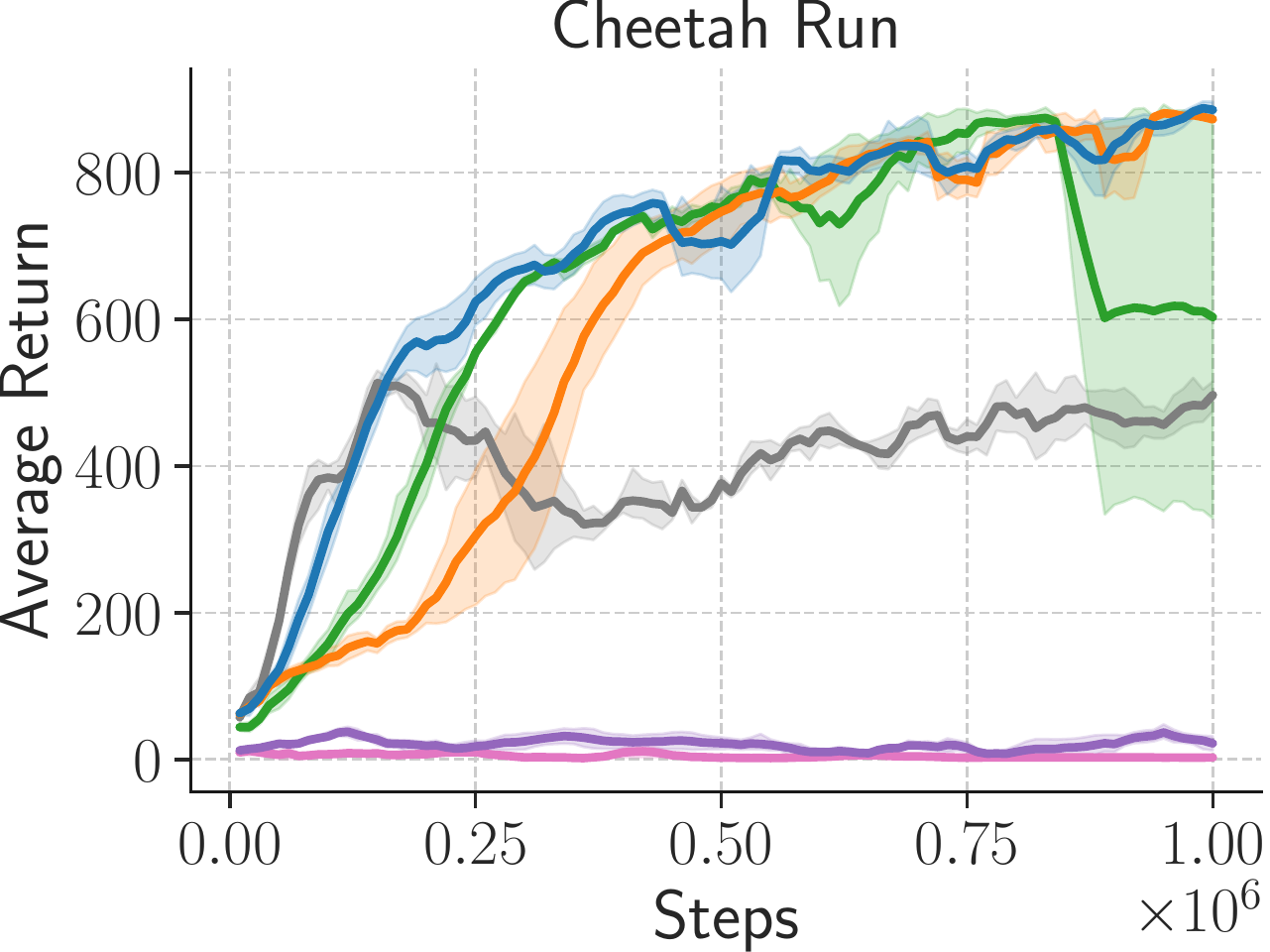} 
    \includegraphics[width=0.24\linewidth]{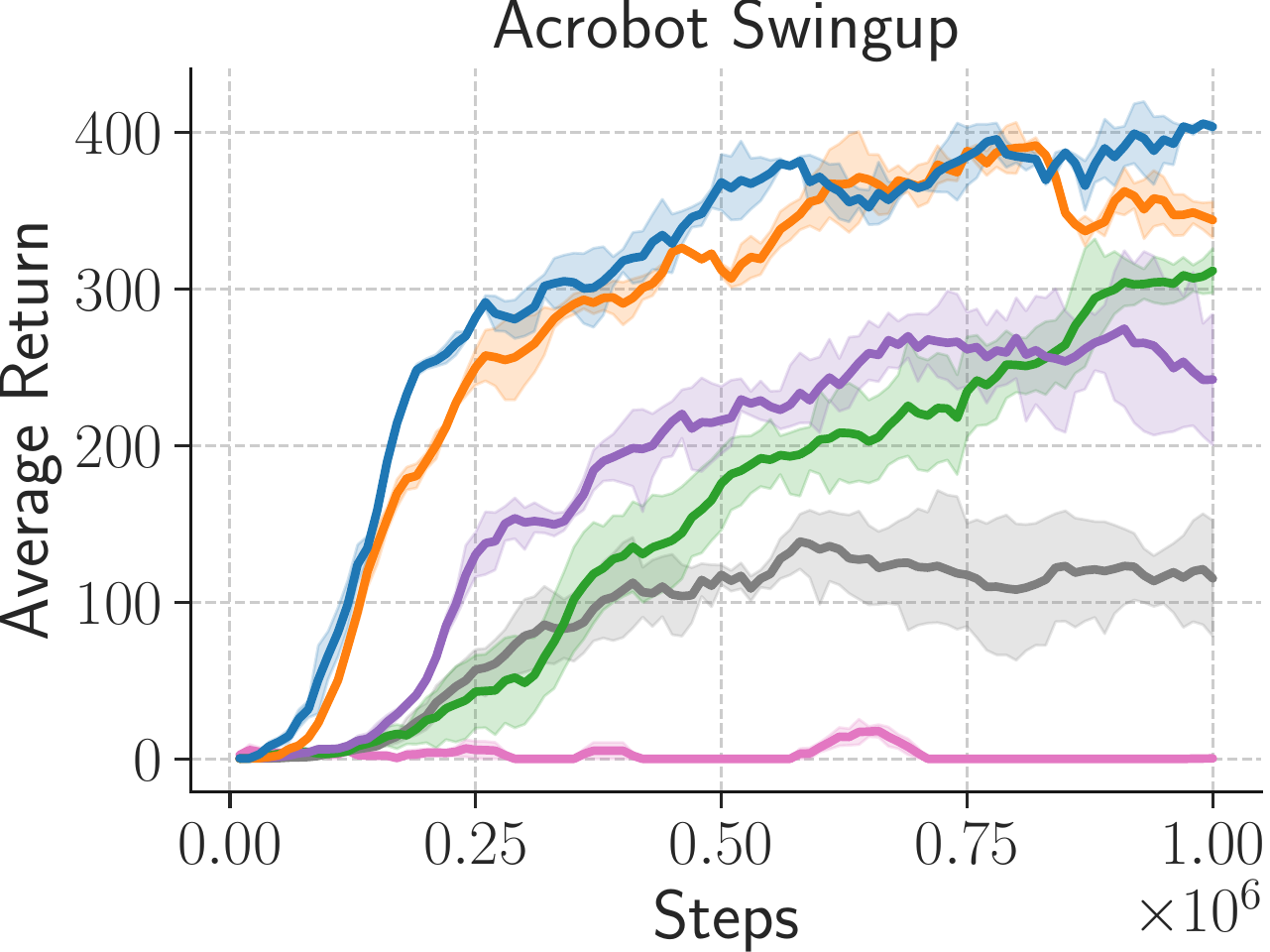} 
    \includegraphics[width=0.24\linewidth]{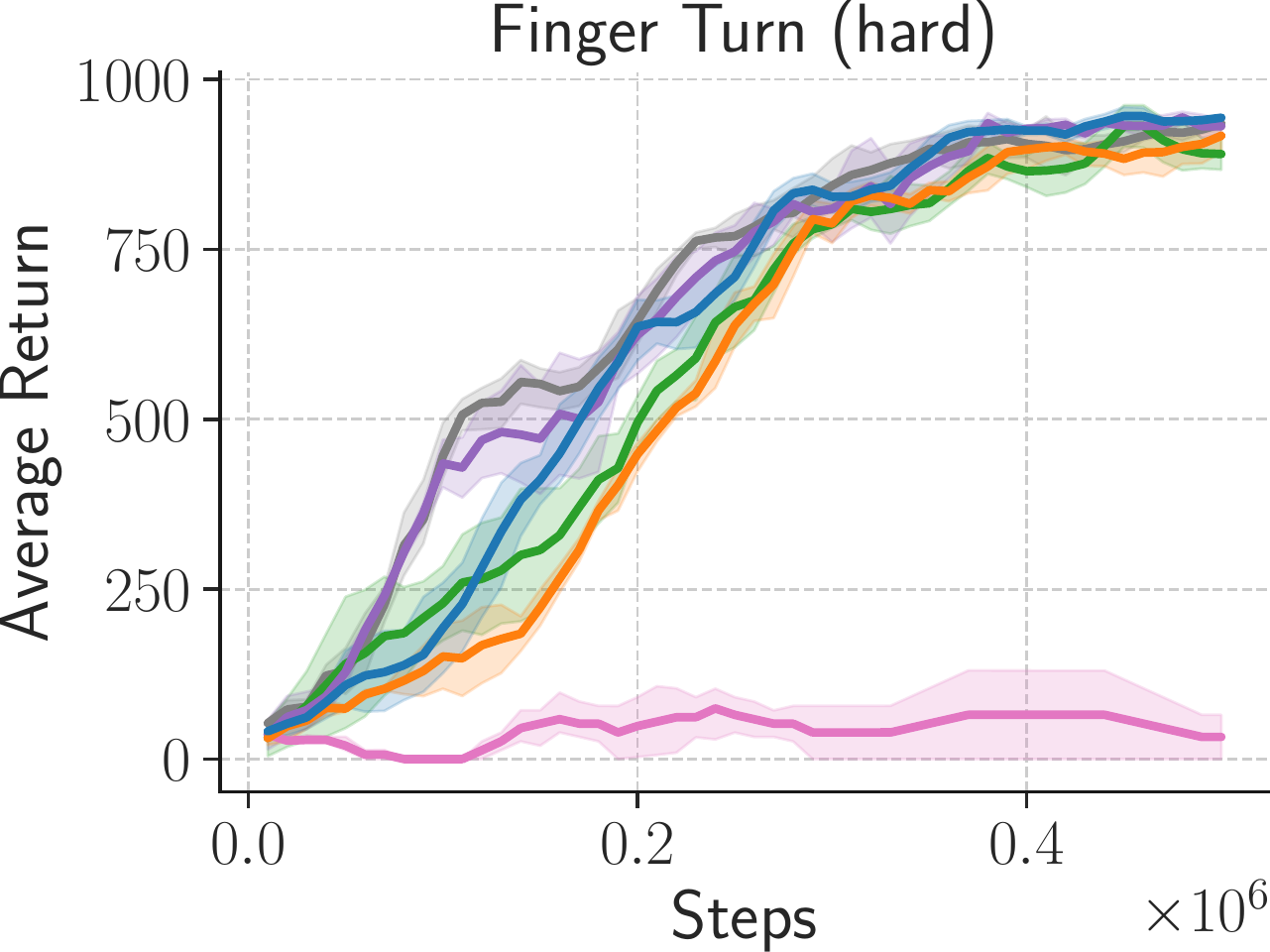}
    \includegraphics[width=0.24\linewidth]{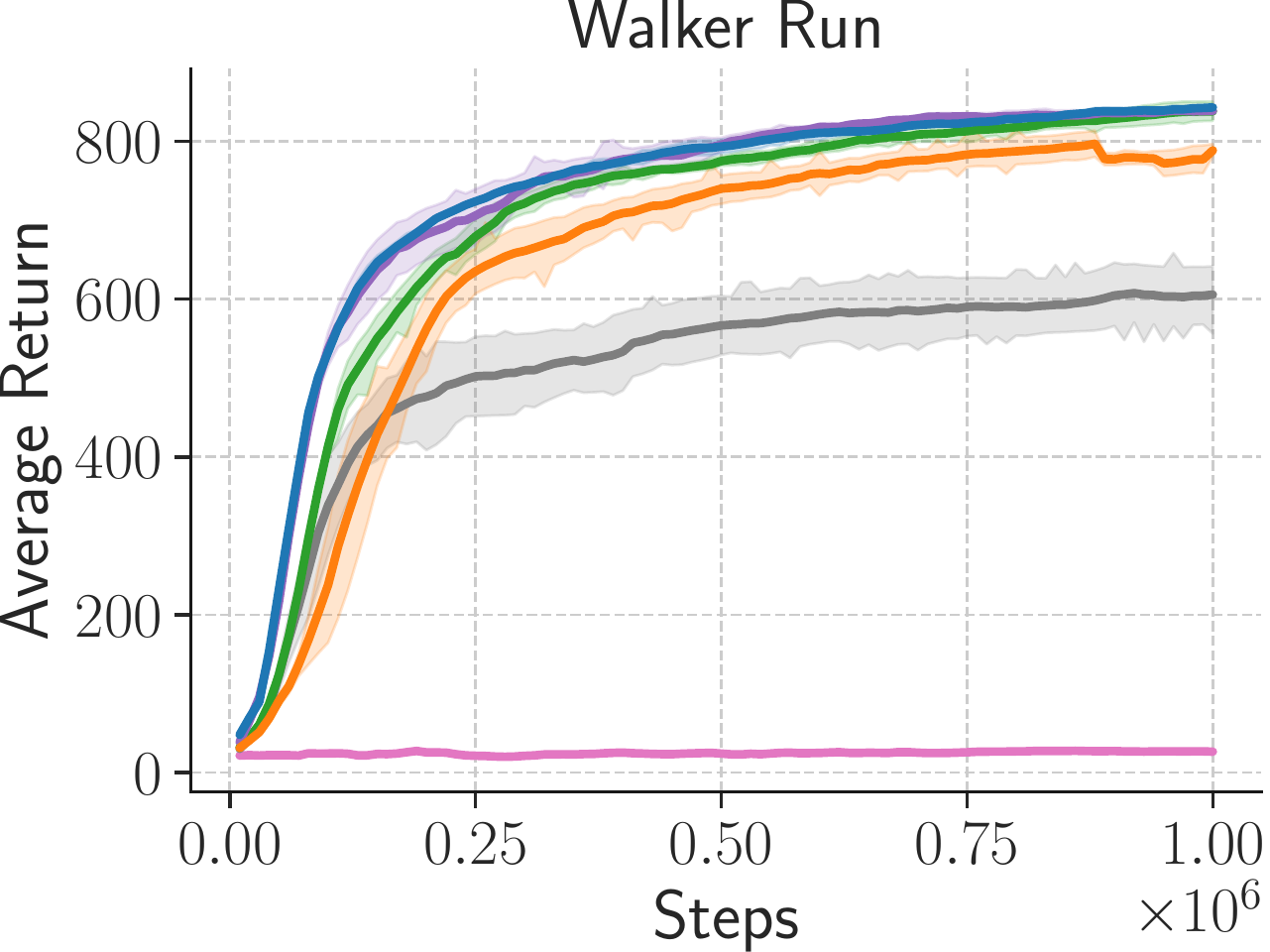} \\
    \includegraphics[width=0.24\linewidth]{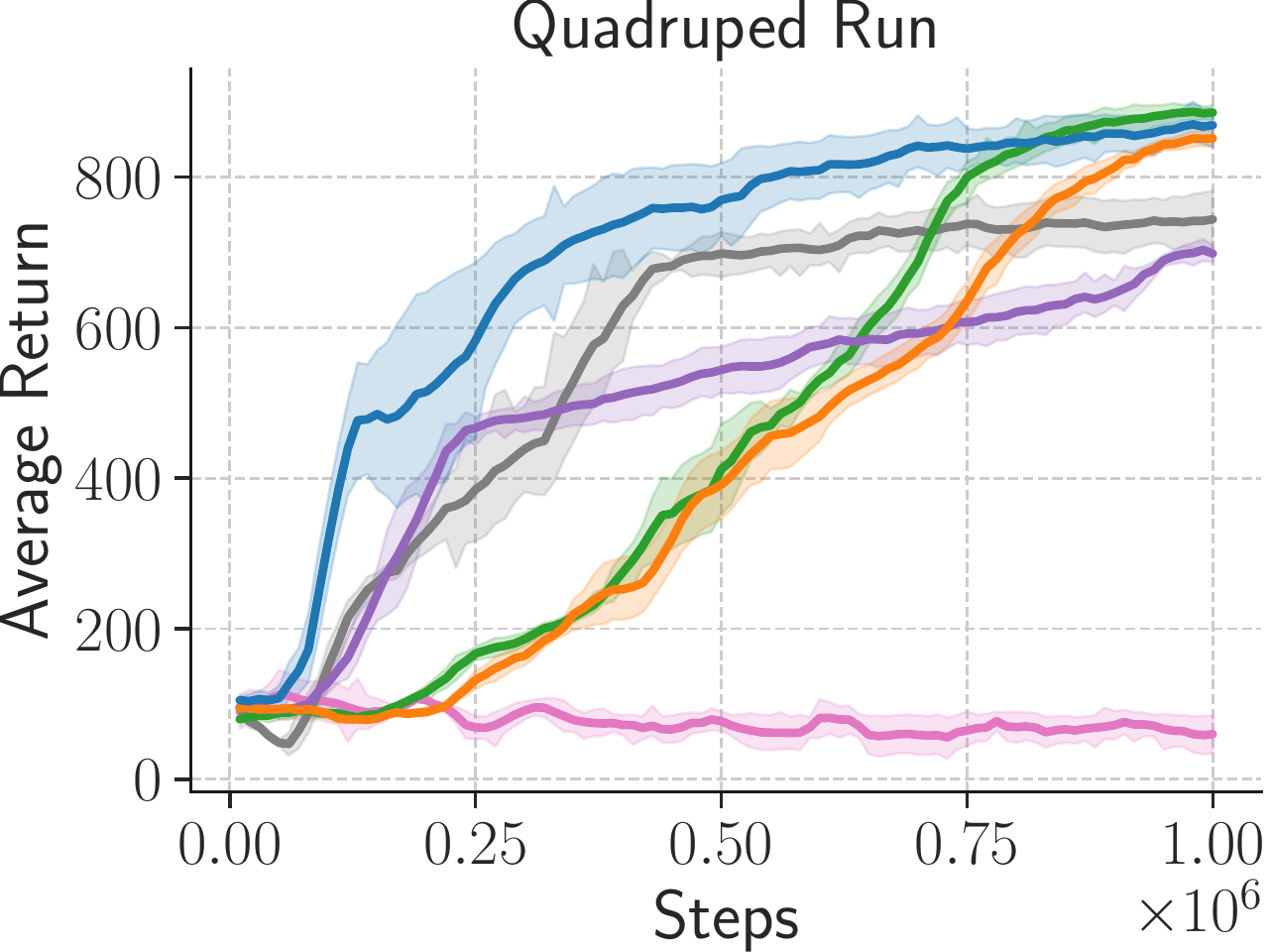}
    \includegraphics[width=0.24\linewidth]{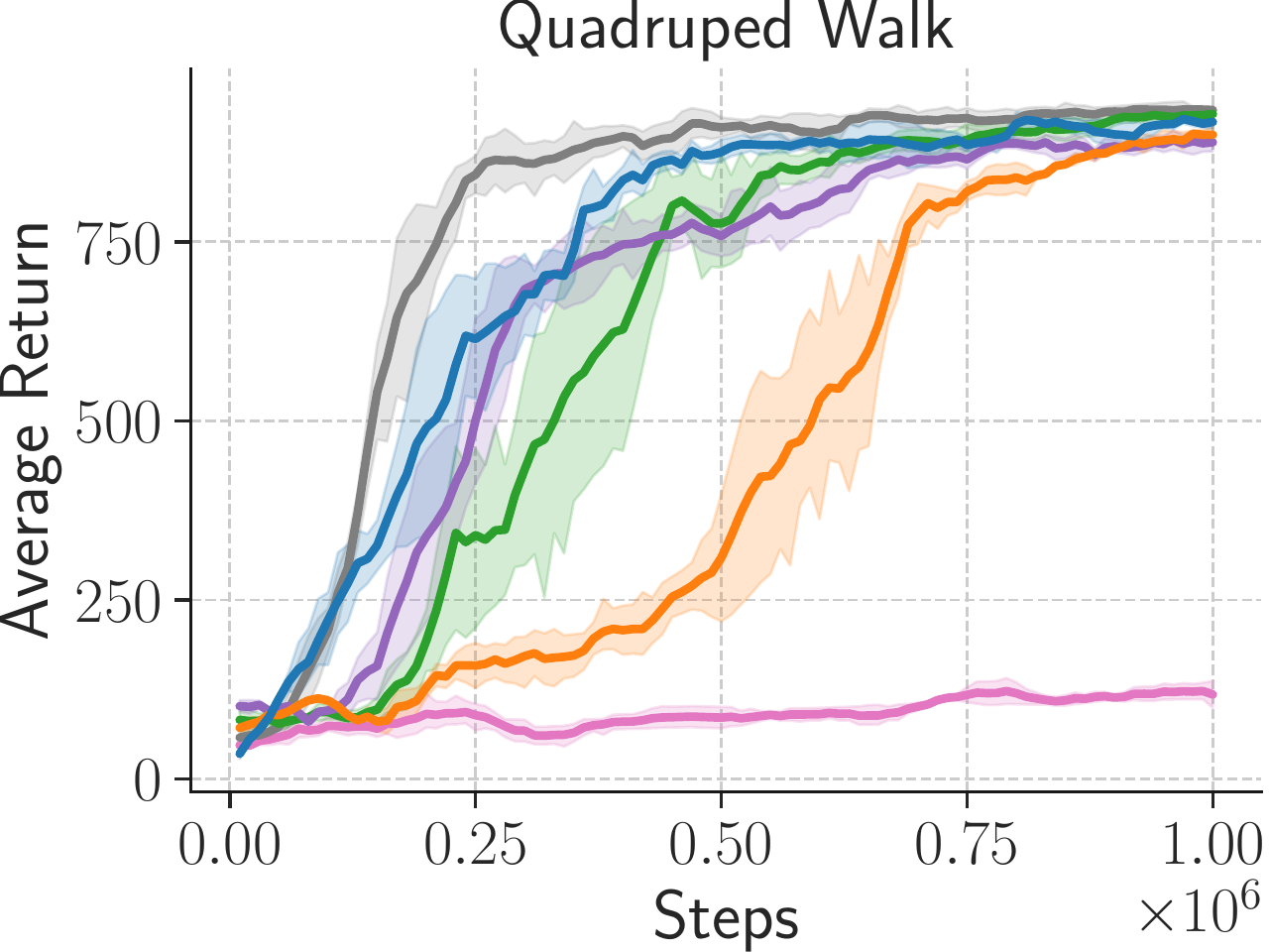} 
    \includegraphics[width=0.24\linewidth]{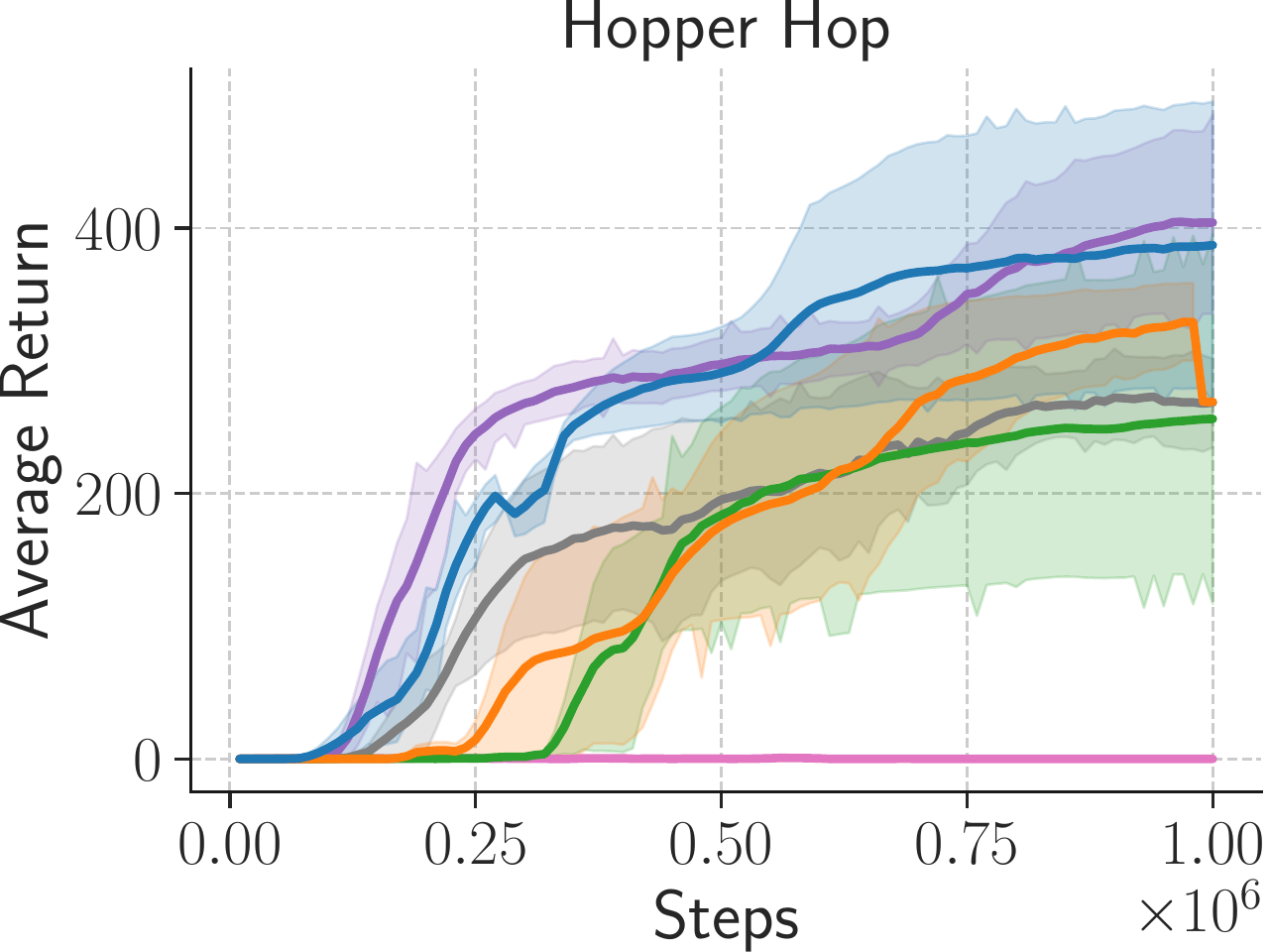} 
    \includegraphics[width=0.24\linewidth]{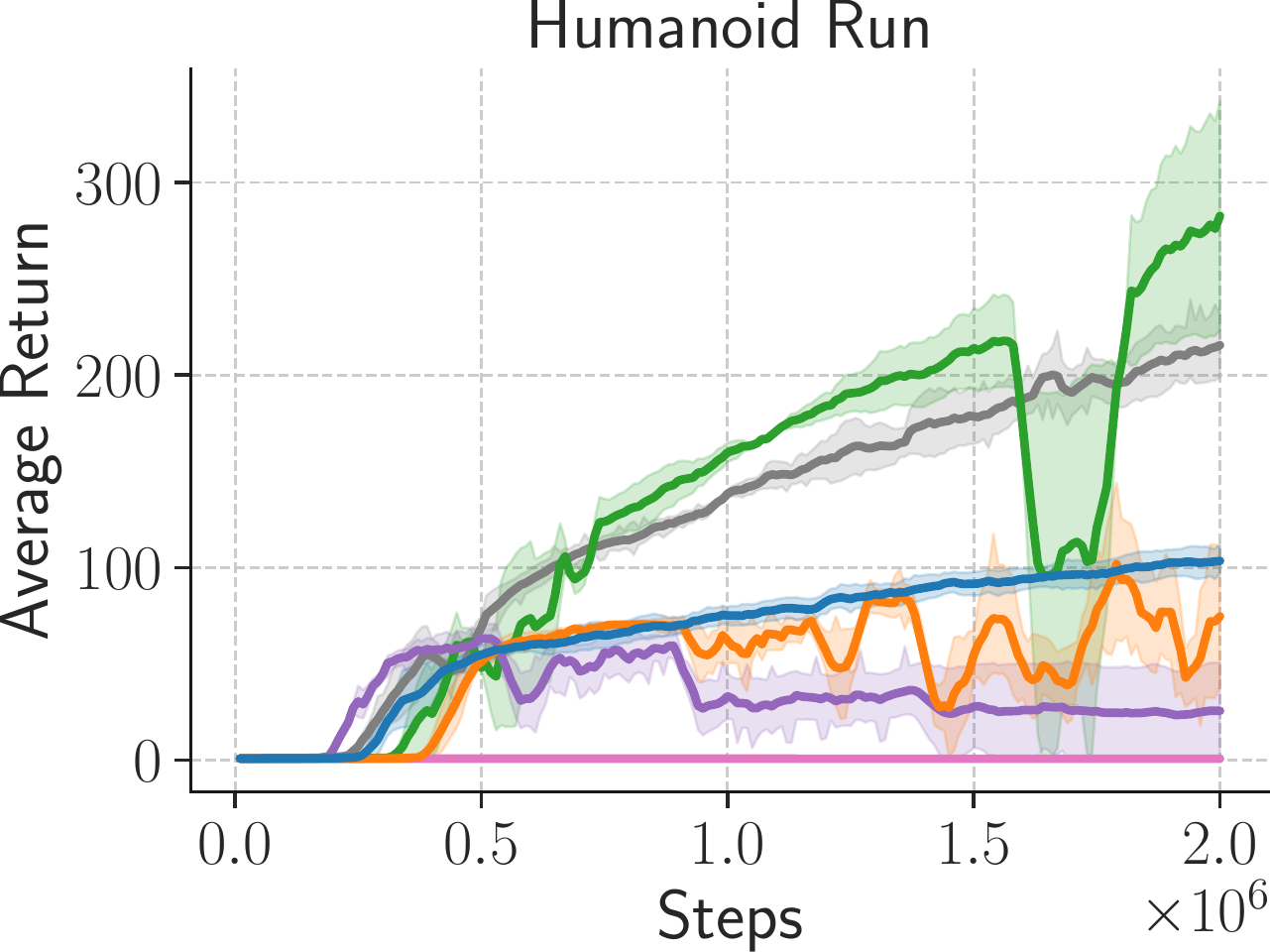} \\
    \includegraphics[width=0.95\linewidth,right]{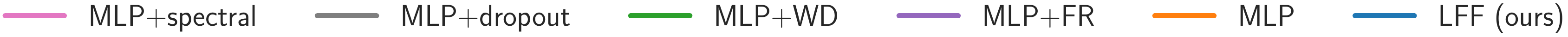} 
    \caption{\textbf{Off-policy State-based Evaluation}: Soft Actor Critic (SAC) experiments on 8 DM Control environments. We emphasize that these results are produced using the same hyperparameters (e.g. learning rate, Polyak averaging parameter, and batch size) tuned for MLPs. These results show that plugging in our LFF architecture can yield more sample-efficient learning on most environments.}
    \label{fig:SAC}
    \vspace{-1em}
\end{figure*}

We treat the learned Fourier feature network as a drop-in replacement for MLP and CNN architectures. We show that just adding Fourier features improves the performance of \textit{existing} state-of-the-art methods on \textit{existing} standard benchmark environments from DeepMind Control Suite~\citep{tassa2018deepmind}. We will release the code, \textbf{which involves only changing a few lines of code in existing RL algorithms}.

\paragraph{State-based LFF Architecture Setup}
We use soft actor-critic (SAC), an entropy-regularized off-policy RL algorithm \citep{haarnoja2018soft}, to learn 8 environments from the DeepMind Control Suite \citep{tassa2018deepmind}. 
We keep the default hyperparameters fixed, varying only the architecture for the policy and Q-function. Our LFF architecture uses our learnable Fourier feature input layer, followed by 2 hidden layers of 1024 units. We use Fourier dimension $d_\text{fourier}$ of size 1024. We initialize the entries of our trainable Fourier basis with $B_{ij} \sim \mathcal N(0, \sigma^2)$, with $\sigma = 0.01$ for all environments except Cheetah, Walker, and Hopper, where we use $\sigma = 0.001$. To make the parameter count roughly equal, we compare against an MLP with three hidden layers. The first MLP hidden layer is slightly wider, about 1100 units, to compensate for the extra parameters in LFF's first layer due to input concatenation. Learning curves are averaged over 5 seeds, with the shaded region denoting 1 standard error.

\paragraph{Image-based LFF Architecture Setup}
We test image-based learning on 4 DeepMind Control Suite environments \citep{tassa2018deepmind} with SAC + RAD \citep{laskin2020reinforcement}, which uses data augmentation to improve the sample efficiency of image-based training. The vanilla RAD architecture, which uses the convolutional architecture from \citet{srinivas2020curl}, is denoted as ``CNN'' in Figure~\ref{fig:sac_pixels}. To apply LFF to images, we observe that computing $Bx$ at each pixel location is equivalent to a 1x1 convolution without bias. This 1x1 convolution maps the the RGB channels at each pixel location from 3 dimensions to $d_\text{fourier} / 2$ channels. We then compute the $\sin$ and $\cos$ of those channels and concatenate the original RGB values, so our image goes from $H \times W \times 3$ to a $H \times W \times (d_\text{fourier} + 3)$ embedding. The 1x1 conv weights are initialized from $\mathcal N(0, \sigma^2)$ with $\sigma = 0.1$ for Hopper and Cheetah and $\sigma=0.01$ for Finger and Quadruped. As we did in the state-based setup, we make the CNN baseline fair by adding an additional 1x1 convolution layer at the beginning. This ensures that the ``CNN'' and ``CNN+LFF'' architectures have the same parameter count, and that performance gains are solely due to LFF.


\section{Results}
\label{sec:results}
\begin{figure*}
    \centering
    \includegraphics[width=0.24\linewidth]{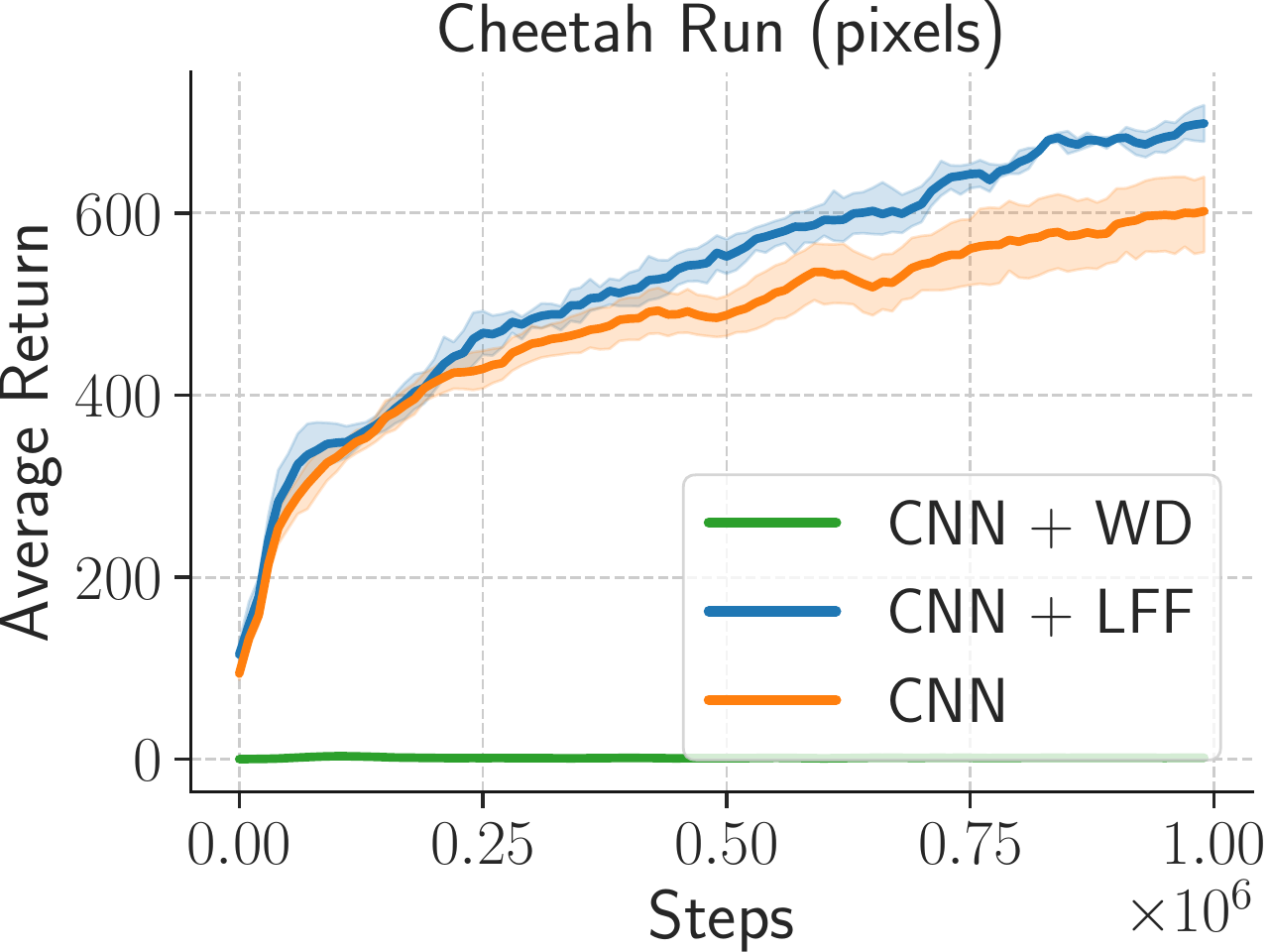}
    \includegraphics[width=0.24\linewidth]{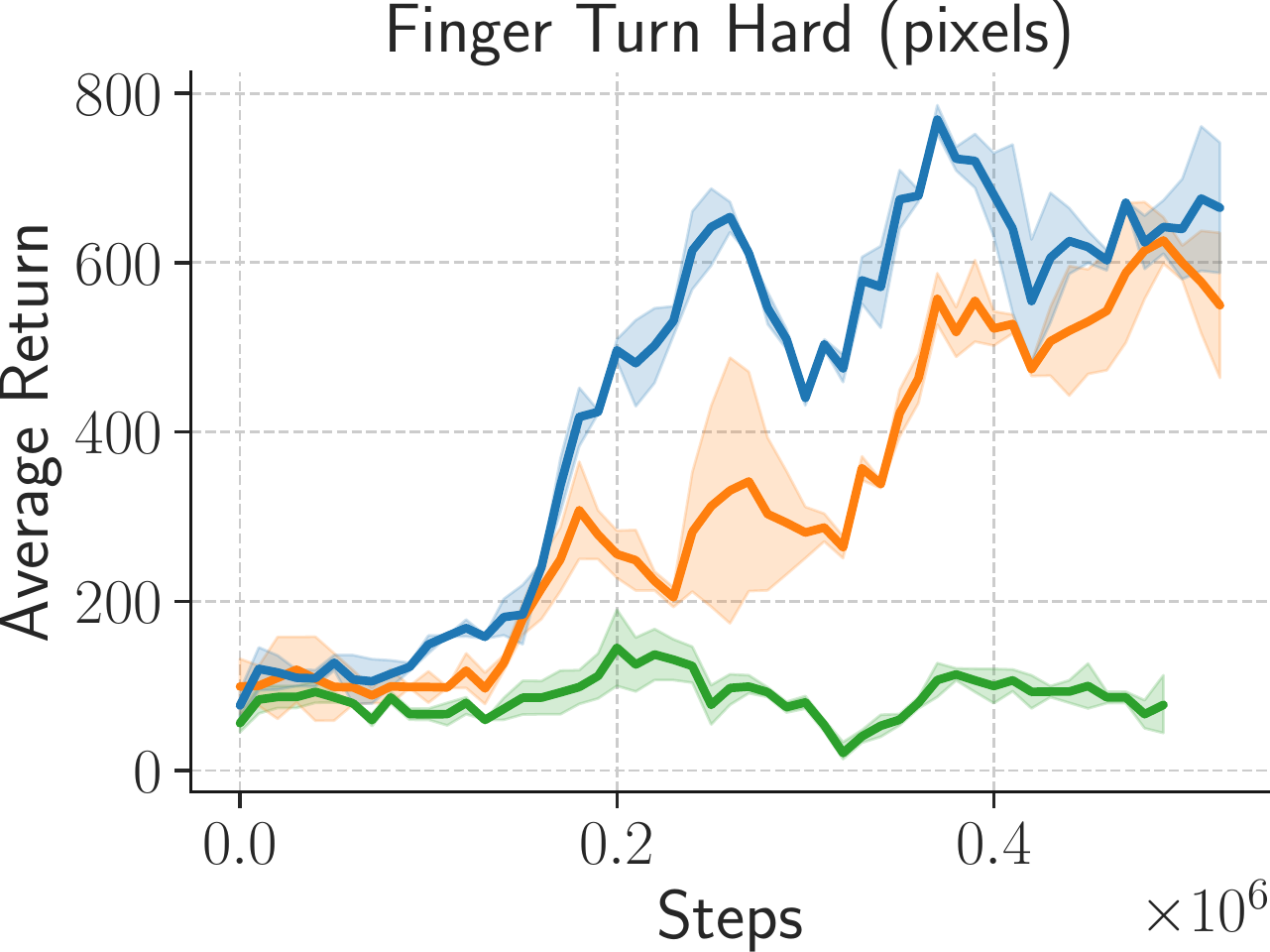}
    \includegraphics[width=0.24\linewidth]{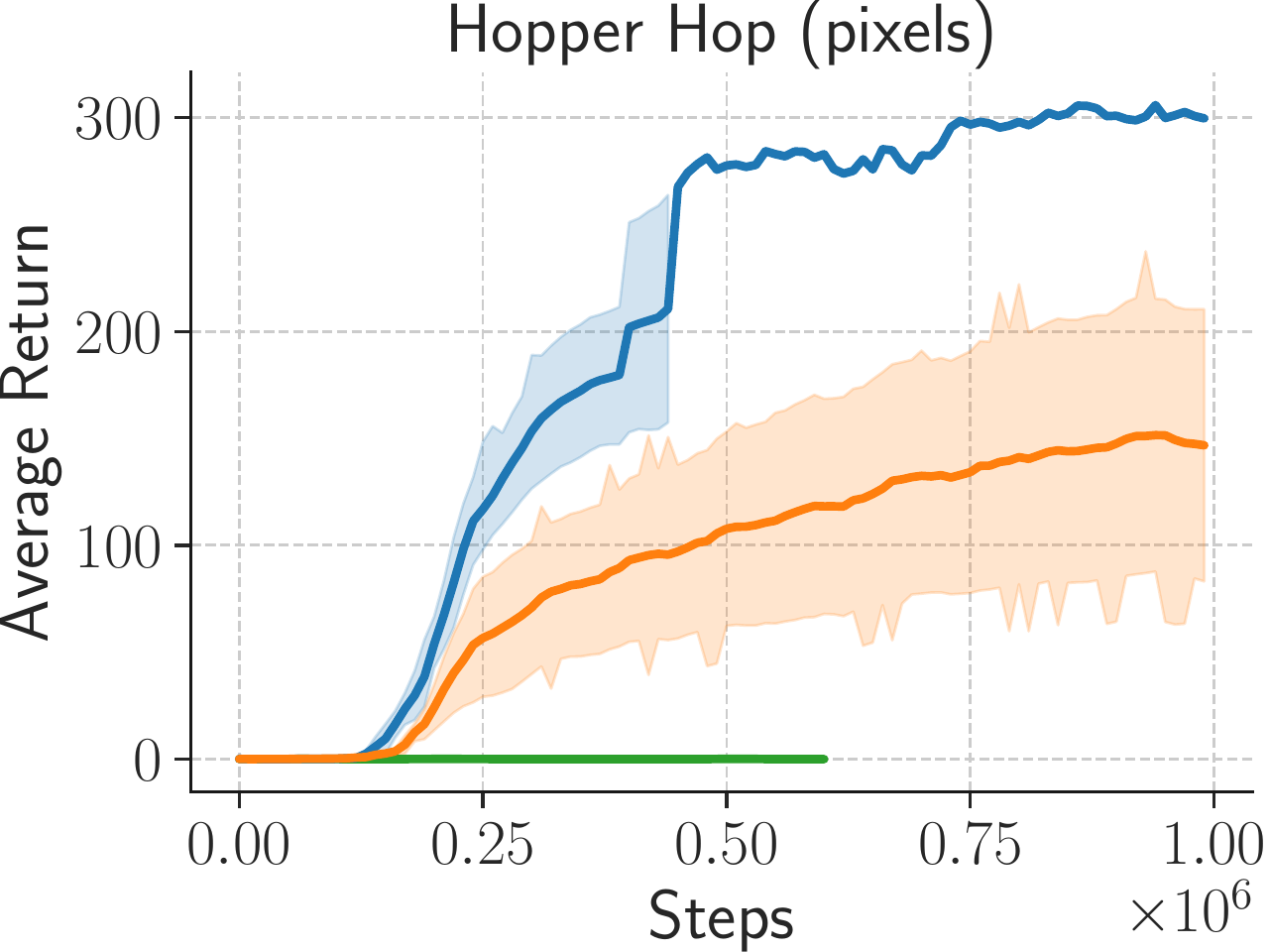}
    \includegraphics[width=0.24\linewidth]{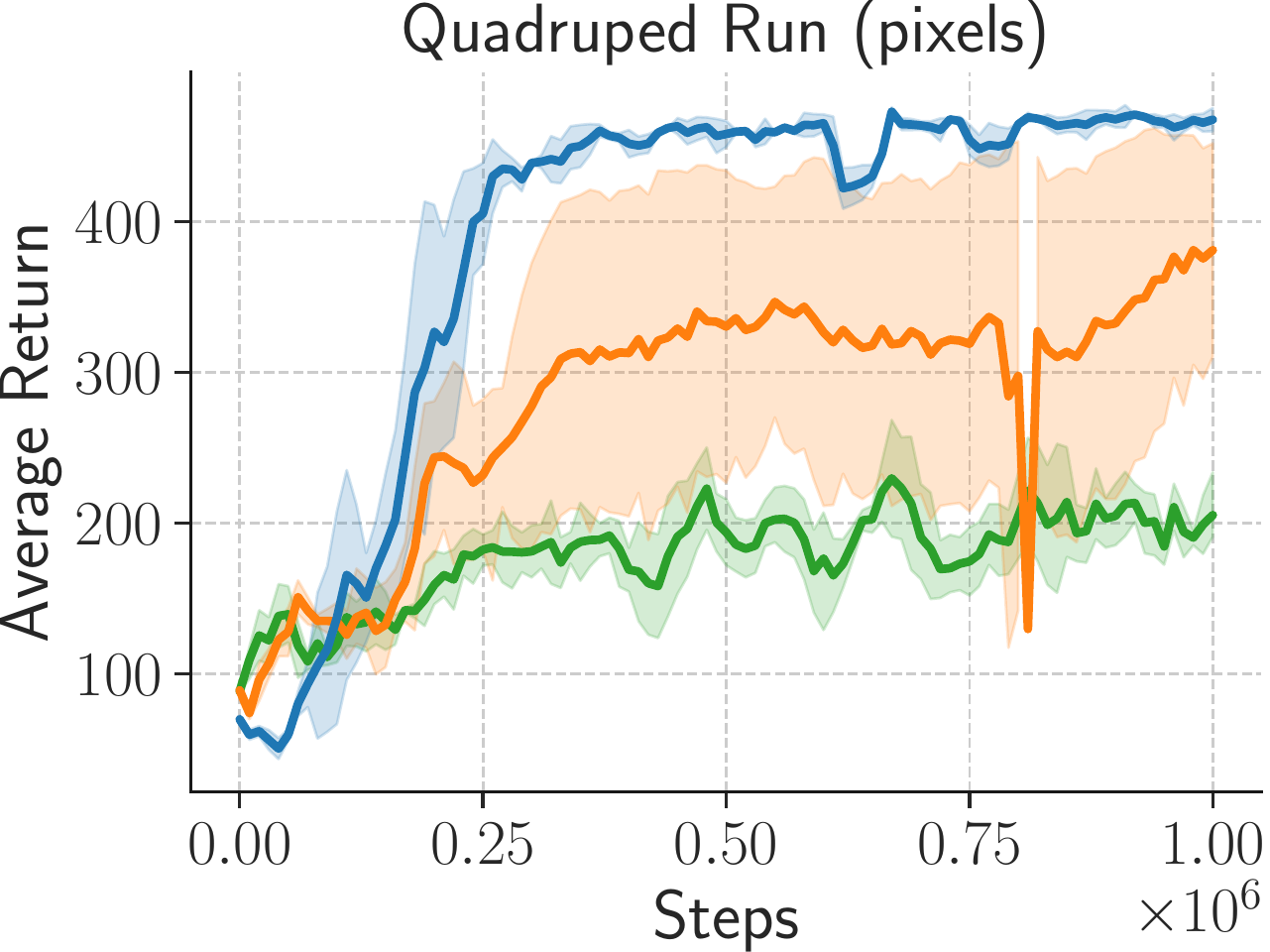}
    \caption{\textbf{Off-policy Image-based Evaluation}: SAC experiments on learning 4 DMControl environments from pixels. LFF can yield dramatic improvements in sample-efficiency over CNNs.}
    \label{fig:sac_pixels}
\end{figure*}

We provide empirical support for the approach by investigating the following questions:
\begin{enumerate}[noitemsep,topsep=0pt]
    \item Does LFF improve the sample efficiency of off-policy state-based or image-based RL?
    \item Do learned Fourier features make the Bellman update more stable?
    \item Does LFF help more when applied to the policy or the Q-function?
    \item Ablation: How important is input concatenation or training the Fourier basis $B$?
\end{enumerate}

\subsection{LFF Architecture for Off-policy RL}
We show the results of using the LFF architecture for state-based RL with SAC in Figure~\ref{fig:SAC}. LFF does clearly better than MLPs in 6 out of 8 environments, and slightly better in the remaining 2.
Figure~\ref{fig:sac_pixels} shows even stronger results on image-based RL with SAC and RAD. 
This is especially promising because these results use the hyperparameters that were tuned for the MLP or CNN baseline. We find that the return consistently starts increasing much earlier with the LFF architecture. We hypothesize that LFF reduces noise propagation due to bootstrapping, so less data is required to overcome incorrect targets. SAC can use these more accurate Q-values to quickly begin exploring high-reward regions of the MDP. 

For the state-space experiments, we also test several baselines: 
\begin{itemize}
    \item MLP with weight decay, tuned over the values $\{10^{-3}, 3\times 10^{-4}, 10^{-4}, 3 \times 10^{-4},10^{-5}\}$. Weight decay helps learning in most environments, but it can hurt performance (Acrobot, Hopper) or introduce instability (Cheetah). Weight decay strong enough to reduce overfitting may simultaneously bias the Q-values towards 0 and cause underestimation bias. 
    \item MLP with dropout \citep{srivastava2014dropout}. We add a dropout layer after every nonlinearity in the MLP. We search over $[0.05, 0.2]$ for the drop probability, and find that lower is better. Dropout does help in most environments, although occasionally at the cost of asymptotic performance. 
    \item MLP with functional regularization \citep{piche2021beyond}. Instead of using a target network to compute target values, we use the current Q-network, but regularize its values from diverging too far from the Q-values from a lagging snapshot of the Q-network. 
    \item MLP with spectral normalization \citep{gogianu2021spectral}. We add spectral normalization to the second-to-last layer of the network, as is done in \citep{gogianu2021spectral}, but find that this works very poorly. It is likely necessary to tune the other hyperparameters (learning rate, target update frequency, Polyak averaging) in order to make spectral normalization work. 
\end{itemize}
Overall, LFF consistently ranks around the top across all of the environments. It can be combined with weight decay, dropout, or functional regularization for more gains, and has a simple plug-and-play advantage because a single set of parameters works over all environments.

\subsection{Do learned Fourier features improve the stability of the Bellman updates?}
\label{sec:exp_stability}

\begin{figure*}
    \centering
    \includegraphics[width=0.24\linewidth]{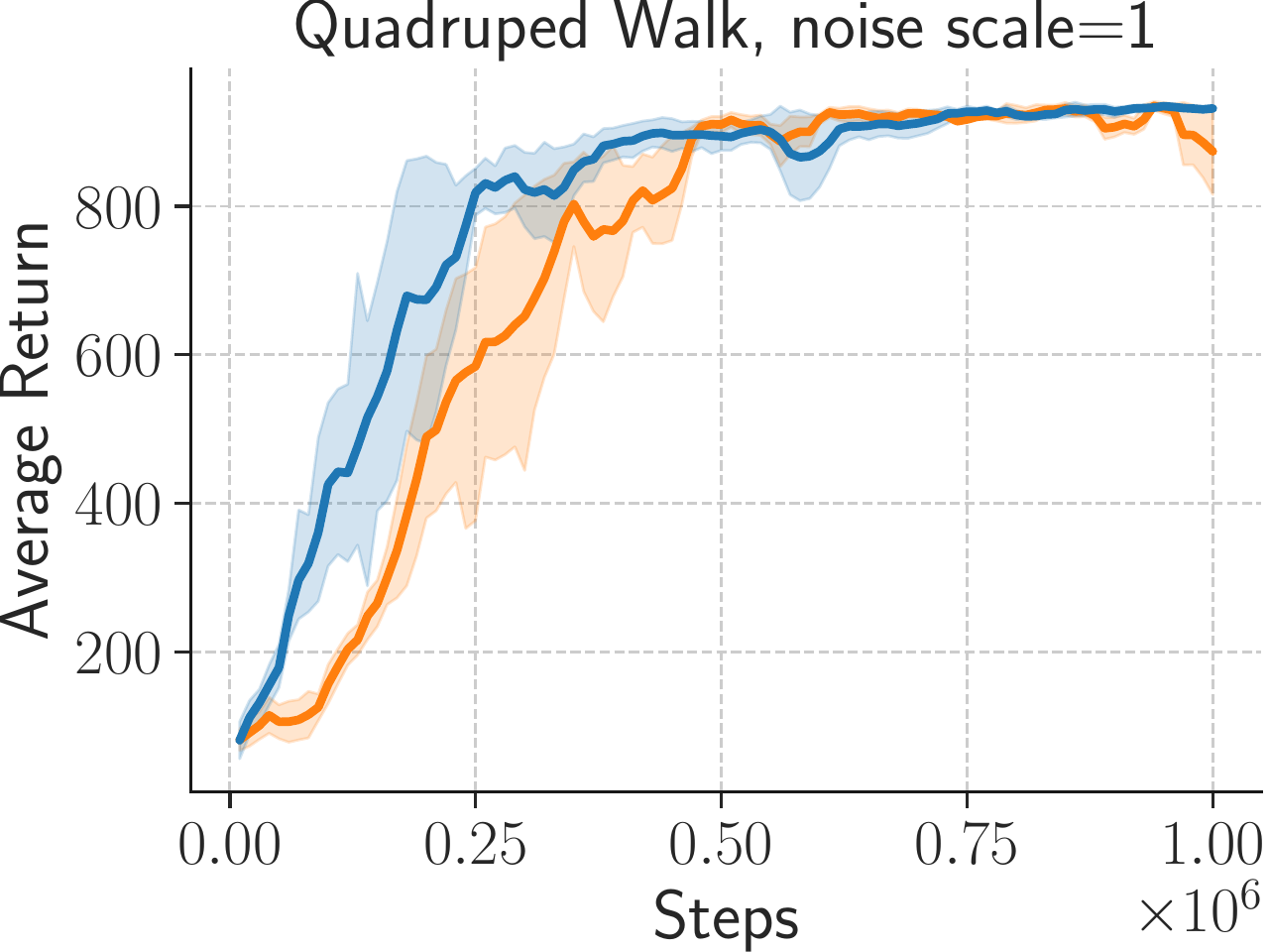} 
    \includegraphics[width=0.24\linewidth]{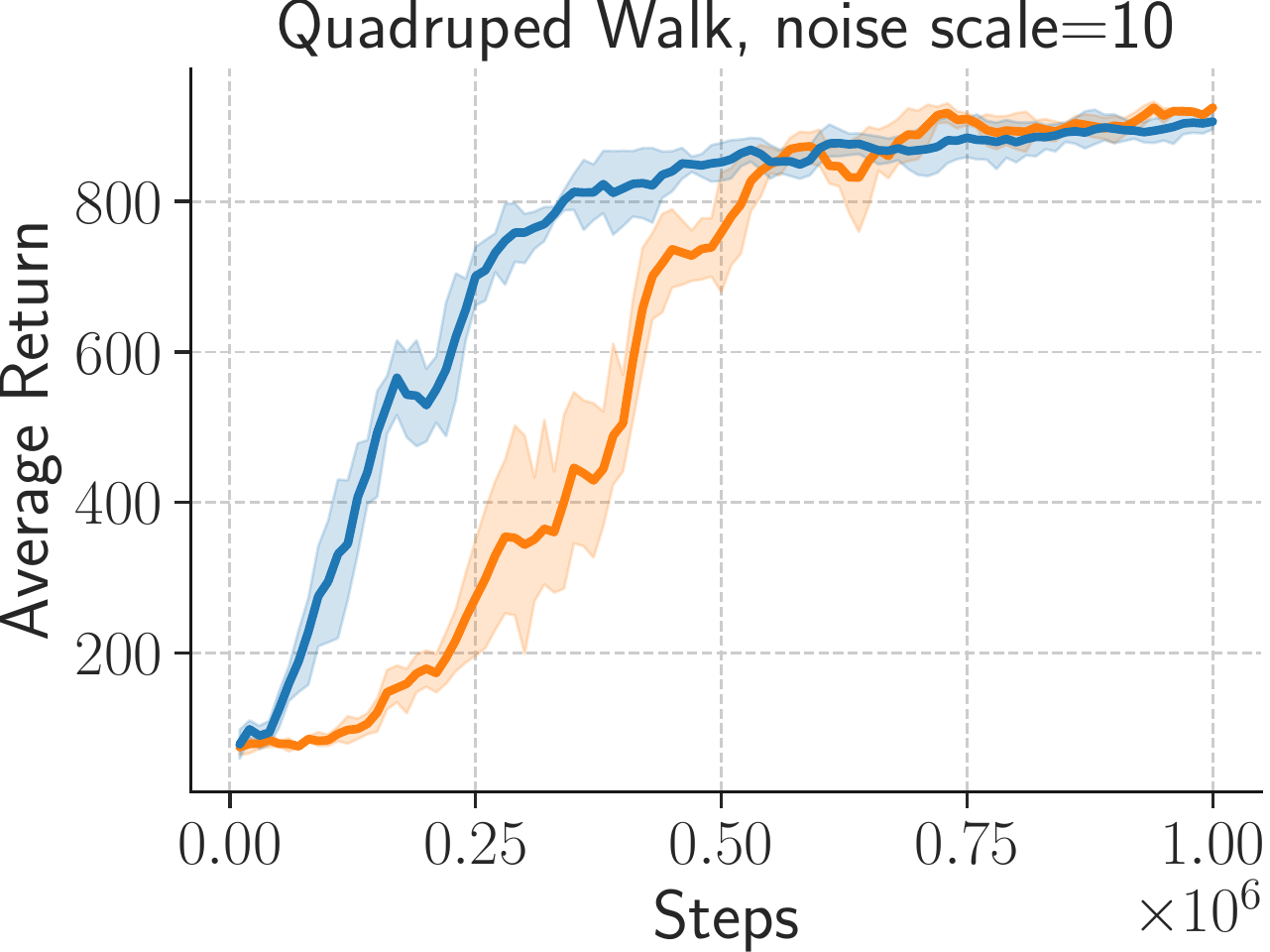} 
    \includegraphics[width=0.24\linewidth]{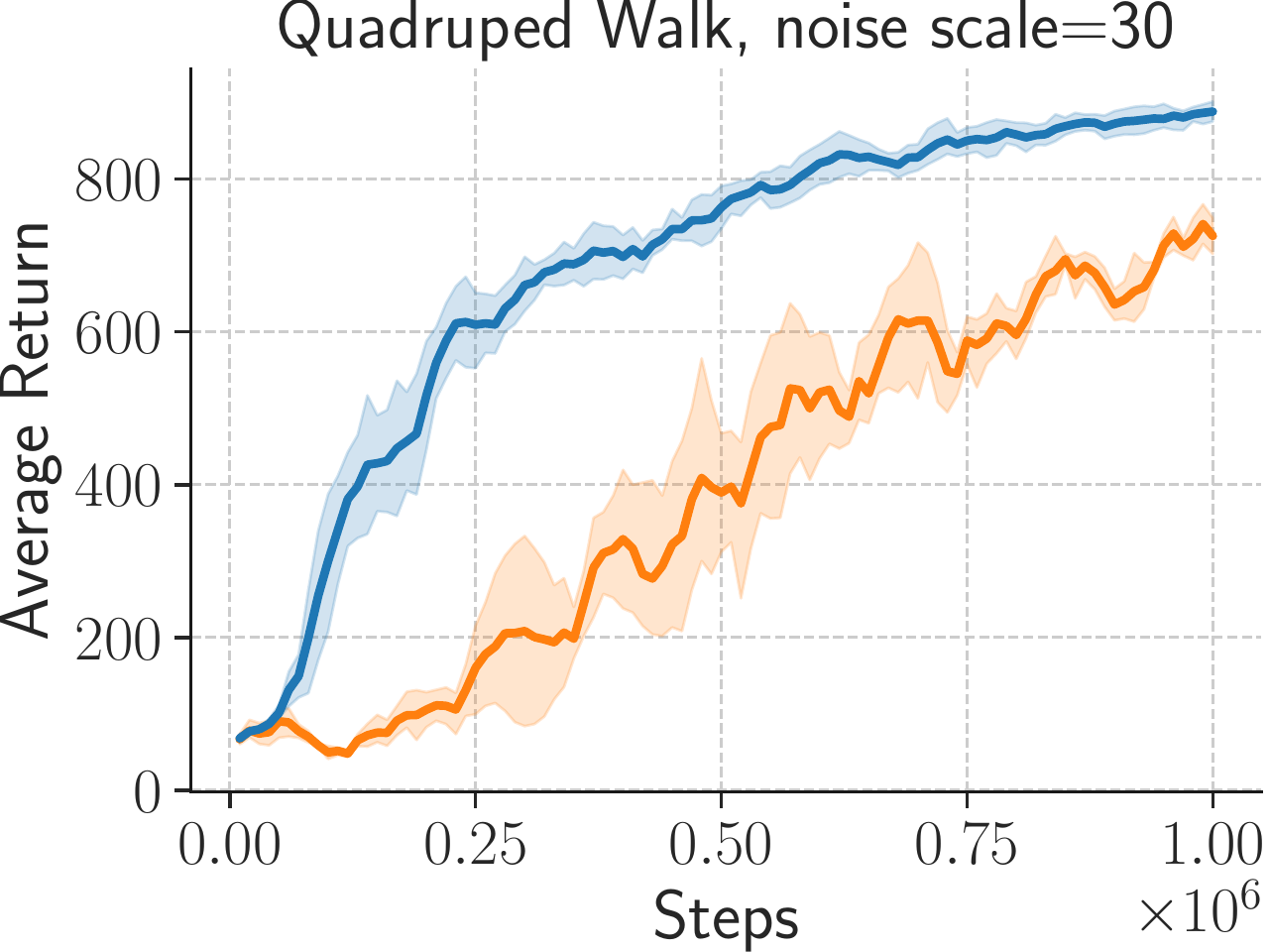} \\
    \caption{\textbf{State-based with Added Noise}: We add zero-mean Gaussian noise to the targets. As the standard deviation of the added noise increases, LFF maintains its performance better than MLPs.}
    \label{fig:small_added_noise}
\end{figure*}

\begin{figure*}
    \centering
    \includegraphics[width=0.24\linewidth]{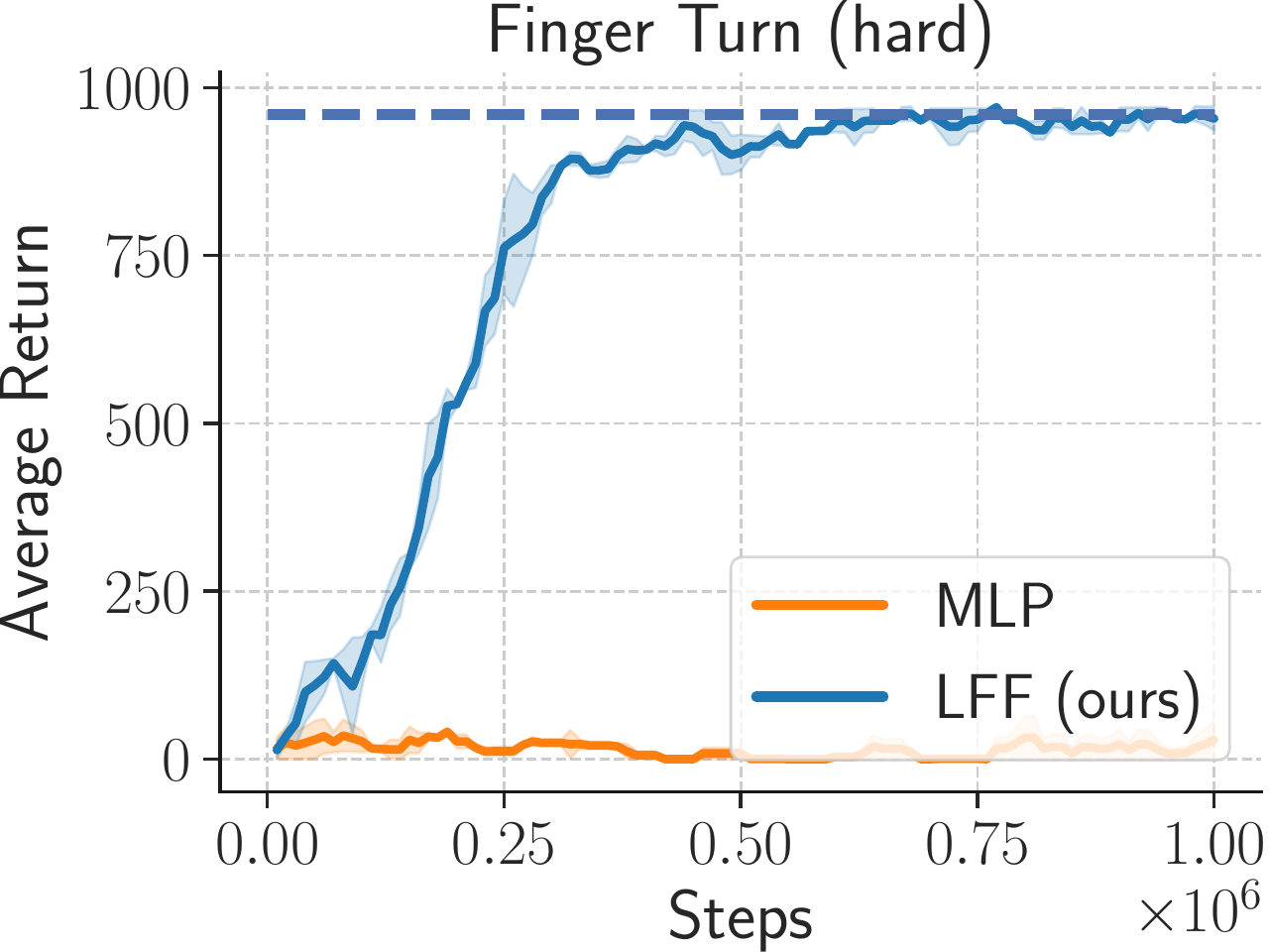}
    \includegraphics[width=0.24\linewidth]{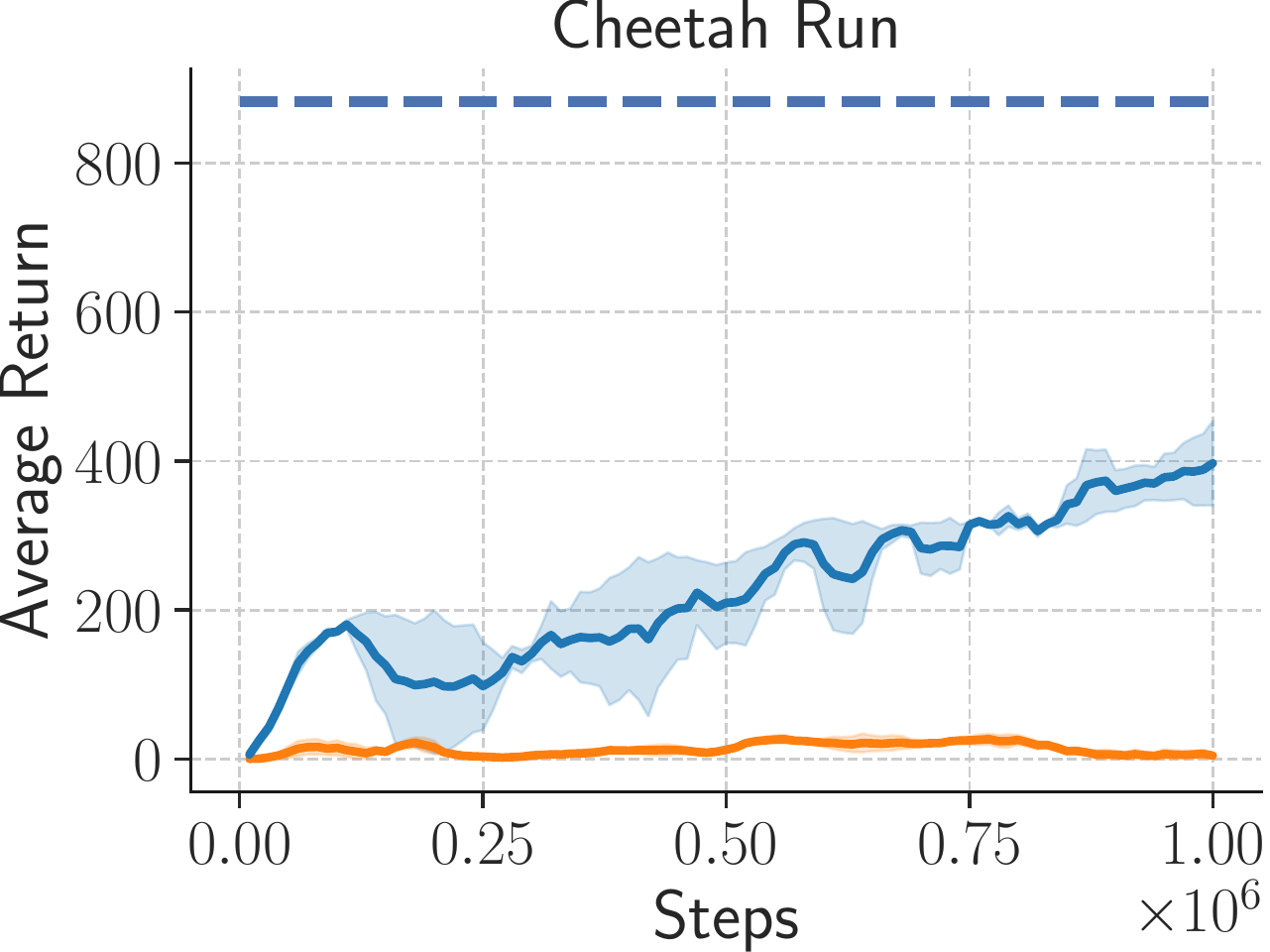}
    \includegraphics[width=0.24\linewidth]{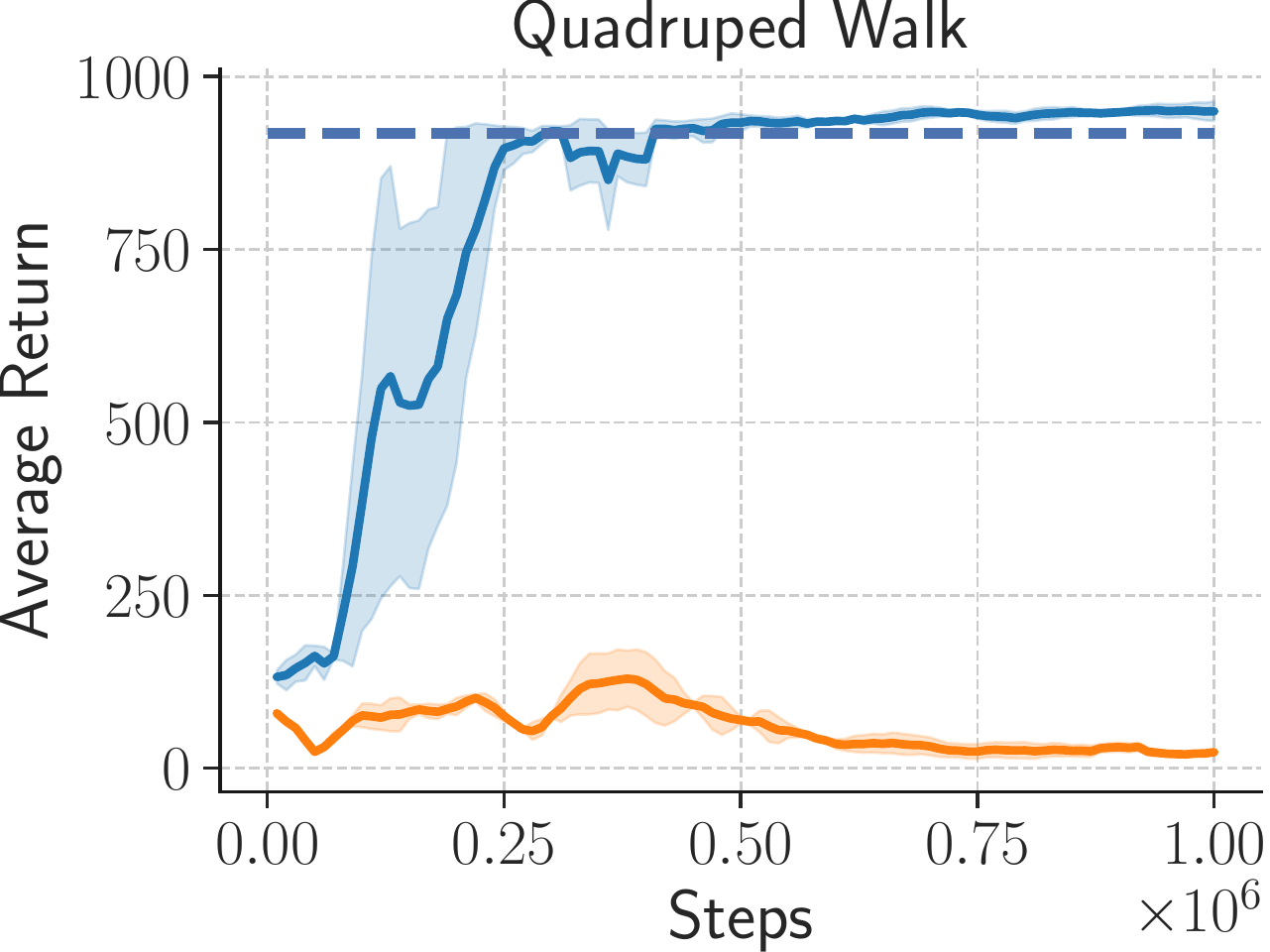}
    \includegraphics[width=0.24\linewidth]{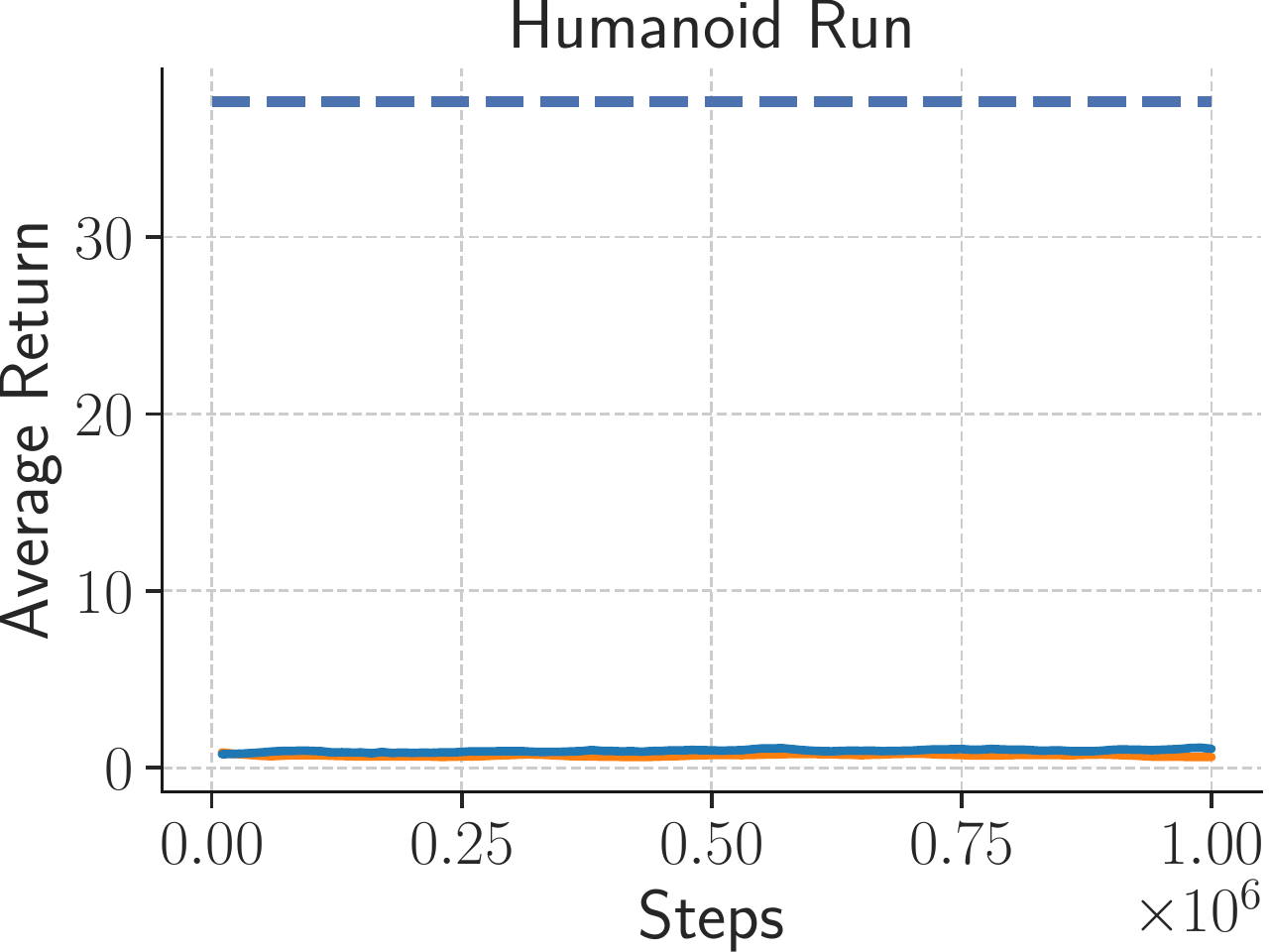}
    \caption{\textbf{Effect of LFF on Stability of Bootstrapping}: We train SAC, foregoing a target network, by bootstrapping directly from the Q-network being trained. The dashed line shows the LFF performance with a target network after 1M steps. We find that the LFF network is remarkably stable, and even learns faster on Quadruped Walk than when using target networks. However, LFF fails to learn on Humanoid, indicating that higher dimensional problems still pose problems.}
    \label{fig:no_target}
\end{figure*}
Our key problem is that standard ReLU MLP Q-functions tend to fit the noise in the target values, introducing error into the Q-function at some $(s, a)$. Bootstrapping with this incorrect Q-value to calculate target values for other $(s', a')$ yields even noisier targets, propagates the error to other states, and causes instability or divergence (see Appendix \ref{sec:noise_amplification} for more details). To further test whether LFF solves this problem by filtering out the noise, we train a SAC agent on state-based DMControl environments with either of the following modifications: testing the Q-function's robustness by adding Gaussian noise to the targets, or removing target networks altogether. 

\paragraph{Gaussian noise added to targets} In each bootstrap step, we add zero-mean Gaussian noise with standard deviation 1, 10, or 30 to the targets. LFF maintains higher performance even at significant noise levels, indicating that it is more robust to bootstrap noise. Full results are in Figure \ref{fig:added_noise}.

\paragraph{No target network}
Target networks, updated infrequently, slow down the propagation of noise due to bootstrapping \citep{mnih2013playing}. LFF fits less noise to begin with, so it should work even when the target network is omitted. Here, we bootstrap directly from the network being trained. 
Figure~\ref{fig:no_target} shows that MLPs consistently fail to learn on all environments in this setting, while the LFF architecture still performs well, except when the problem is very high dimensional. LFF even manages to learn faster in Quadruped Walk than it does when using a target network, since there is no longer Polyak averaging \citep{lillicrap2015continuous} with a target to slow down information propagation. Omitting the target network allows us to use updated values for $Q_\theta(s', a')$, instead of stale values from the target network. This result is in line with recent work that achieves faster learning by removing the target network and instead penalizing large changes to the Q-values \citep{shao2020grac}.

Overall, Figure \ref{fig:small_added_noise} and \ref{fig:no_target} validate our theoretical claims that LFF controls the effect of high-frequency noise on the learned function, and indicates that LFF successfully mitigates bootstrap noise in most cases. Tuning the SAC hyperparameters should increase LFF sample efficiency even further, since we can learn more aggressively when the noise problem is reduced. 

\subsection{Where Do Learned Fourier Features Help?}
\begin{figure*}
    \centering
    \includegraphics[width=0.24\linewidth]{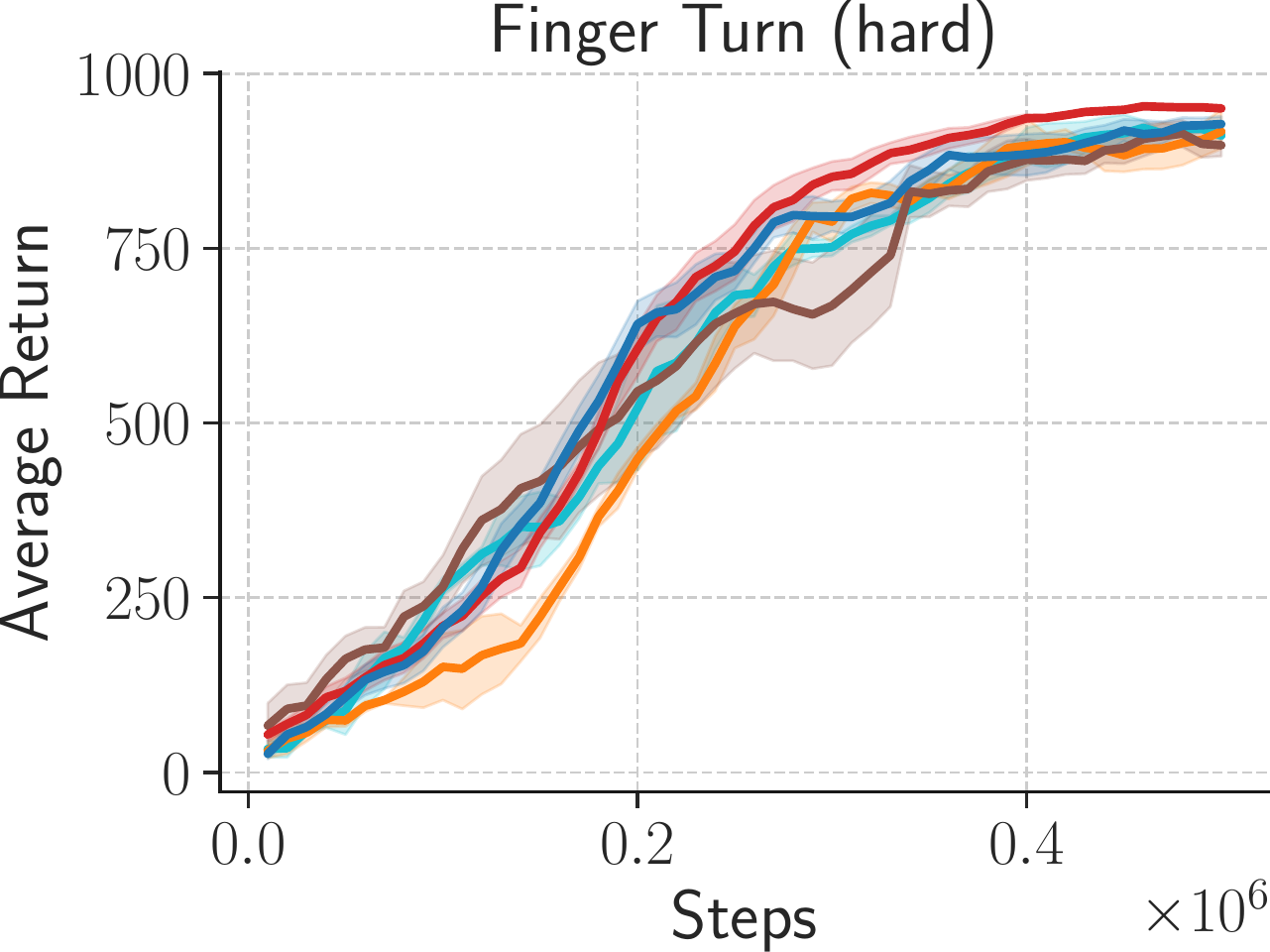}
    \includegraphics[width=0.24\linewidth]{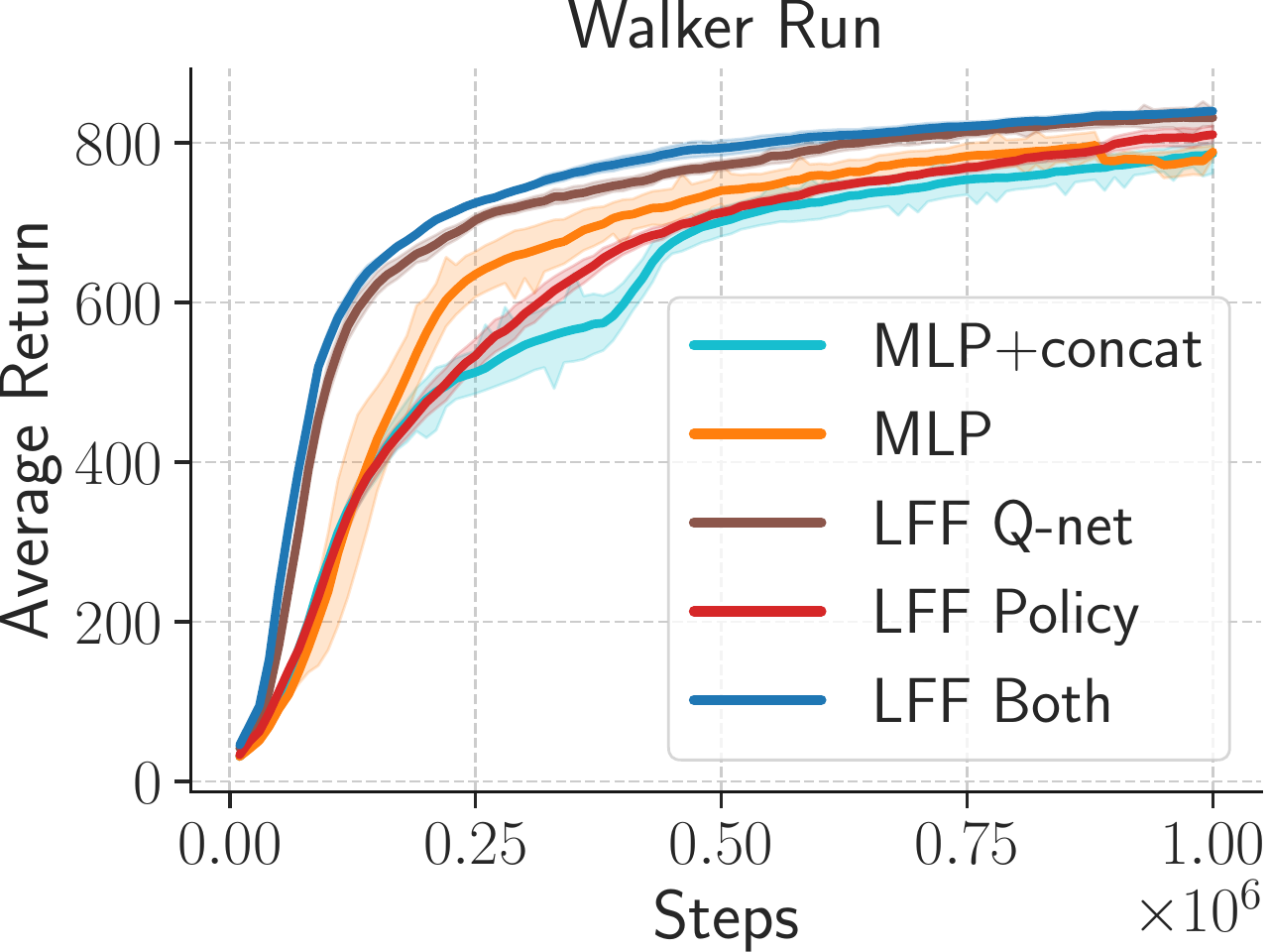} 
    \includegraphics[width=0.24\linewidth]{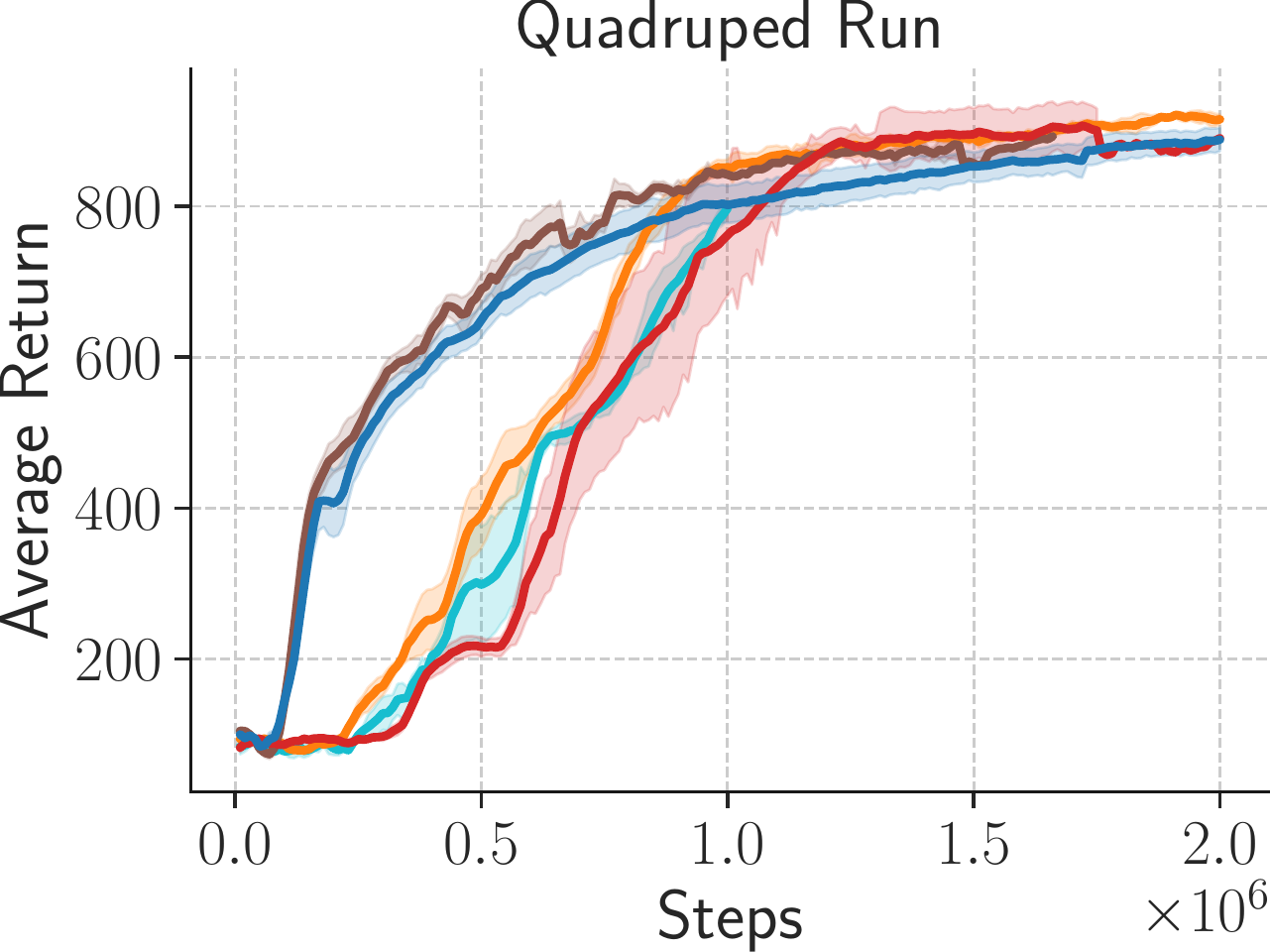}
    \includegraphics[width=0.24\linewidth]{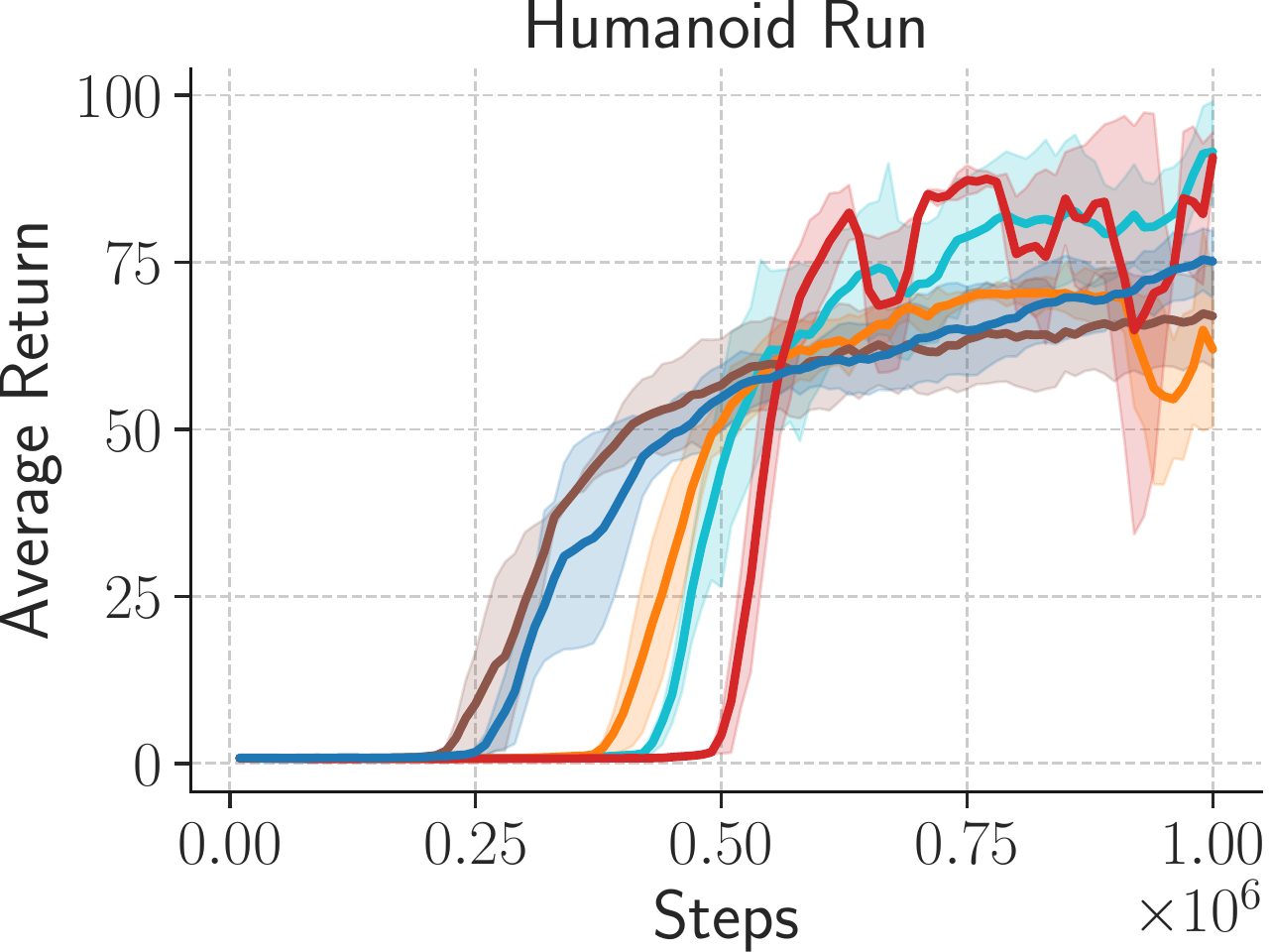}
    \caption{\textbf{LFF Policy vs Q-function.} Walker and Quadruped results indicate that only using LFF for the Q-network is just as good as using LFF for both networks. In contrast, using LFF for the policy network is about as bad as the MLP baseline. This suggests that LFF primarily improves off-policy learning by reducing noise in the Q-network optimization.}
    \label{fig:policy_vs_qf}
    \vspace{-0.5em}
\end{figure*}
In this section, we confirm that our LFF architecture improves RL performance by primarily preventing the Q-network from fitting noise. We train state-based SAC with an MLP policy and LFF Q-network (LFF Q-net in Figure~\ref{fig:policy_vs_qf}), or with an LFF policy and MLP Q-network (LFF Policy). Figure~\ref{fig:policy_vs_qf} shows that solely regularizing the Q-network is sufficient for LFF's improved sample efficiency. This validates our hypothesis that the Bellman updates remain noisy, even with tricks like double Q-networks and Polyak averaging, and that LFF reduces the amount of noise that the Q-network accumulates. These results suggest that separately tuning $\sigma$ for the Q-network and policy networks may yield further improvements, as they have separate objectives. The Q-network should be resilient to noise, while the policy can be fine-grained and change quickly between nearby states.
However, for simplicity, we use LFF for the Q-networks and the policy networks. 
Finally, we also train vanilla MLPs where we concatenate the input $x$ to the first layer output. LFF outperforms this variant, confirming that concatenation is not solely responsible for the improved sample efficiency. 

\subsection{Architectural Ablations}
\label{sec:ablations}
\begin{figure*}
    \centering
    \includegraphics[width=0.24\linewidth]{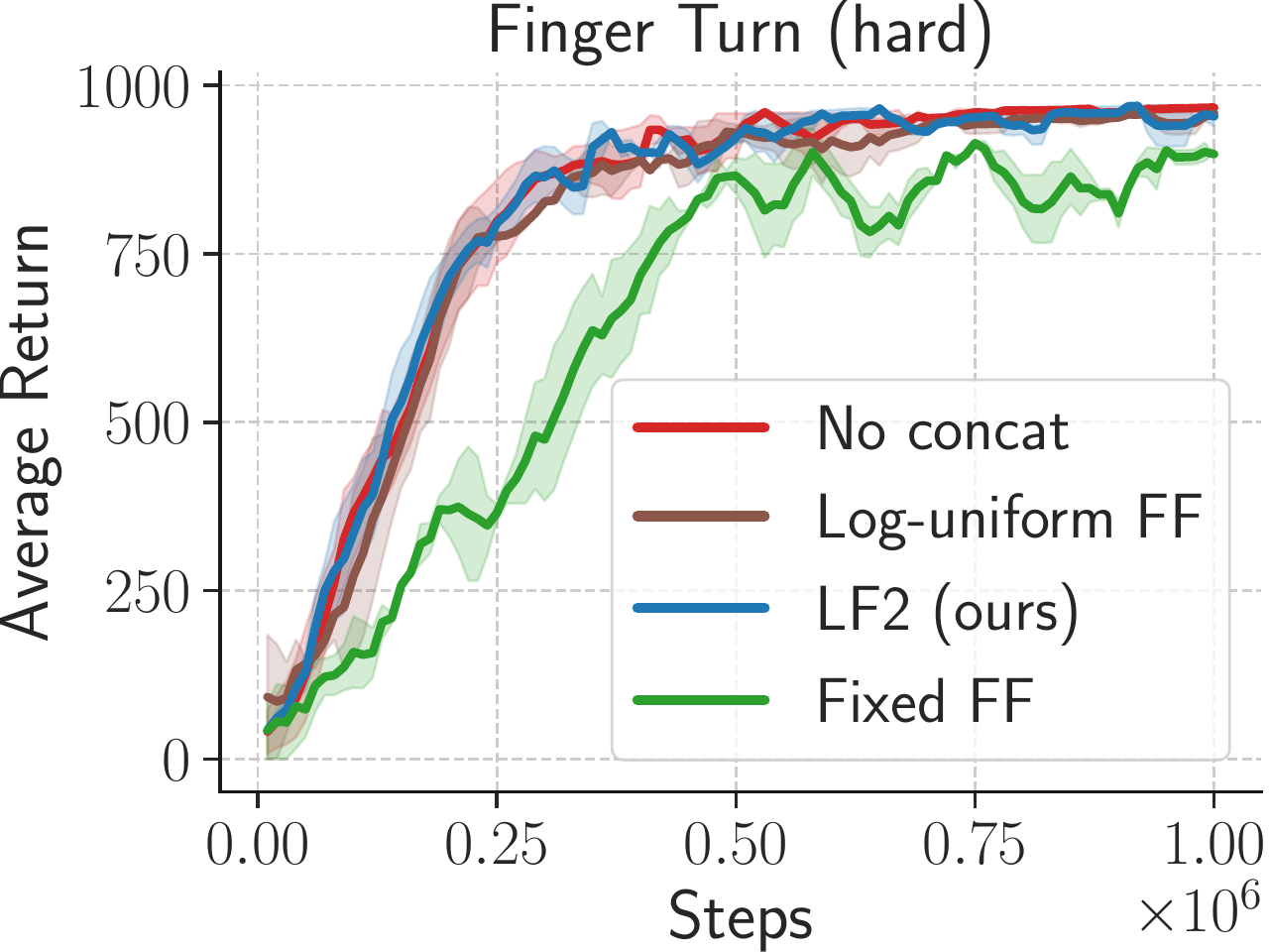}
    \includegraphics[width=0.24\linewidth]{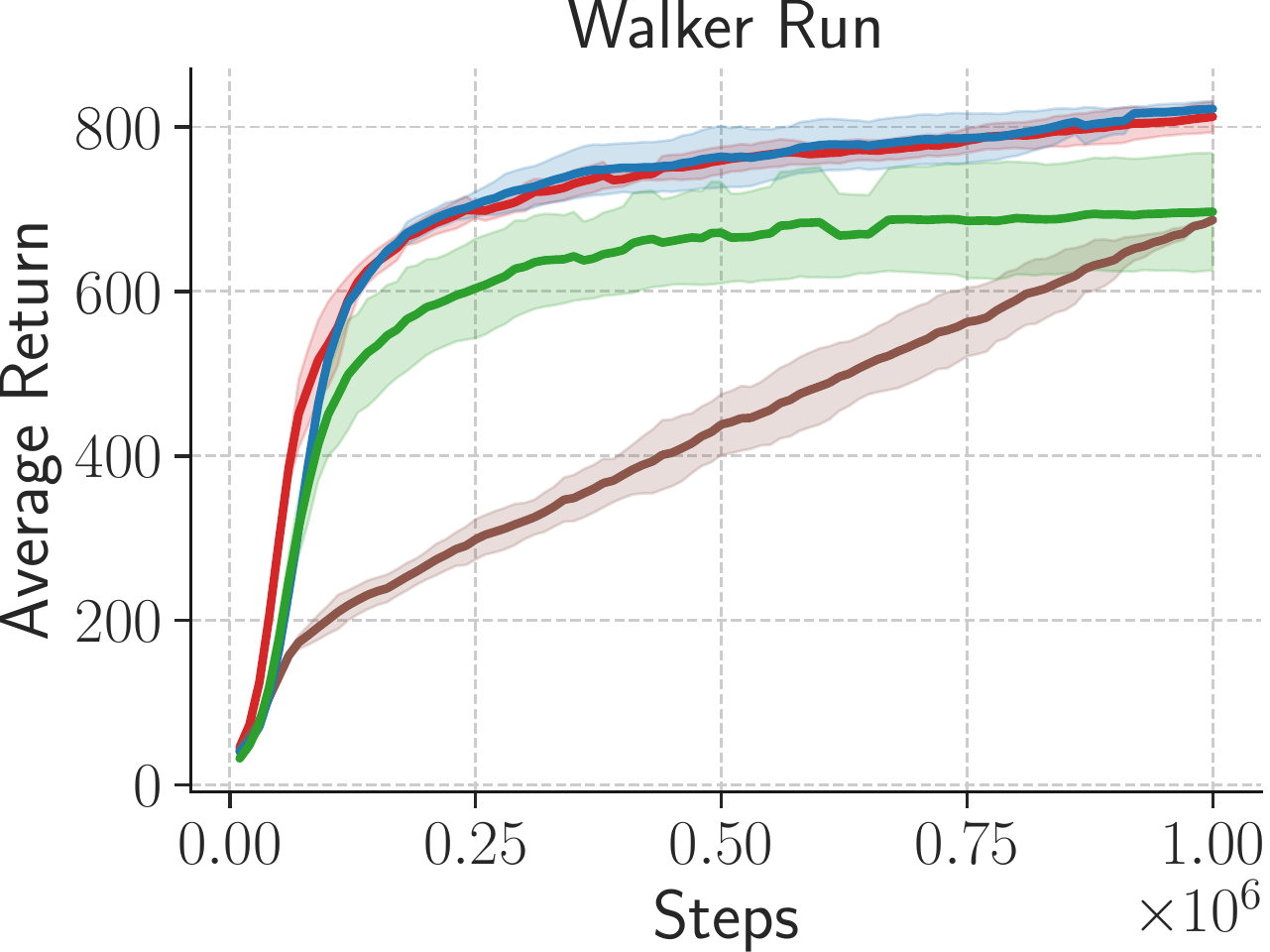}
    \includegraphics[width=0.24\linewidth]{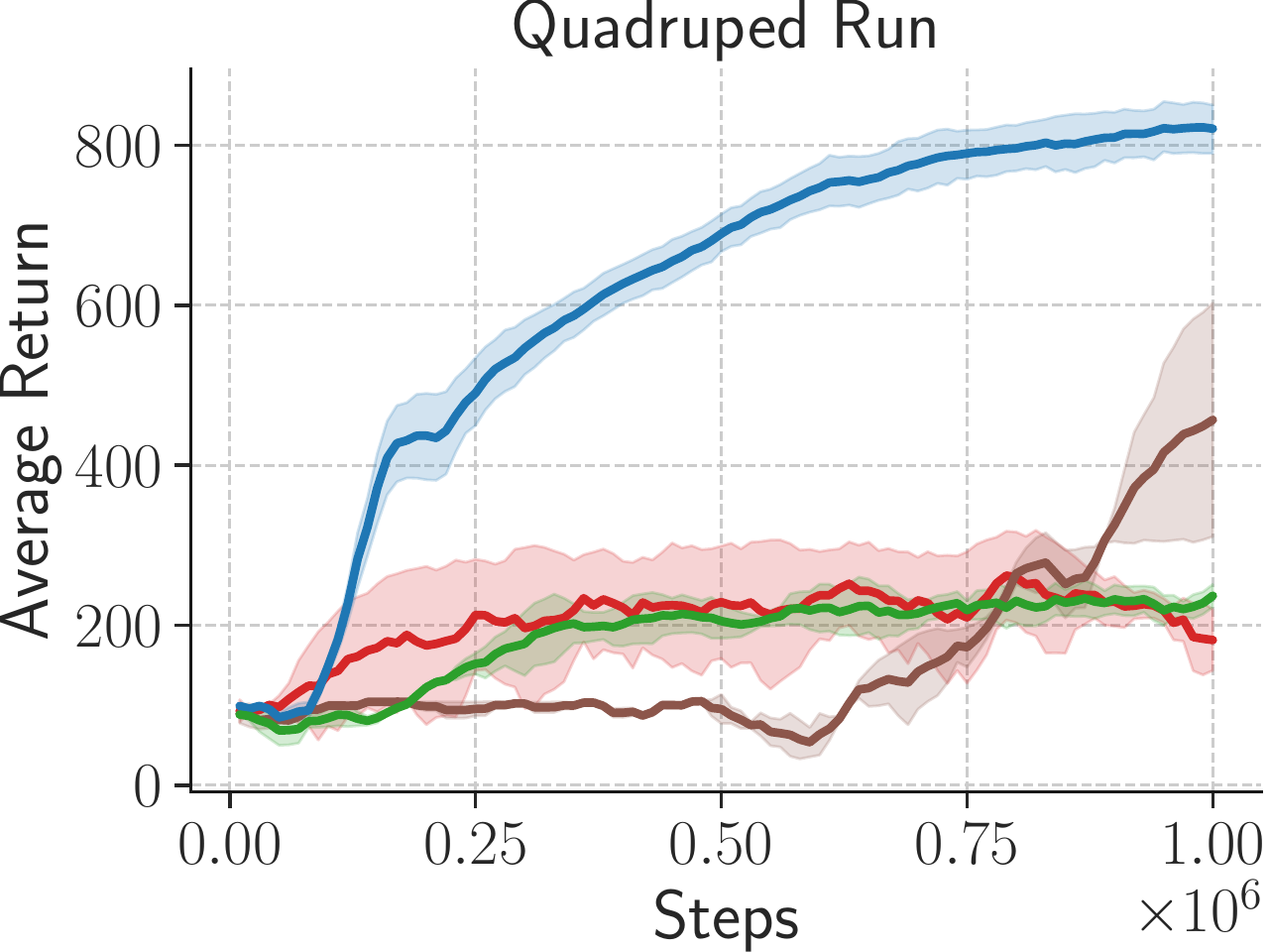}
    \includegraphics[width=0.24\linewidth]{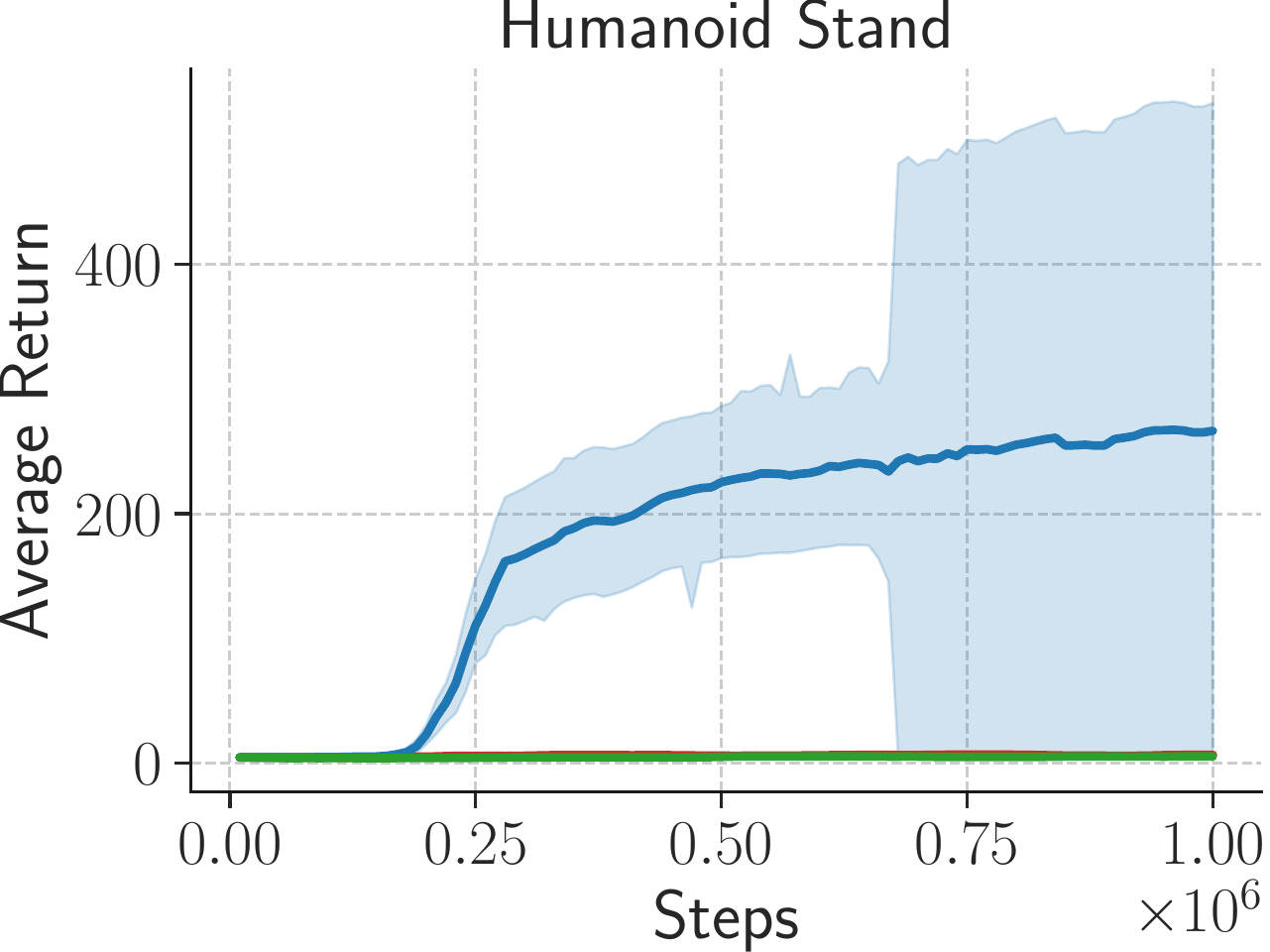}
    \caption{\textbf{Ablation analysis}: We train SAC on 4 DMControl environments with three variants of our architecture: LFF, LFF with fixed Fourier features, and LFF without input concatenation. Lower dimensional environments like Finger and Walker are more forgiving for fixed Fourier features or omitting input concatenation. However, Humanoid absolutely requires both modifications to learn.}
    \label{fig:ablation}
    \vspace{-1em}
\end{figure*}

LFF features two key improvements: learning the Fourier feature basis $B$ and concatenating the input $x$ to the Fourier features. We perform an ablation on several DMControl environments with SAC in Figure~\ref{fig:ablation} to investigate the impact of these modifications.

We first find that training the Fourier feature basis $B$ is critical. Across all four environments, learning is impacted by using a fixed, randomly initialized $B$. This is because finding the right Fourier features at initialization is unlikely in our high dimensional RL problems. Training $B$ allows the network to discover the relevant Fourier features on its own. The relationship with dimension is clear: as the input dimension increases, the performance gap between LFF and fixed Fourier features grows.

Concatenating the input $x$ to the Fourier features is also important. It maintains all of the information that was present in $x$, which is critical in very high dimensional environments. 
If the Fourier basis $B$ is poorly initialized and it blends together or omits important dimensions of $x$, the network takes a long time to disentangle them, if at all. This problem becomes more likely as the observation dimension increases. While LLF can learn without concatenation in low-dimensional environments like Walker and Finger, it has a much harder time learning in Quadruped and Humanoid. 

Finally, we test an alternative approach to initializing the values of $B$. $B$, which has shape $(k \cdot d_\text{input}) \times d_\text{input}$, is now initialized as $B = (I, cI, c^2I, \dots, c^{k-1} I)^\top$ where $I$ is the identity matrix, $k$ is an integer, and $0 < c < 1$ is a tuned multiplier. This parallels the axis-aligned, log-uniform spacing used in NeRF's positional encoding~\citep{mildenhall2020nerf}, but we additionally concatenate $x$ and train $B$. We find that this initialization method, dubbed ``Log-uniform FF'' in Figure~\ref{fig:ablation}, is consistently worse than sampling from $\mathcal{N}(0, \sigma^2)$. This is likely because the initialization fails to capture features that are not axis-aligned, so most of the training time is used to discover the right combinations of input features.


\section{Conclusions and Future Work}
We highlight that the standard MLP or CNN architecture in state-based and image-based deep RL methods remain susceptible to noise in the Bellman update. 
To overcome this, we proposed embedding the input using \textit{learned} Fourier features. We show both theoretically and empirically that this encoding enables fine-grained control over the network's frequency-specific learning rate. Our LFF architecture serves as a plug-and-play addition to any state-of-the-art method and leads to consistent improvement in sample efficiency on standard state-space and image-space environments.

One shortcoming of frequency-based regularization is that it does not help when the noise and the signal look similar in frequency space. Future work should examine when this is the case, and test whether other regularization methods are complementary to LFF.  
Another line of work, partially explored in Appendix \ref{sec:high_freq}, is using LFF with large $\sigma$ to fit high frequencies and reduce underfitting in model-based or tabular reinforcement learning scenarios. 
We hope this work will provide new perspectives on existing RL algorithms for the community to build upon.

\paragraph{Acknowledgments}
We thank Shikhar Bahl, Murtaza Dalal, Wenlong Huang, and Aravind Sivakumar for comments on early drafts of this paper. The work was supported in part by NSF IIS-2024594 and GoodAI Research Award. AL is supported by the NSF GRFP. This material is based upon work supported by the National Science Foundation Graduate Research Fellowship Program under Grant No. DGE1745016 and DGE2140739. Any opinions, findings, and conclusions or recommendations expressed in this material are those of the author(s) and do not necessarily reflect the views of the National Science Foundation

\bibliographystyle{abbrvnat}
\bibliography{main}

{
\small
}


\newpage
\appendix
\section{Theoretical Analysis}
\label{sec:proofs}
\subsection{Functional Convergence Rate}
\label{subsec:convergence_sketch}
We provide Lemma~\ref{lemma:training_dynamics} and a quick proof sketch as background for readers who are not familiar with the neural tangent kernel literature. 
\begin{lemma}
\label{lemma:training_dynamics}
When training an infinite width neural network via gradient flow on the squared error, the training residual at time $t$ is: 
\begin{align}
\label{eq:convergence}
    f_{\theta_t}(x) - y = e^{-\eta Kt} (f_{\theta_0}(x) - y)
\end{align}
where $ f_{\theta_t}(x)$ is the column vector of model predictions for all $x_i$, $y$ is the column vector of stacked training labels, $\eta$ is the multiplier for gradient flow, and $K$ is the NTK kernel matrix with $K_{ij} = \langle  \nabla_\theta f_{\theta_0}(x_i),  \nabla_\theta f_{\theta_0}(x_j) \rangle$.
\end{lemma}
\begin{proof}
The squared error $L(\theta, x) = \| f_\theta(x) - y\|_2^2$ has gradient:
\begin{align}
    \nabla_\theta L(\theta, x) = \nabla_\theta f_\theta(x)^\top (f_\theta(x) - y)
\end{align}
Since we train with gradient flow, the parameters change at rate:
\begin{align}
    \frac{d\theta_t}{dt} &= -\eta \nabla_{\theta_t} L(\theta_t, x) \\
                         &= -\eta \nabla_{\theta_t} f_{\theta_t}(x)^\top (f_{\theta_t}(x) - y)
\end{align}
By the chain rule,  
\begin{align}
    \frac{df_{\theta_t}(x)}{dt} &= \frac{df_{\theta_t}(x)}{d\theta_t} \frac{d\theta_t}{dt} \\
                 &= -\eta \nabla_{\theta_t} f_{\theta_t}(x) \nabla_{\theta_t} f_{\theta_t}(x)^\top (f_{\theta_t}(x) - y) \\
                 &= -\eta K_{\theta_t} (f_{\theta_t}(x) - y)
\end{align}
where $K$ is the NTK kernel matrix at time $t$, with entries $K_{ij} = \langle \nabla_{\theta_t} f_{\theta_t}(x_i), \nabla_{\theta_t} f_{\theta_t}(x_j)\rangle$.
Since $y$ is a constant, and $K_{\theta_t} \approx K_{\theta_0} \triangleq K$ due to the infinite-width limit, we can write this as:
\begin{align}
    \frac{d(f_{\theta_t}(x) - y)}{dt} &= -\eta K (f_{\theta_t}(x) - y)
\end{align}
This is a well-known differential equation, with closed form solution: 
\begin{align}
    f_{\theta_t}(x) - y = e^{-\eta Kt} (f_{\theta_0}(x) - y)
\end{align}
\end{proof}

\subsection{Eigenvalues of the NTK matrix and the Discrete Fourier Transform}
\label{subsec:dft_proof}
\begin{lemma}
\label{lemma:circulant}
A circulant matrix $C \in \mathbb{R}^{n \times n}$ with first row $(c_0, \dots, c_{n-1})$ has eigenvectors $\{x^{(k)}\}_{k=1}^n$ corresponding to the column vectors of the DFT matrix:
\begin{align}
x^{(k)} = (\omega_n^{0k}, \omega_n^{1k}, \dots, \omega_n^{(n-1)k})^\top
\end{align}
where $\omega_n = e^{\frac{2\pi i}{n}}$ is the $n$th root of unity.  
The corresponding eigenvalue is the $k$th DFT value $\lambda^{(k)} = DFT(c_0, \dots, c_{n-1})_k$.
\end{lemma}
\begin{proof}
This is a well-known property of circulant matrices \citep{bamieh2018discovering}. Nevertheless, we provide a simple proof here. 
First, let's make clear the structure of the circulant matrix $C$:
\begin{align}
    C = \begin{bmatrix}
    c_0 & c_1 & c_2 & \dots &c_{n-1} \\
    c_{n-1} & c_0 & c_1 & \dots & c_{n-2} \\
    c_{n-2} & c_{n-1} & c_0 & \dots & c_{n-3} \\
    \vdots & \vdots & \vdots & \ddots & \vdots \\
    c_1 & c_2 & c_3 & \dots & c_0
    \end{bmatrix}
\end{align}
Again, we want to show that an eigenvector of $C$ is $x^{(k)}$, the $k$th column of the DFT matrix $F$.
\begin{align}
x^{(k)} = (\omega_n^{0k}, \omega_n^{1k}, \dots, \omega_n^{(n-1)k})^\top
\end{align}
where $\omega_n = e^{\frac{2\pi i}{n}}$ is the $n$th root of unity. Note that for all positive integers $k$, $\omega_n^{jk}$ is periodic in $j$ with period $n$. This is because:
\begin{align}
    \omega_n^{nk} &=  e^{\frac{2\pi nki}{n}} \\
                  &= \cos(2\pi k) + i \sin(2\pi k) \\
                  &= 1
\end{align}

Now, let us show that $x^{(k)}$ is an eigenvector of $C$. Let $y = Cx^{(k)}$. The $i$th element of $y$ is then
\begin{align}
    y_i &= \sum_{j=0}^{n-1}c_{j-i}\omega_n^{jk} \\
        &= \omega_n^{ik} \sum_{j=0}^{n-1}c_{j-i}\omega_n^{(j-i)k}
\end{align}
The remaining sum does not depend on $i$, since $c_{j-i}$ and $\omega_n^{(j-i)k}$ are periodic with period $n$. This means we can rearrange the indices of the sum to get: 
\begin{align}
    y_i &= \omega_n^{ik} \sum_{j=0}^{n-1}c_{j}\omega_n^{jk} \\
        &= x^{(k)}_i \lambda_k
\end{align}
where $\lambda_k = \sum_{j=0}^{n-1}c_{j}\omega_n^{jk}$ is exactly the $k$th term in the DFT of the signal $(c_0, c_1, \dots, c_{n-1})$. Thus, $Cx^{(k)} = \lambda_k x^{(k)}$, so $x^{(k)}$ is an eigenvector of $C$ with corresponding eigenvalue $\lambda^{(k)} = DFT(c_0, \dots, c_{n-1})_k$.
\end{proof}

\subsection{Proof of Lemma 1: NTK for 2 layer Fourier feature model}
\label{subsec:2layer_proof}

\begin{proof}
The gradient consists of two parts: the gradient with respect to $B$ and the gradient with respect to $W$. We can calculate the NTK for each part respectively and then sum them. \\ 

\textbf{For $W$:}
\begin{align}
    \nabla_W f(x) = \sqrt{\frac{2}{m}} \begin{bmatrix}
    \sin(Bx) \\
    \cos(Bx)
    \end{bmatrix}
\end{align}
The width-$m$ kernel is then:
\begin{align}
\label{eq:summand}
    k_m^W(x, x') &= \frac{2}{m} \sum_{i=1}^{m/2} \cos(b_i^\top x) \cos(b_i^\top x') + \sin(b_i^\top x) \sin(b_i^\top x') 
\end{align}
Using the angle difference formula, this reduces to: 
\begin{align}
    k_m^W(x, x') &= \frac{2}{m} \sum_{i=1}^{m/2} \cos(b_i^\top(x-x'))
\end{align}
As $m \rightarrow \infty$, this converges to a deterministic kernel $k^W(x, x') = \mathbb{E}_{B_{ij} \sim \mathcal N(0, \tau)}[\cos(b_i^\top (x - x'))]$.

Using the fact that $\mathbb \mathbb{E}_{X \sim \mathcal N(0, \Sigma)}[\cos(t^\top X)] = \exp\{-\frac{1}{2}t^\top \Sigma t\}$ for fixed vector $t$ and the fact that $\Sigma = \text{diag}(\sigma^2, \sigma^2)$ in our case, this kernel function simplifies to: 
\begin{align}
    k^W(x, x') = \exp\left\{-\frac{\sigma^2}{2}\|x - x'\|_2^2\right\}
\end{align}
\textbf{For $B$:}
\begin{align}
    \nabla_{b_i}f(x) = \sqrt{\frac{2}{m}}\left(W_i \cos(b_i^\top x)  - W_{i+m/2} \sin(b_i^\top x) \right) x
\end{align}
The width-$m$ kernel for $B$ is then:
\begin{align}
\begin{split}
    k_m^B(x, x') = \frac{2 x^\top x'}{m} \sum_{i=1}^{m/2} &W_i^2 \left( \cos(b_i^\top x) \cos(b_i^\top x')\right) + W_{i+m/2}^2 \left(\sin(b_i^\top x) \sin(b_i^\top x) \right) \\
    &- W_i W_{i+m/2} \left(\cos(b_i^\top x)\sin(b_i^\top x') + \sin(b_i^\top x) \cos(b_i^\top x')\right)
\end{split}
\end{align}
As we take $m \rightarrow \infty$, recall that $W_j \sim \mathcal N (0, 1)$ i.i.d.. Thus, $\mathbb{E}[W_j^2] = 1$ and $\mathbb{E}[W_j W_k] = 0$ for $j \neq k$. The NTK is then
\begin{align}
    k^B(x, x') = x^\top x' \mathbb{E}\left[ \cos(b_i^\top x) \cos(b_i^\top x') + \sin(b_i^\top x) \sin(b_i^\top x)\right]
\end{align}
The interior of the expectation is exactly the same as the summand in Equation \ref{eq:summand}. Following the same steps, we get the simplified kernel function: 
\begin{align}
    k^B(x, x') = x^\top x' \exp\left\{-\frac{\sigma^2}{2}\|x - x'\|_2^2\right\}
\end{align}

Finally, our overall kernel function $k(x, x') = k^W(x, x') + k^B(x, x')$ is:
\begin{align}
    k(x, x') = \left(1 + x^\top x' \right) \exp\left\{-\frac{\sigma^2}{2}\|x - x'\|_2^2\right\}
\end{align}
Since $x, x' \in \mathbb{S}^{d-1}$, they have unit norm, with $\|x - x'\|_2^2 = 2(1 - x^\top x') = 2(1 - \cos \theta)$. This gives us two equivalent forms of the NTK: 
\begin{align}
    k(x, x') &= \left(2 - \frac{\|x - x'\|_2^2}{2} \right) \exp\left\{-\frac{\sigma^2}{2}\|x - x'\|_2^2\right\} \\
             &= \left(1 + \cos \theta \right) \exp\left\{ \sigma^2(\cos \theta - 1) \right\}
\end{align}
\end{proof}

\subsection{When is the Bellman Update a Contraction?}

Here, we examine LFF's stability under Bellman updates by using results from \citet{achiam2019towards}, who used the NTK approximation to prove that
the Bellman update is a contraction in finite MDPs if 
\begin{align}
\label{eq:conditions}
    \forall i, \quad &\alpha K_{ii} \rho_i < 1 \\
\label{eq:conditions2}
    \forall i, \quad &(1+\gamma) \sum_{j \neq i}|K_{ij}|\rho_j \leq (1-\gamma) K_{ii} \rho_i
\end{align}
where $K$ is the NTK of the Q-network and $\rho_i$ is the density of transition $i$ in the replay buffer. 
Intuitively, $K_{ij}$ measures the amount that a gradient update on transition $j$ affects the function output on transition $i$. In order for the Bellman update to be a contraction, the change at $(s_i, a_i)$ due to gradient contributions from all other transitions should be relatively small. This suggests that LFF, with very large $\sigma$, could fulfill the conditions in Equation \ref{eq:conditions} and \ref{eq:conditions2}. We formalize this in Theorem \ref{theorem:contraction}:
\begin{theorem}
\label{theorem:contraction}
For a finite MDP with the state-action space as a finite, uniform subset of $\mathcal S^d$, and when we have uniform support for each transition in the replay buffer, the Bellman update on a 2 layer LFF architecture is a contraction for suitably small learning rate $\alpha$.
\end{theorem}
\begin{proof}
We need our kernel to satisfy two conditions \citep{achiam2019towards}:
\begin{align}
\label{eq:conditions}
    \forall i, \quad &\alpha K_{ii} \rho_i < 1 \\
\label{eq:conditions2}
    \forall i, \quad &(1+\gamma) \sum_{j \neq i}|K_{ij}|\rho_j \leq (1-\gamma) K_{ii} \rho_i
\end{align}
Equation~\ref{eq:conditions} is easy to satisfy, as we assumed that all transitions appear uniformly in our buffer, and we know from Lemma 1 that $k(x, x) = 2$. Thus, we only need to make the learning rate $\alpha$ small enough such that $\alpha < \frac{1}{2\rho_i}$. 

For Equation~\ref{eq:conditions2}, we can prove a loose lower bound on the variance $\sigma^2$ that is required for the Bellman update to be a contraction. As stated in the main text, we assume that we have $N+1$ datapoints $x_i$ that are distributed approximately uniformly over $\mathbb{S}^{d-1}$. Here, uniformly simply implies that we have an upper bound $x_i^\top x_j = \cos \theta < 1 - \delta$ for all $i \neq j$ and fixed positive $\delta > 0$. 

First, we bound $|K_{ij}|$ for all $i \neq j$. Using the expression from Lemma 1 and our upper bound on $\cos \theta$, we have $|K_{ij}| \leq (2-\delta) \exp \left\{ -\delta\sigma^2) \right\}, \forall i \neq j$. Then, plugging this into Equation~\ref{eq:conditions2} and cancelling the buffer frequencies $\rho_i$, which are equal by assumption,
\begin{align}
    N (1 + \gamma) (2-\delta) \exp\left\{ -\delta \sigma^2 \right\} \leq 2(1-\gamma)
\end{align}
Rearranging gives us: 
\begin{align}
    \sigma^2 \geq \frac{1}{\delta} \log \frac{N(1+\gamma)(2-\delta)}{2(1-\gamma)}
\end{align}
As long as $\alpha$ is small enough and $\gamma < 1$, an infinite-width LFF architecture initialized with a suitably large $\tau$ will enjoy Bellman updates that are always contractions in the sup-norm. 
\end{proof}
\paragraph{Note} However, this does not align with what we see in practice, where small $\sigma$ yields the best performance, and increasing $\sigma$ leads to worse performance. This is because \citet{achiam2019towards} makes the assumption that the replay buffer assigns positive probability to every possible transition, which is impossible in our continuous MDPs. We also do stochastic optimization with minibatches, which further deviates from the theory. Finally, guarantee of a contraction does not imply sample efficiency. Indeed, an algorithm that contracts slowly at every step can be much worse than an algorithm that greatly improves in expectation.

\newpage
\section{On-policy Results}
\begin{figure}[H]
    \centering
    \includegraphics[width=0.24\linewidth]{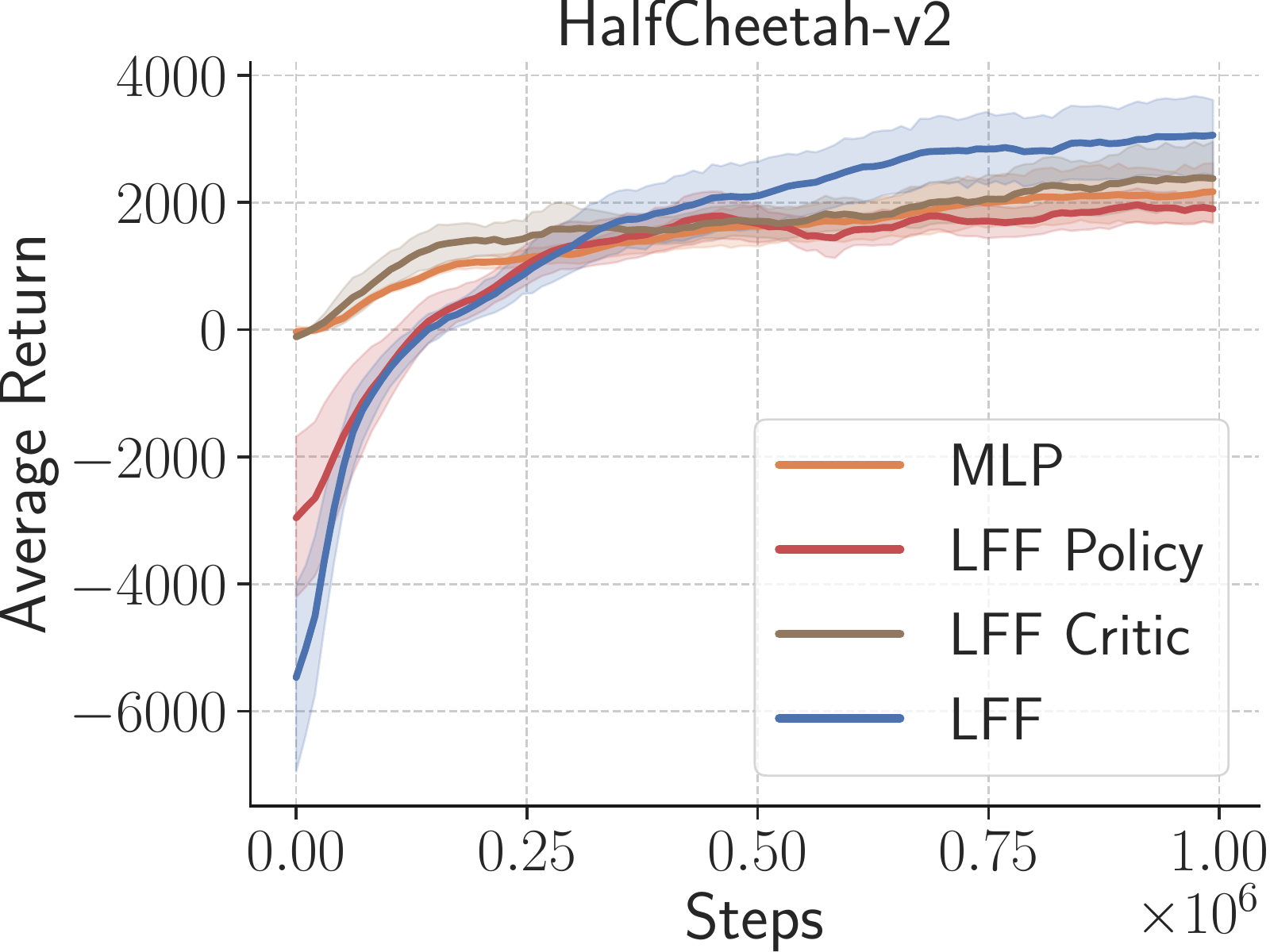}
    \includegraphics[width=0.24\linewidth]{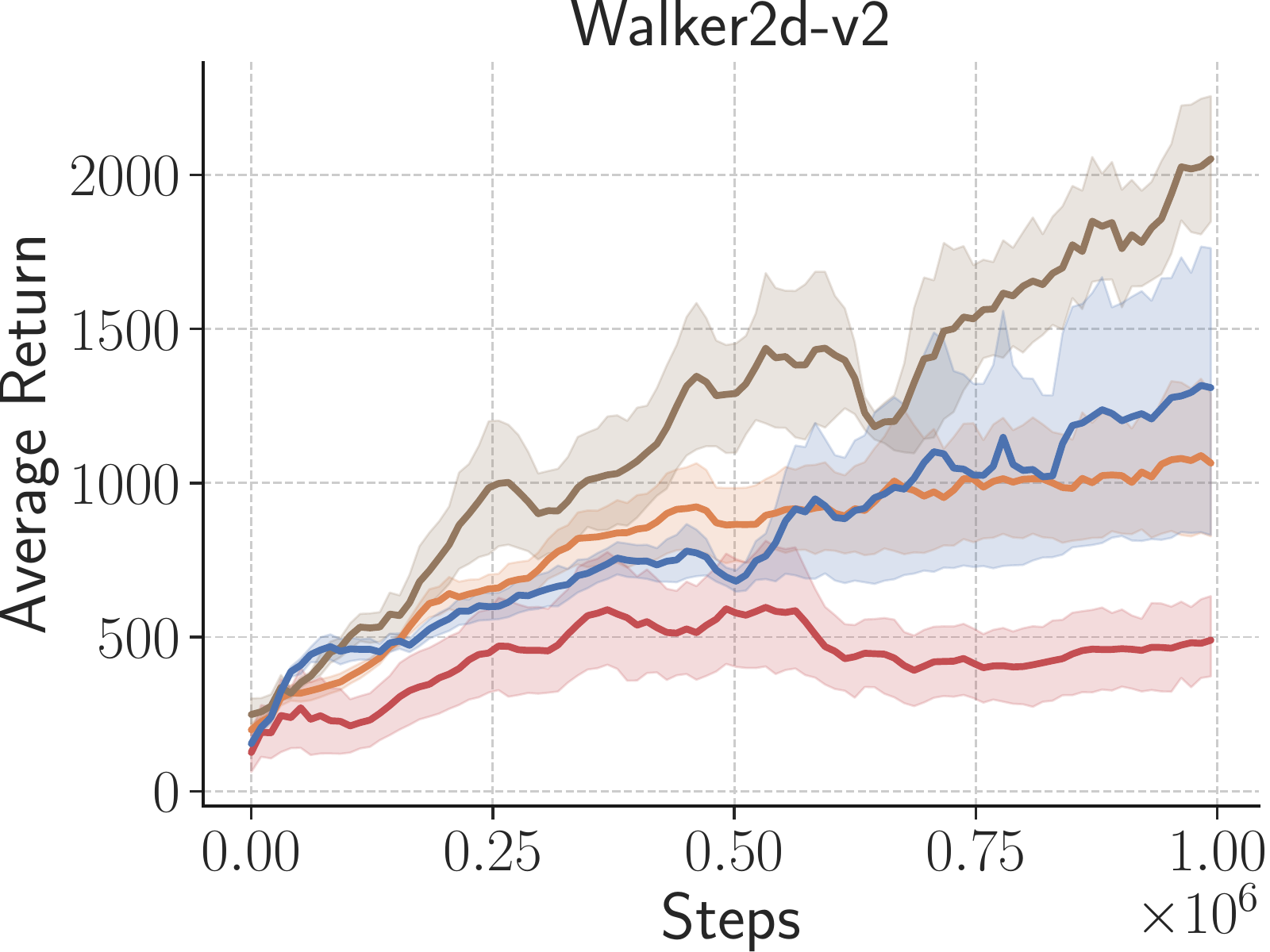}
    \includegraphics[width=0.24\linewidth]{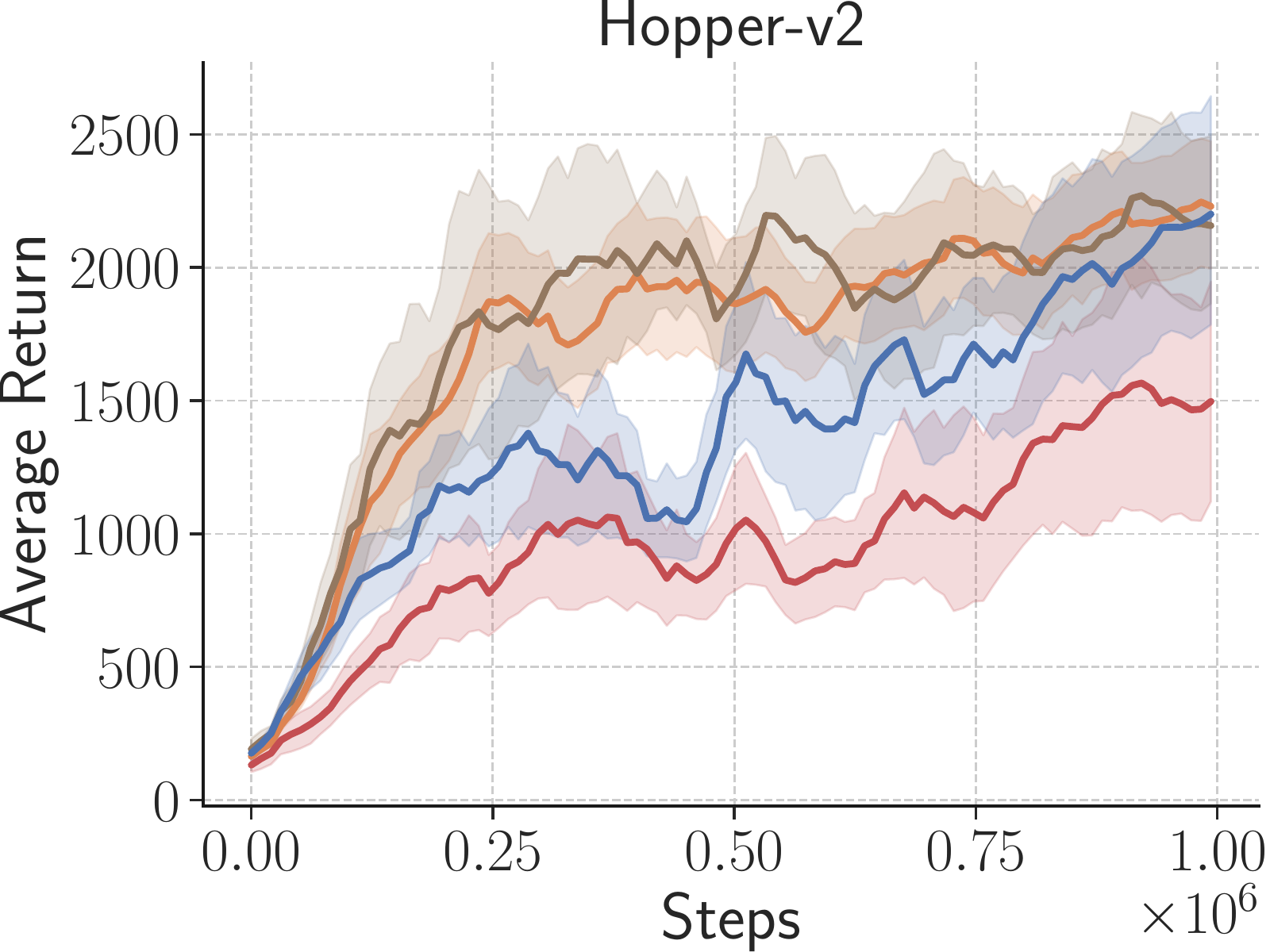}
    \includegraphics[width=0.24\linewidth]{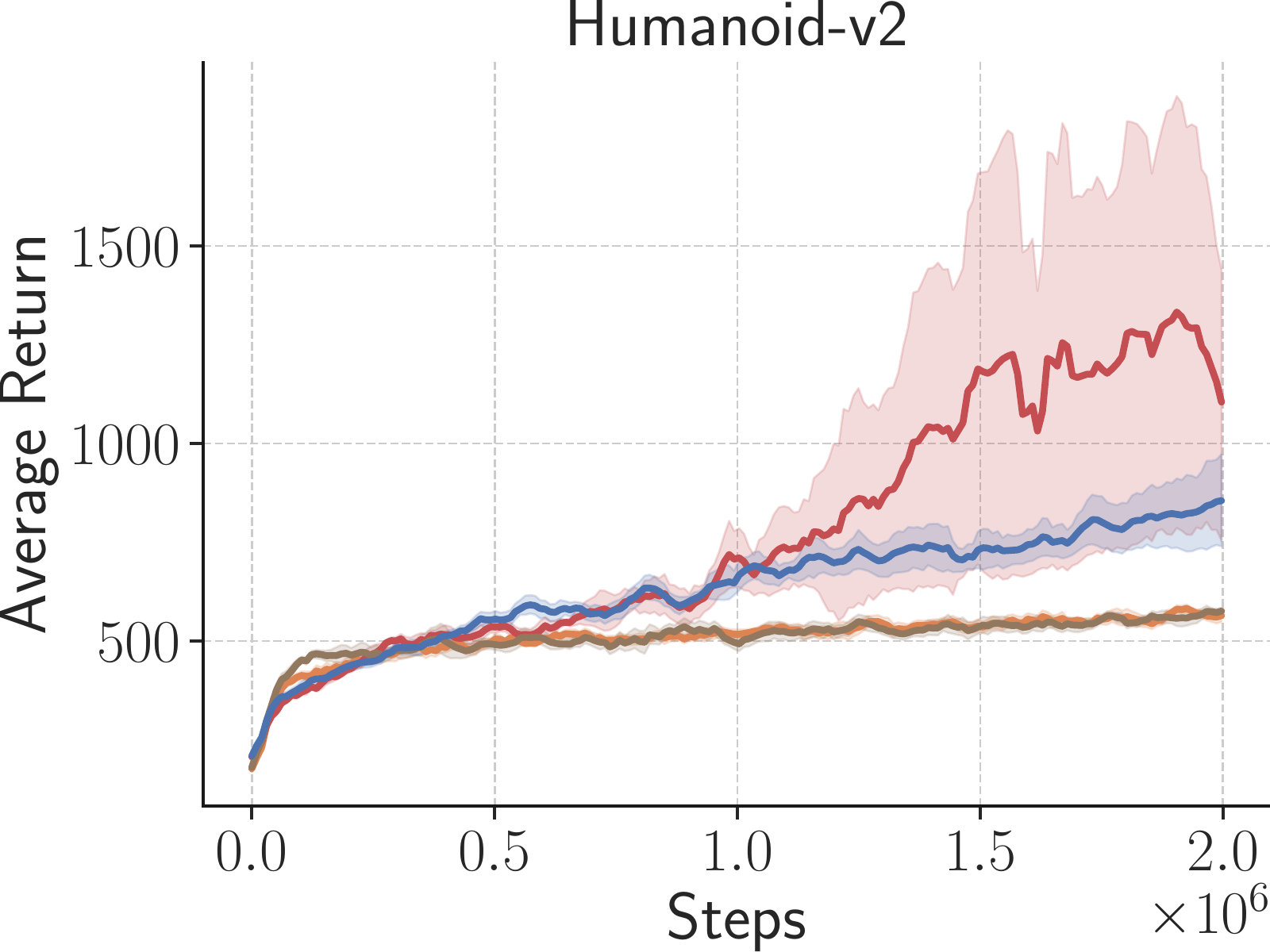}
    \caption{\textbf{On-Policy Evaluation}: We train PPO on 4 OpenAI Gym environments with our LFF architecture and a vanilla MLP, both of which have roughly the same number of parameters. LFF does not produce a consistent gain in sample efficiency here, which is consistent with our hypothesis that LFF helps only for off-policy RL.}
    \label{fig:ppo}
\end{figure}
\label{sec:on_policy}
\subsection{On-policy LFF Setup}
We evaluate proximal policy optimization (PPO, \citet{schulman2017proximal}) on 4 environments from OpenAI gym \citep{brockman2016openai}. These environments range from easy, e.g. HalfCheetah-v2, to difficult, e.g. Humanoid-v2. Just as we did in the off-policy setup, we modify only the architecture. keeping the hyperparameters fixed. We compare MLPs with 3 hidden layers to LFF with our Fourier feature input layer followed by 2 hidden layers. We use $d_\text{fourier} = 1024$ and $\sigma = 0.001$ for all environments and also test what happens if we use LFF for only the policy or only the critic.

\subsection{LFF architecture for on-policy RL}
In Figure~\ref{fig:ppo}, we show the results of using LFF for PPO. Unlike LFF on SAC, LFF did not yield consistent gains for PPO. The best setting was to use LFF for only the critic, which does as well as MLPs on 3 environments (HalfCheetah, Hopper, and Humanoid) and better on Walker2d, but this is only a modest improvement.
This is not surprising and is likely because policy gradient methods have different optimization challenges than those based on Q-learning. For one, the accuracy of the value function baseline is less important for policy gradient methods. The on-policy cumulative return provides a lot of reward signal, and the value function baseline mainly reduces variance in the gradient update. In addition, generalized advantage estimation \citep{schulman2015high} further reduces variance. Thus, noise in the bootstrapping process is not as serious of a problem as in off-policy learning.

\newpage 
\section{Performance Under Added Noise}
\begin{figure}[H]
    \centering
    \includegraphics[width=0.28\linewidth]{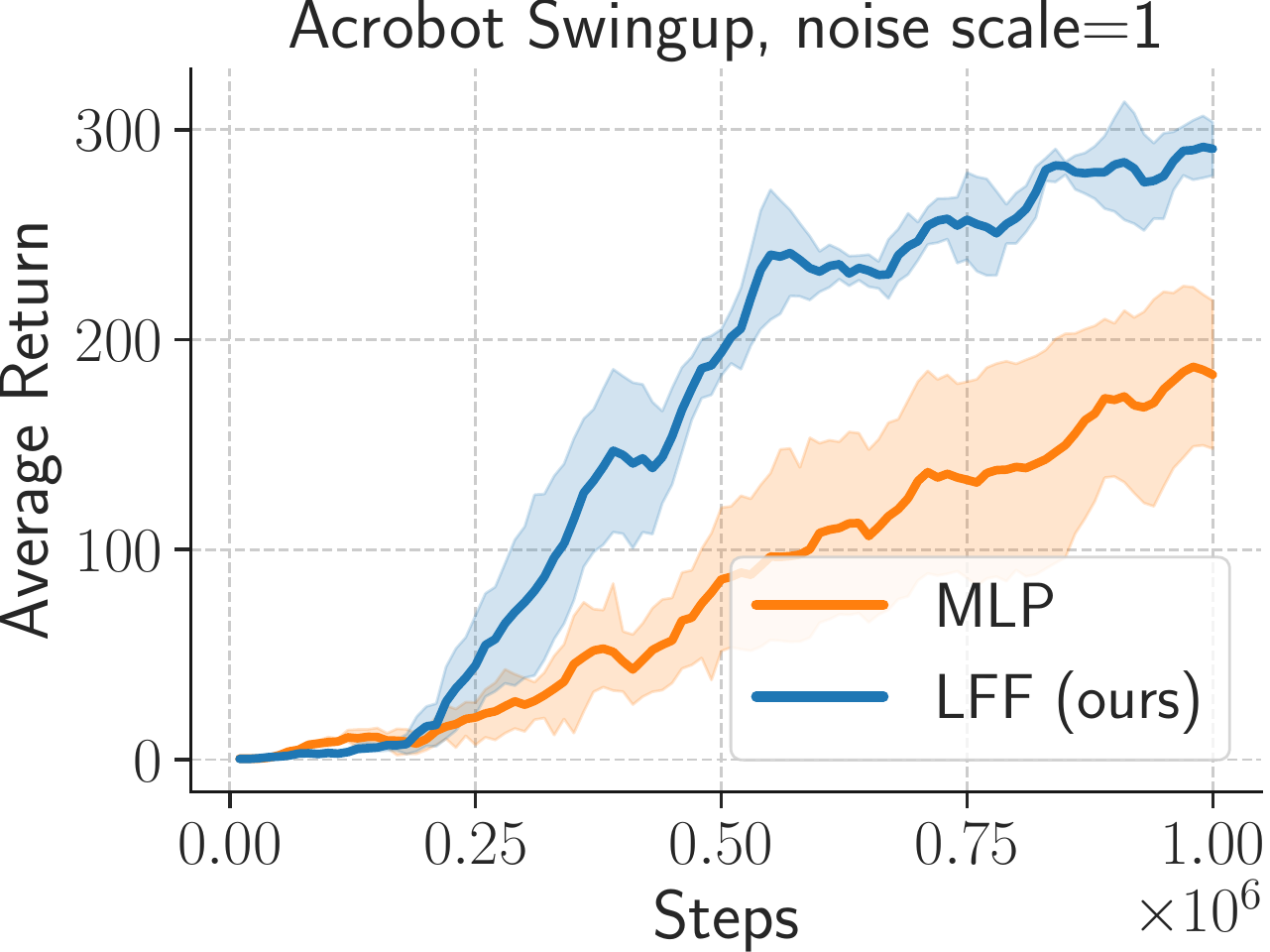}
    \includegraphics[width=0.28\linewidth]{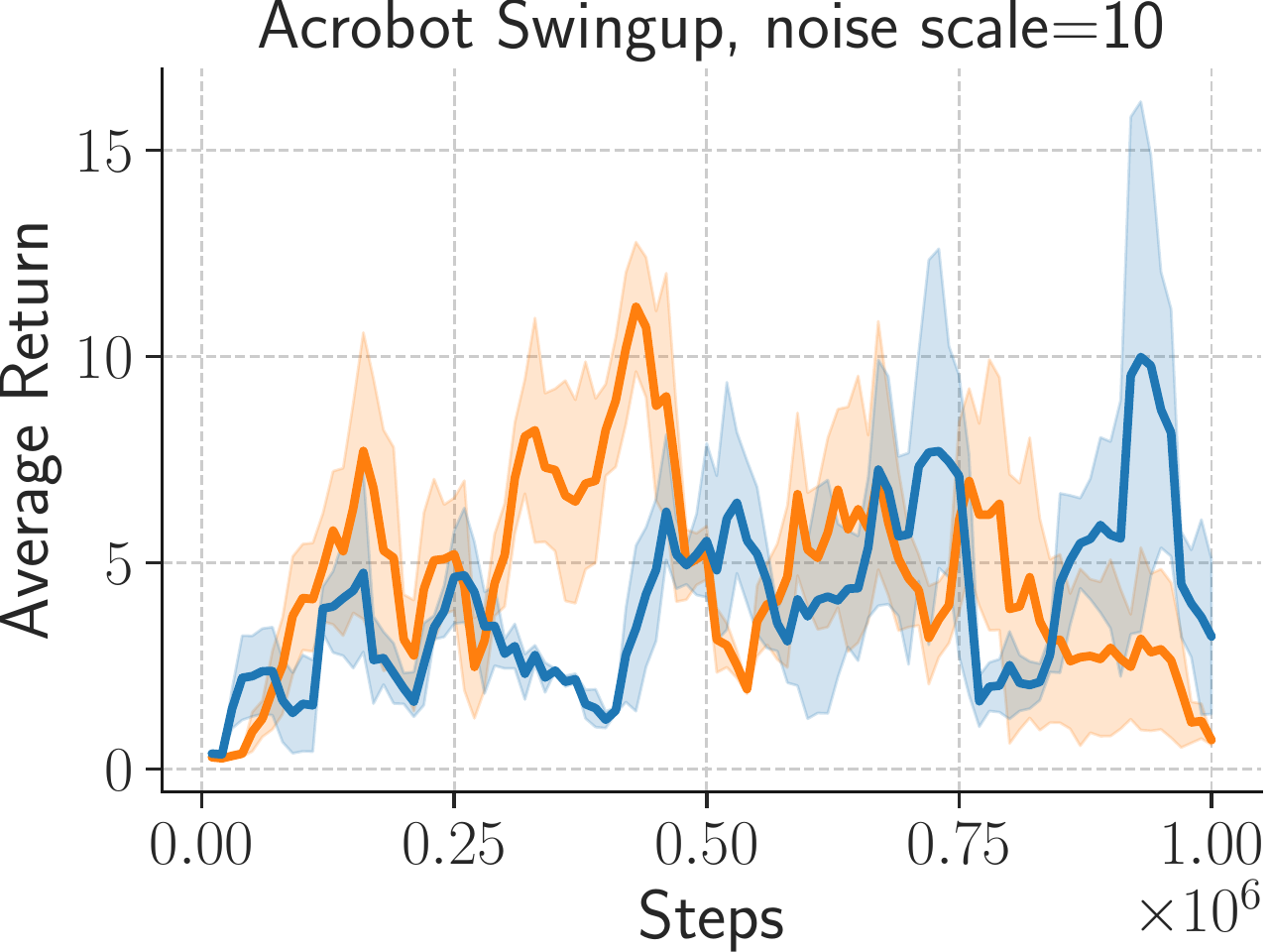}
    \includegraphics[width=0.28\linewidth]{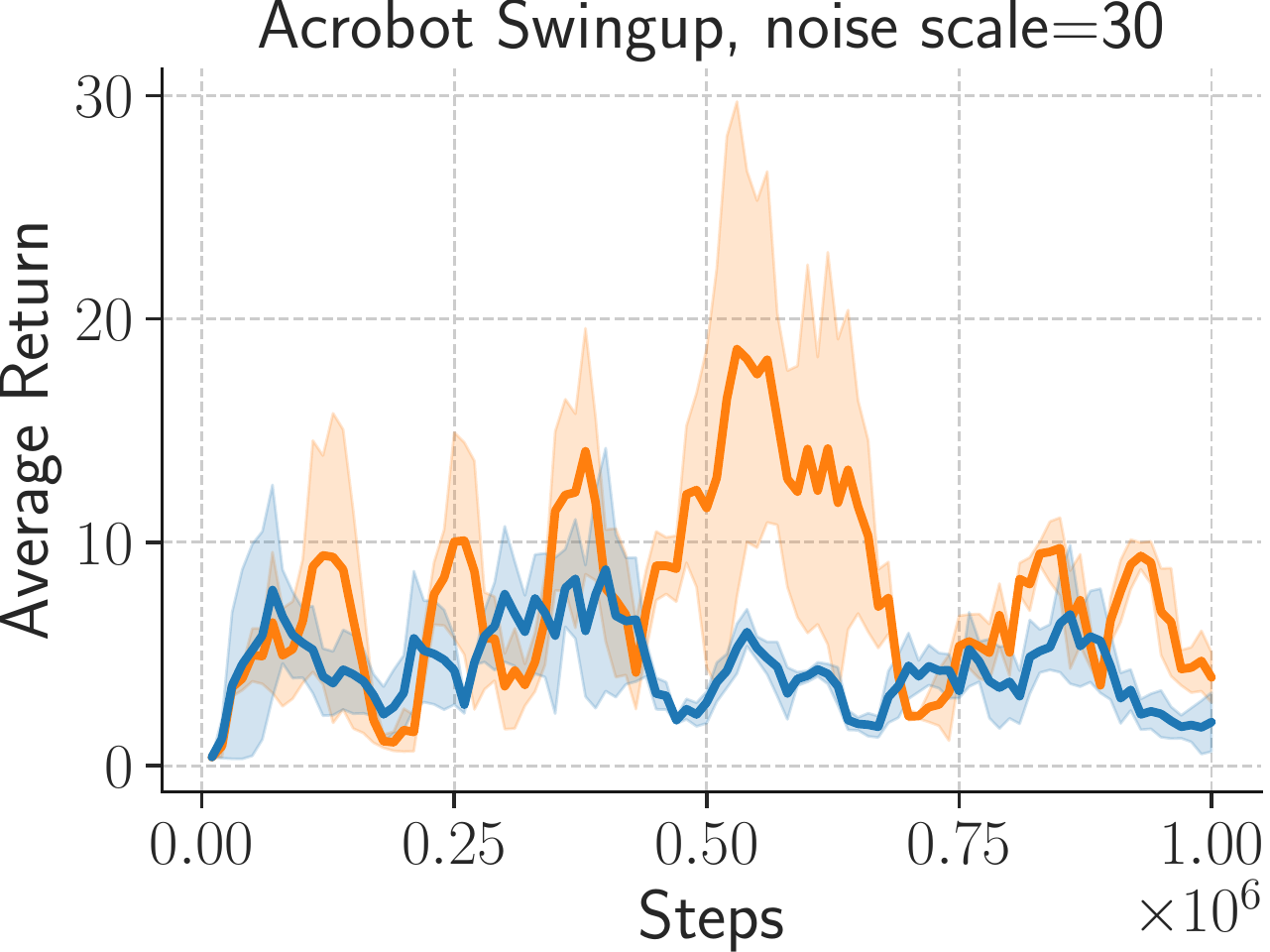} \\
    \includegraphics[width=0.28\linewidth]{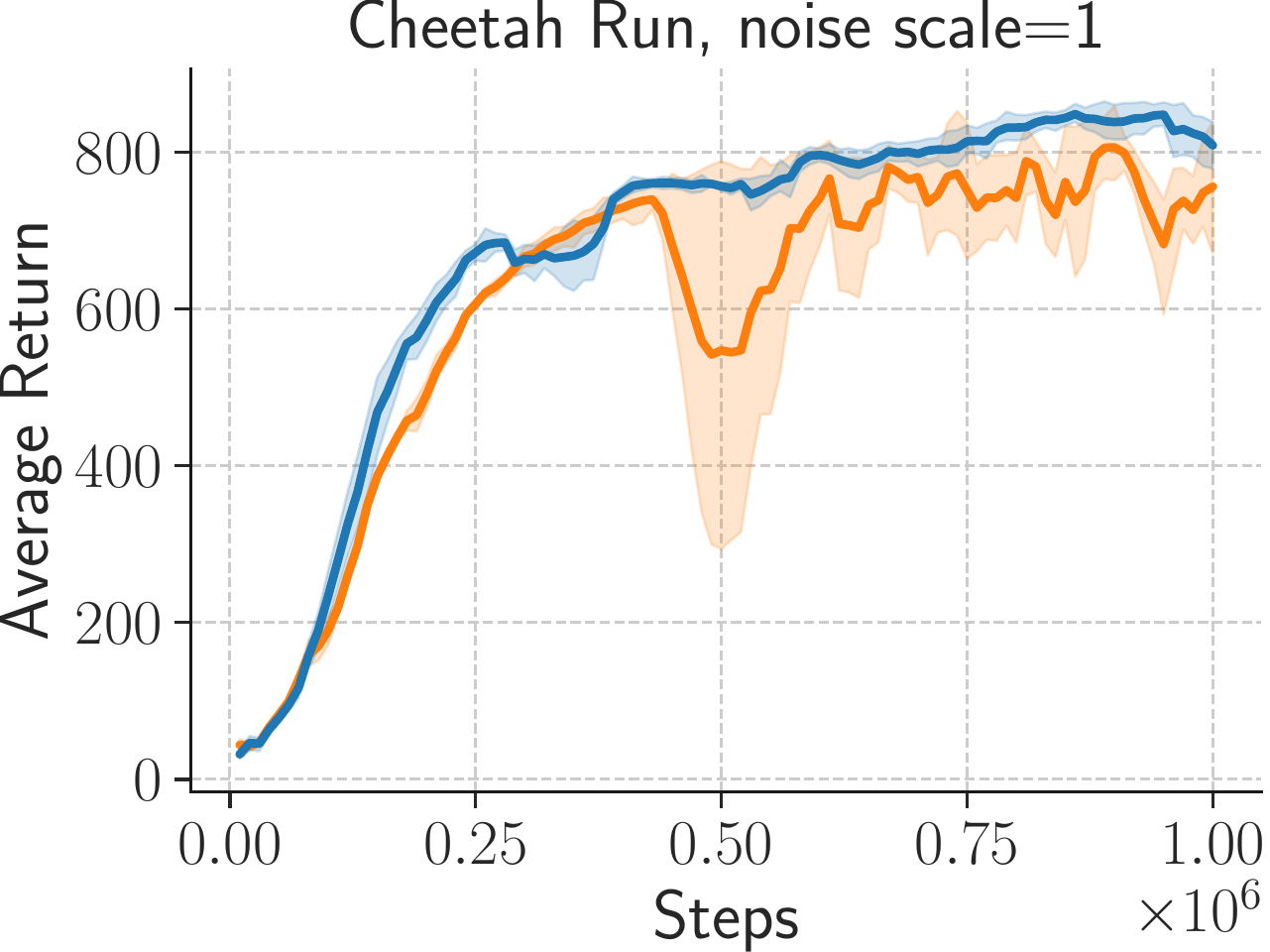} 
    \includegraphics[width=0.28\linewidth]{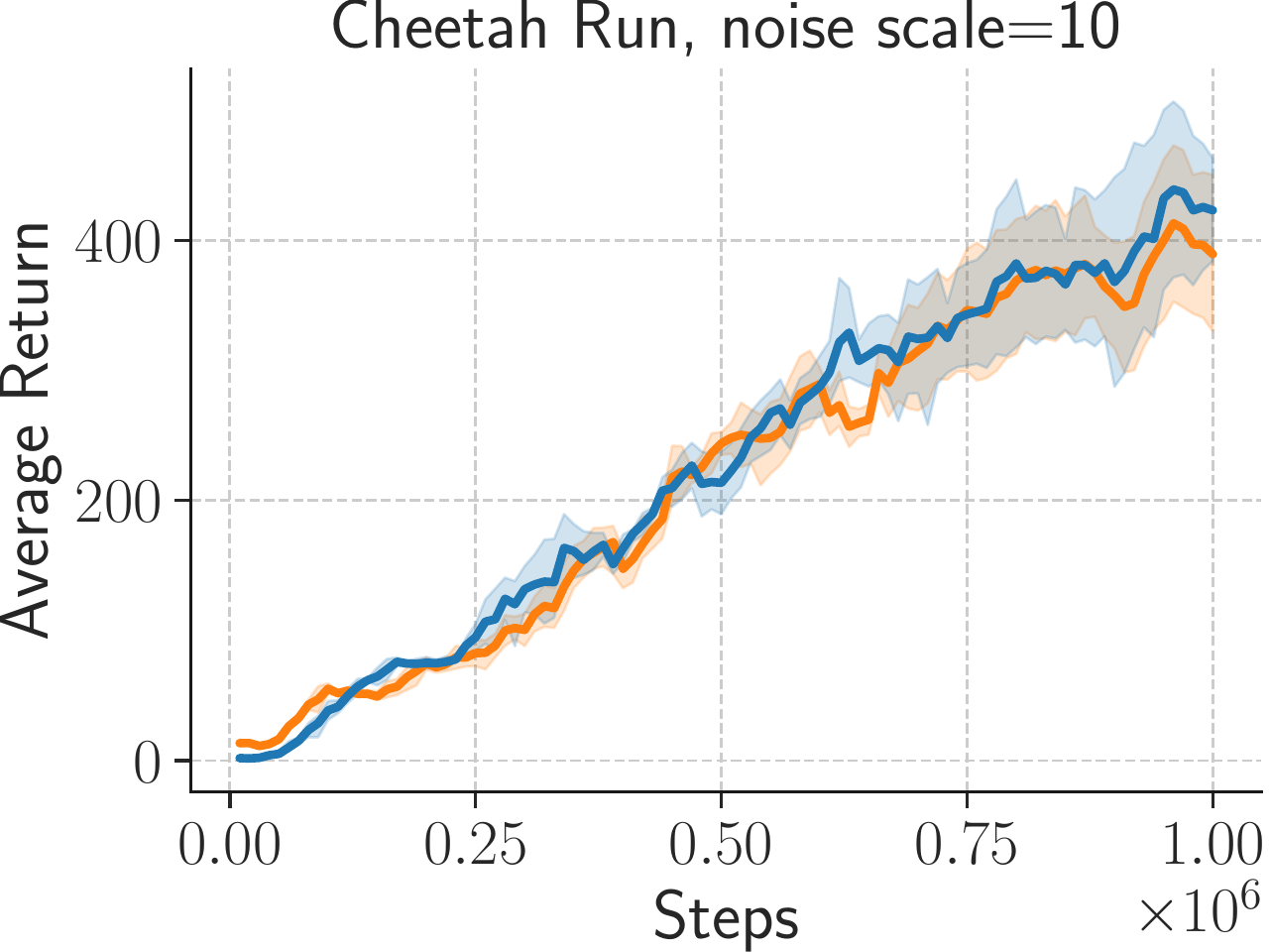} 
    \includegraphics[width=0.28\linewidth]{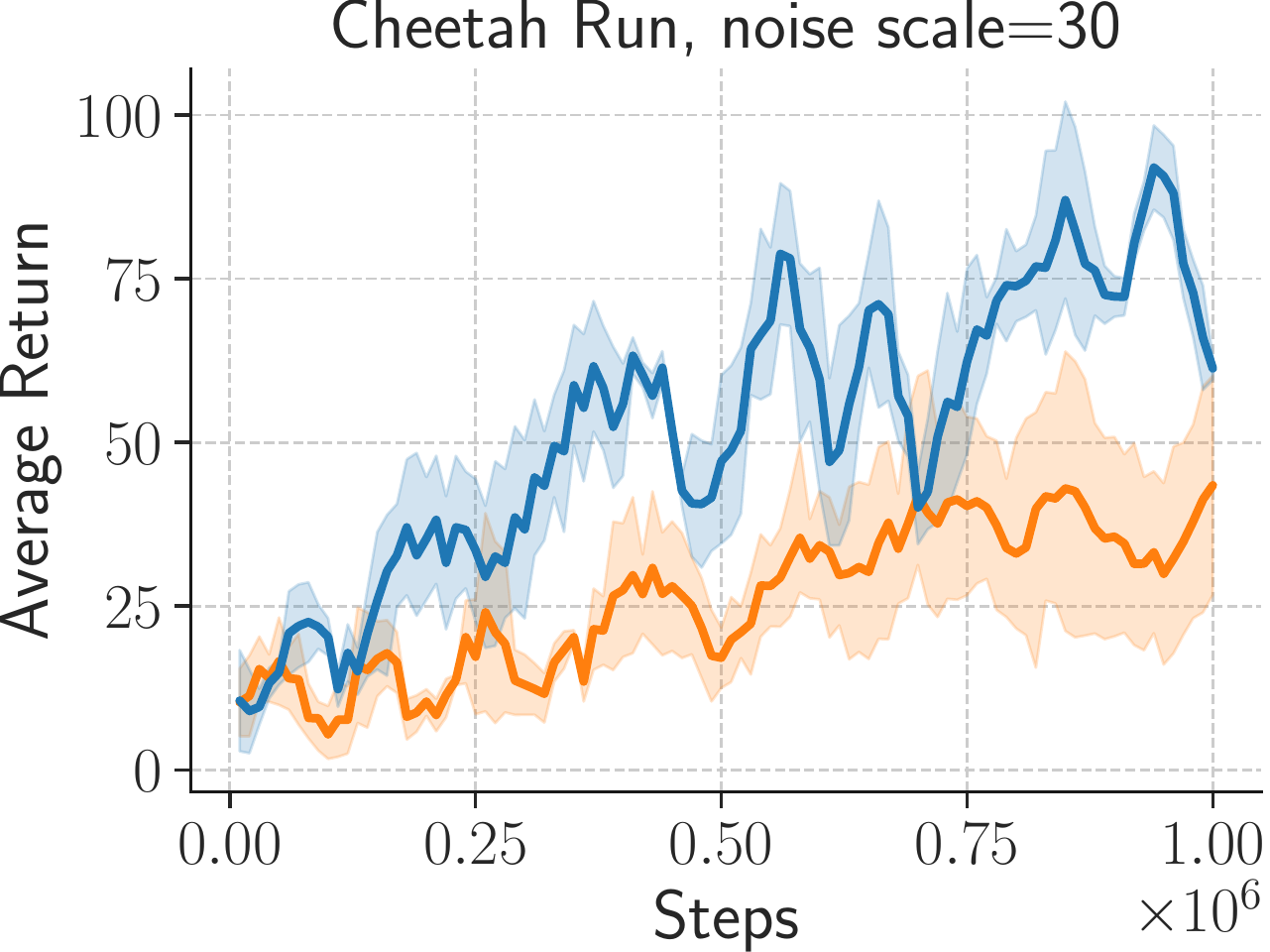} \\
    \includegraphics[width=0.28\linewidth]{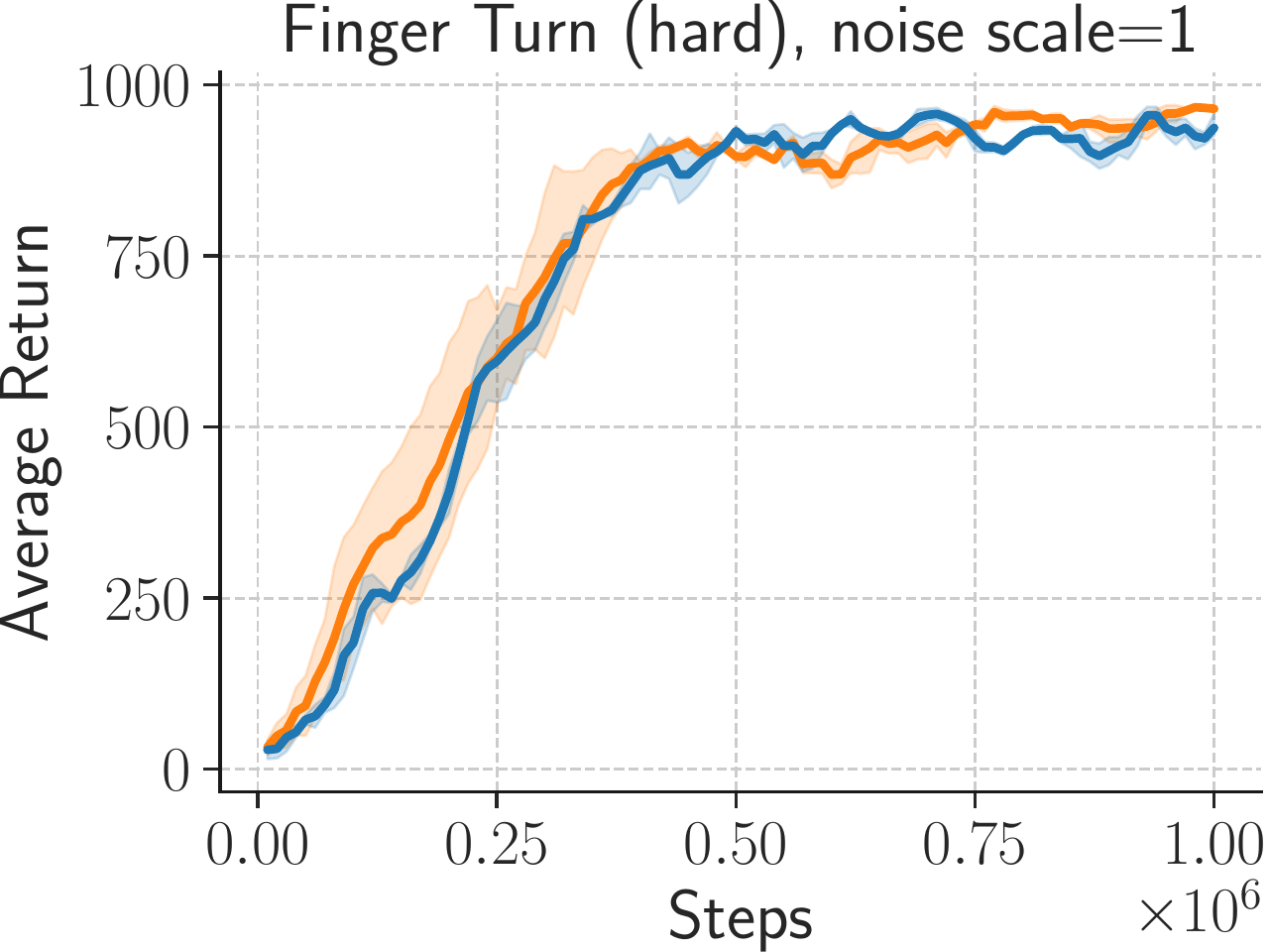} 
    \includegraphics[width=0.28\linewidth]{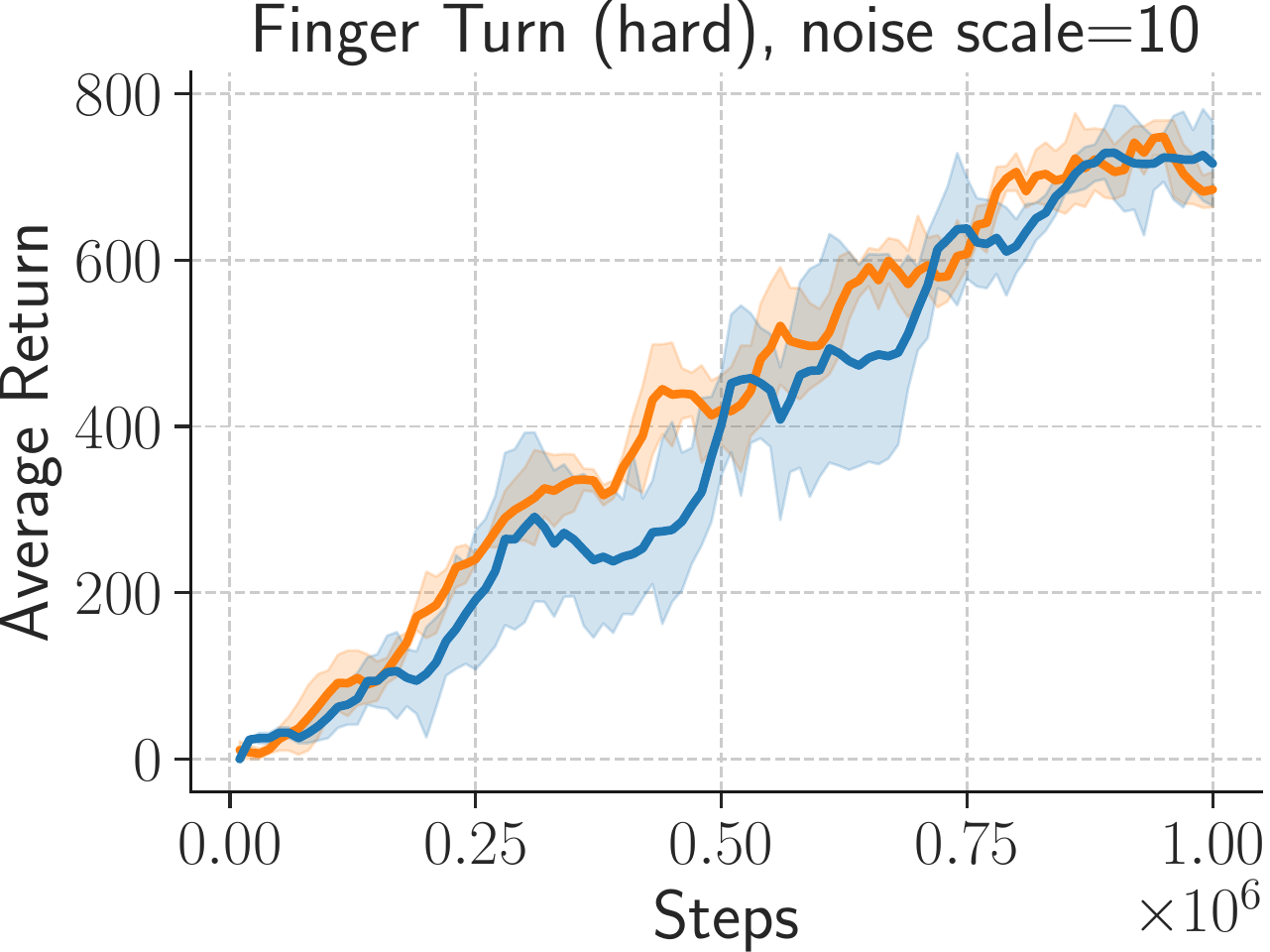} 
    \includegraphics[width=0.28\linewidth]{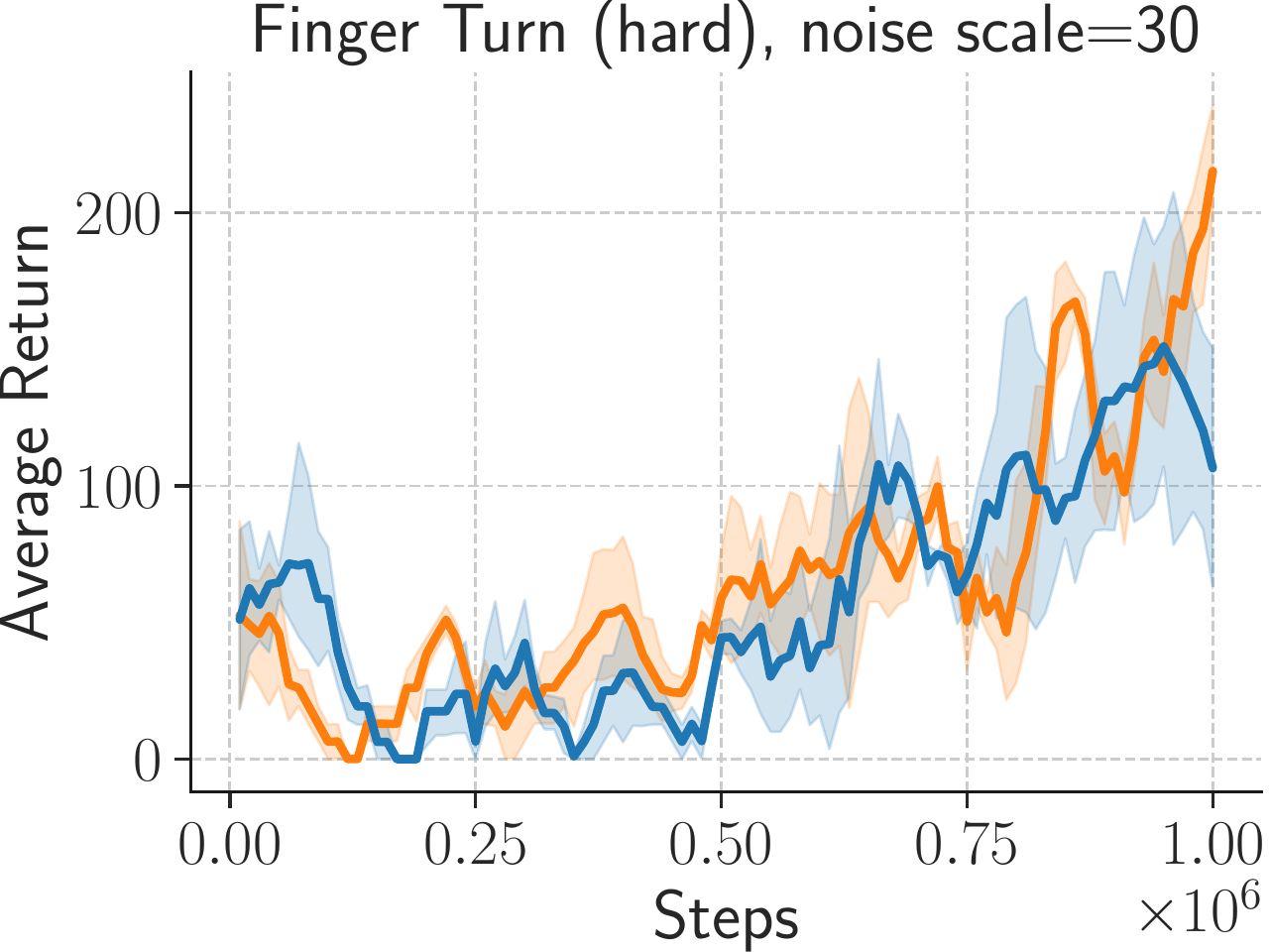} \\
    \includegraphics[width=0.28\linewidth]{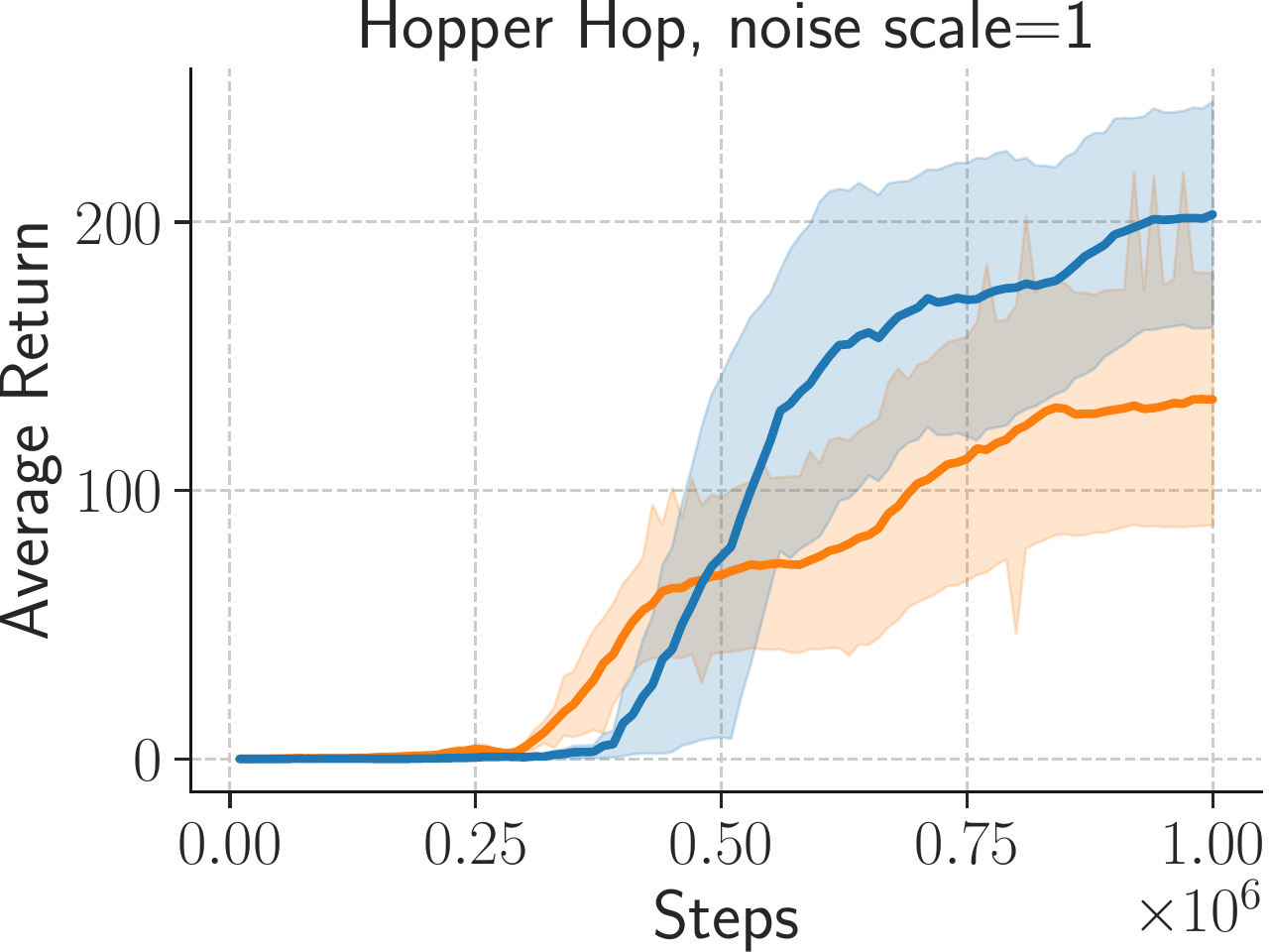} 
    \includegraphics[width=0.28\linewidth]{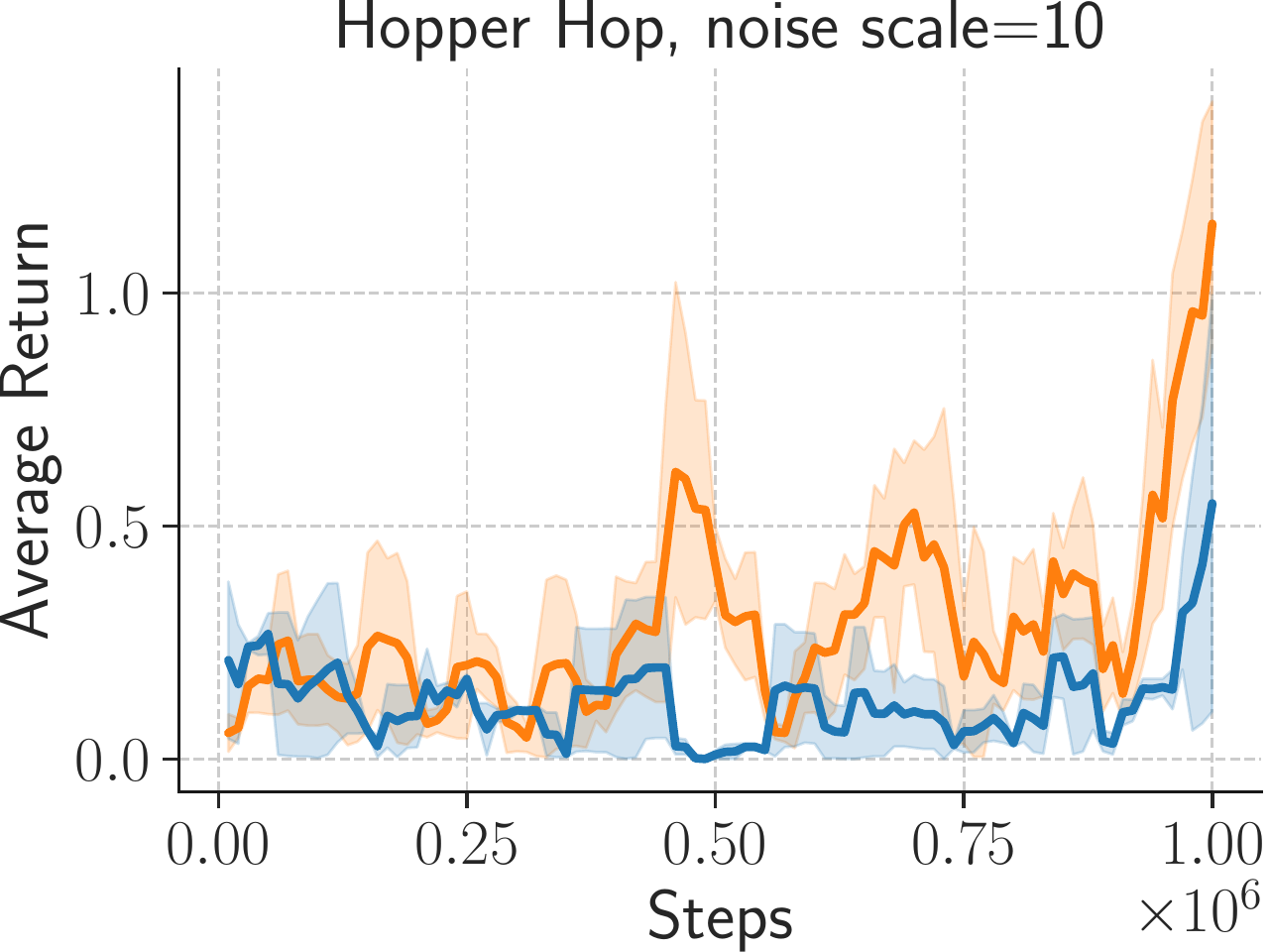} 
    \includegraphics[width=0.28\linewidth]{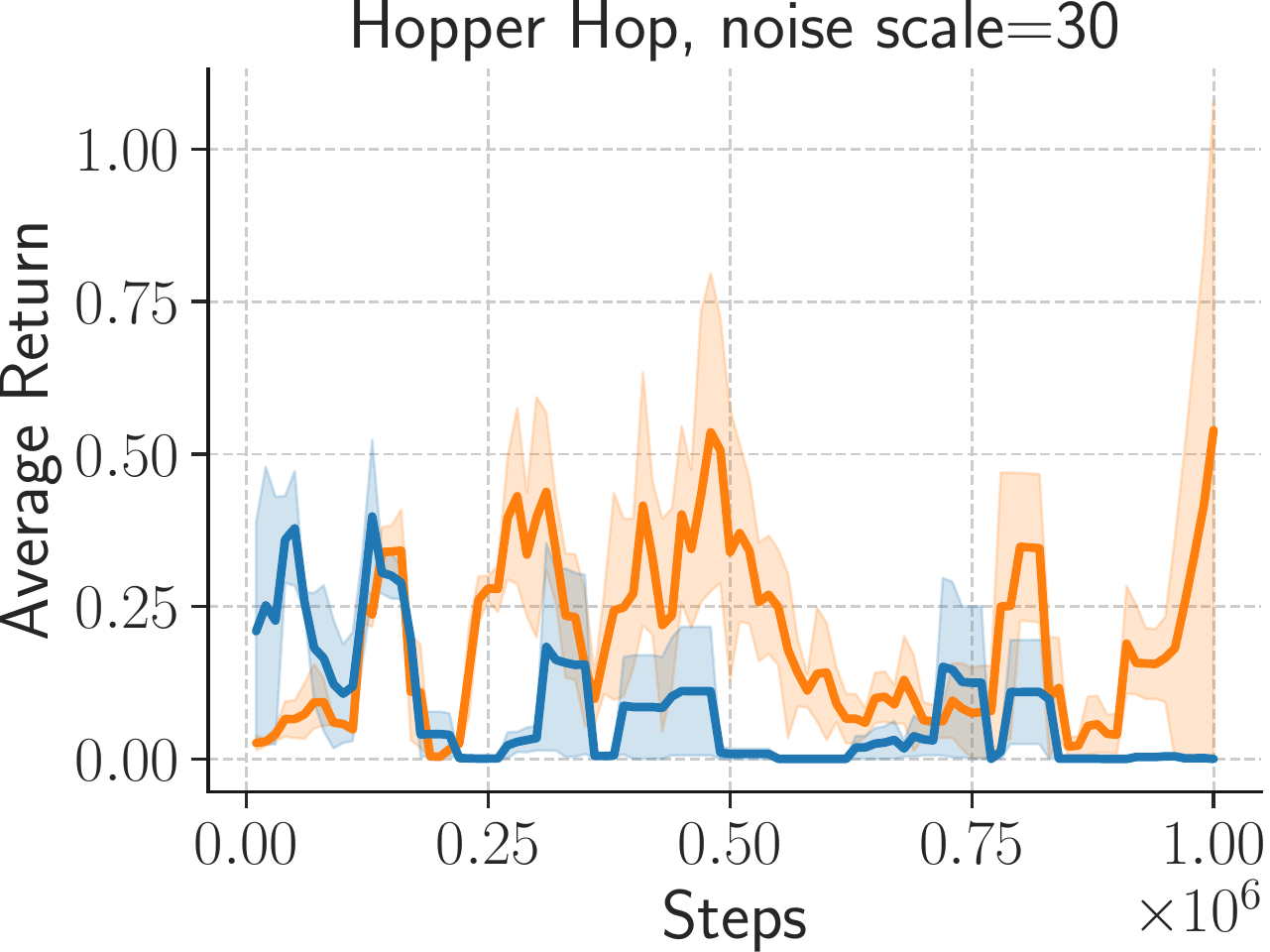}  \\
    \includegraphics[width=0.28\linewidth]{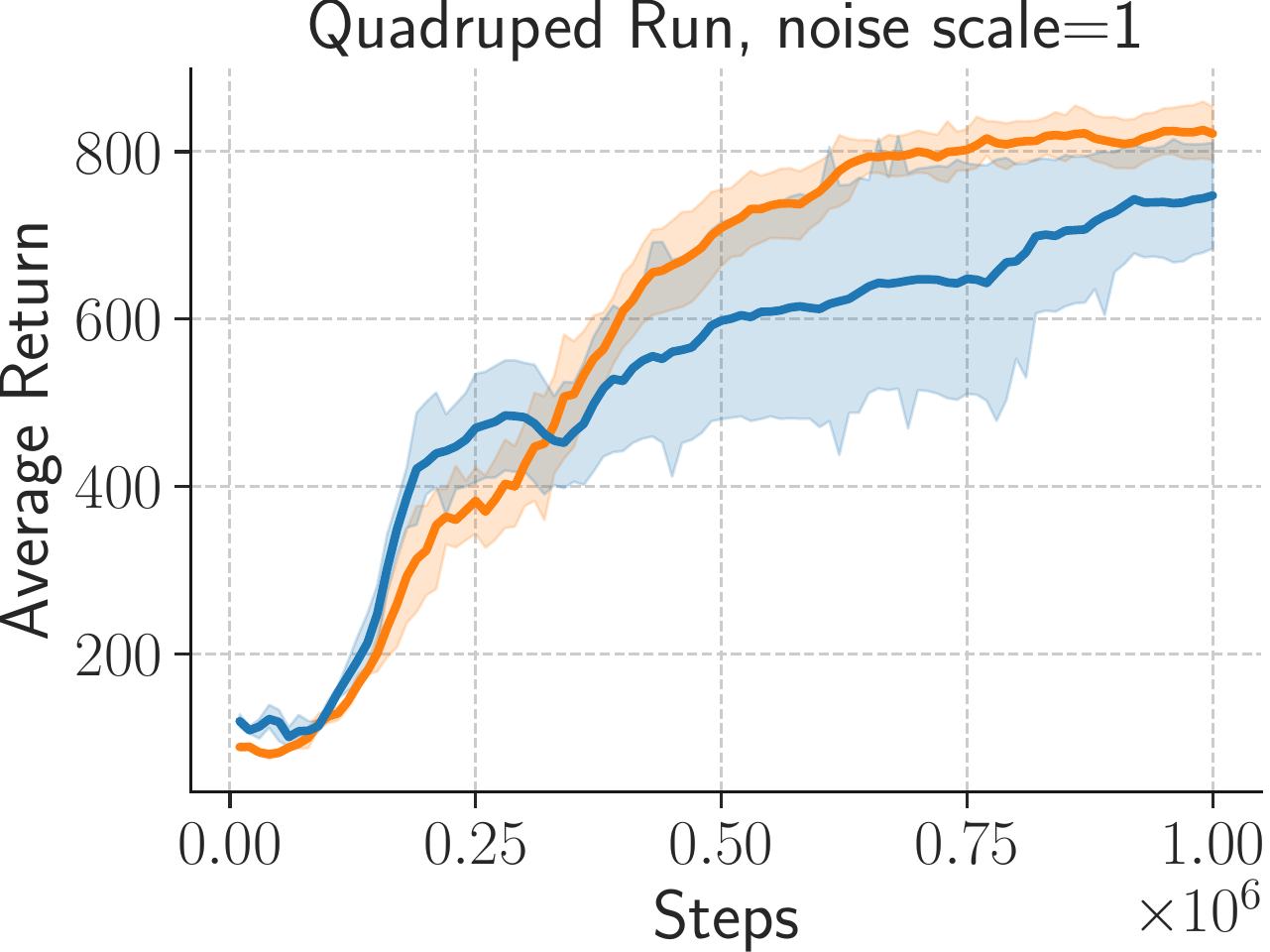} 
    \includegraphics[width=0.28\linewidth]{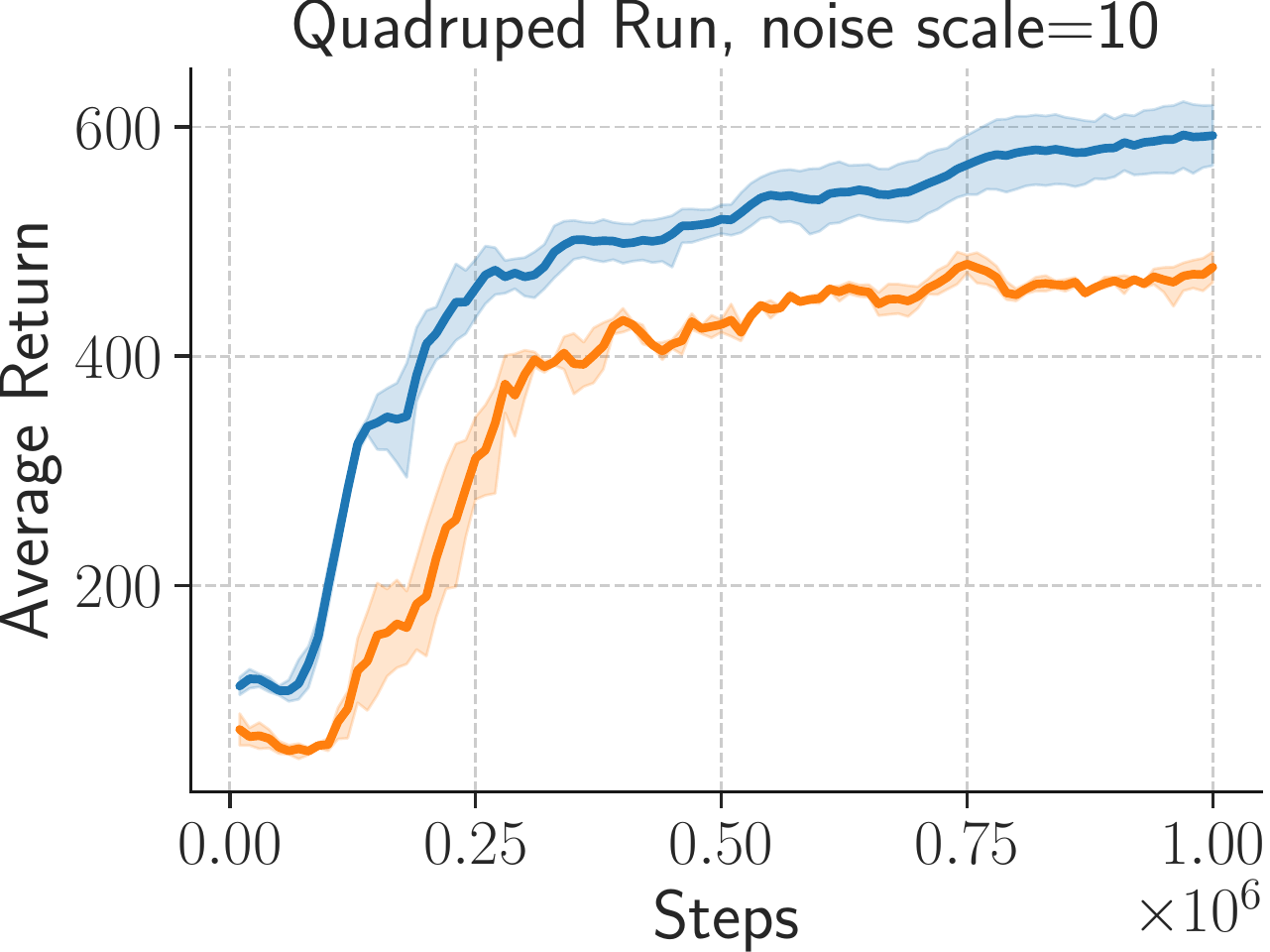} 
    \includegraphics[width=0.28\linewidth]{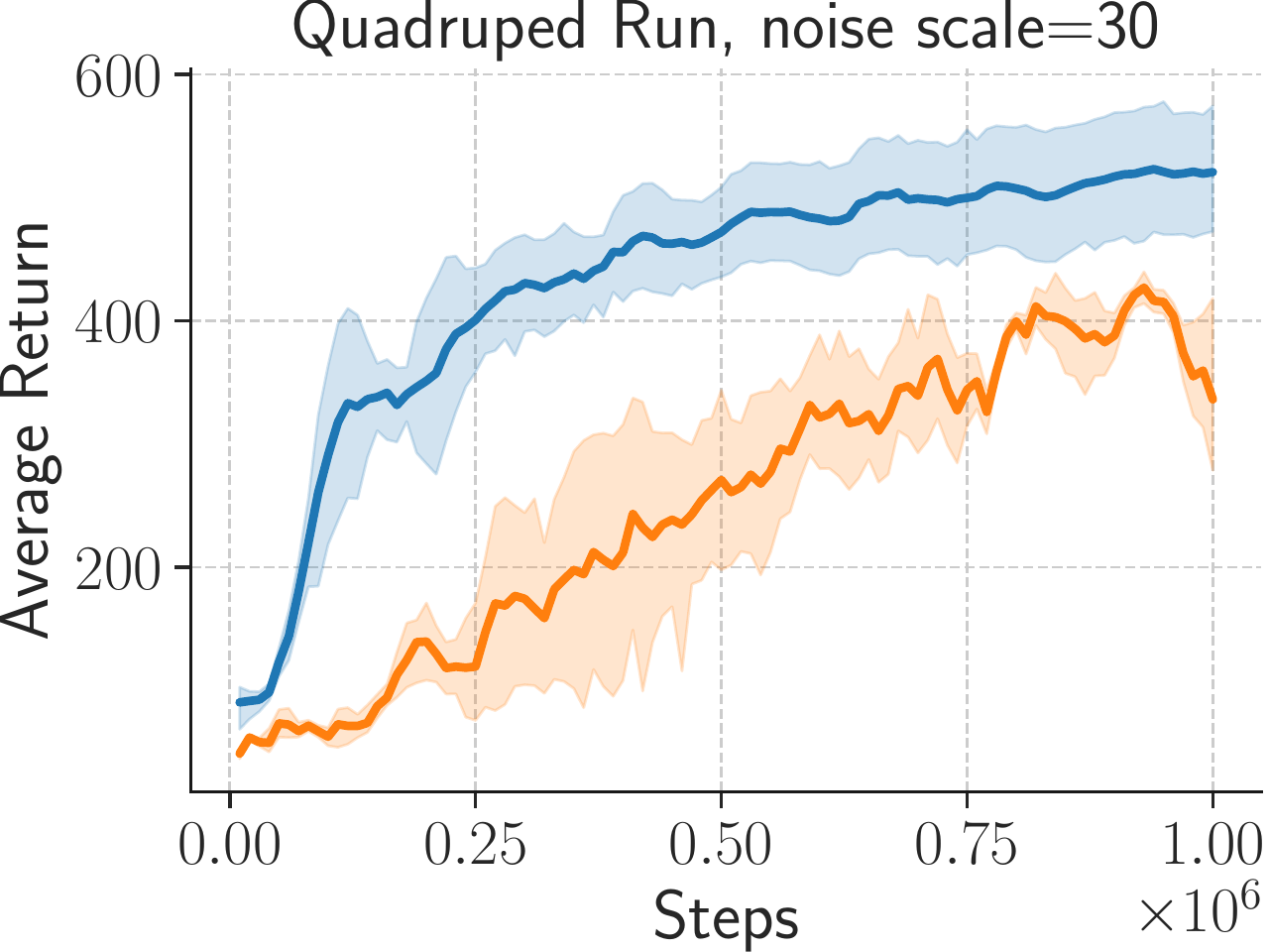} \\
    \includegraphics[width=0.28\linewidth]{rebuttal_figures/noise/quadruped_walk_1.pdf} 
    \includegraphics[width=0.28\linewidth]{rebuttal_figures/noise/quadruped_walk_10.pdf} 
    \includegraphics[width=0.28\linewidth]{rebuttal_figures/noise/quadruped_walk_30.pdf} \\
    \includegraphics[width=0.28\linewidth]{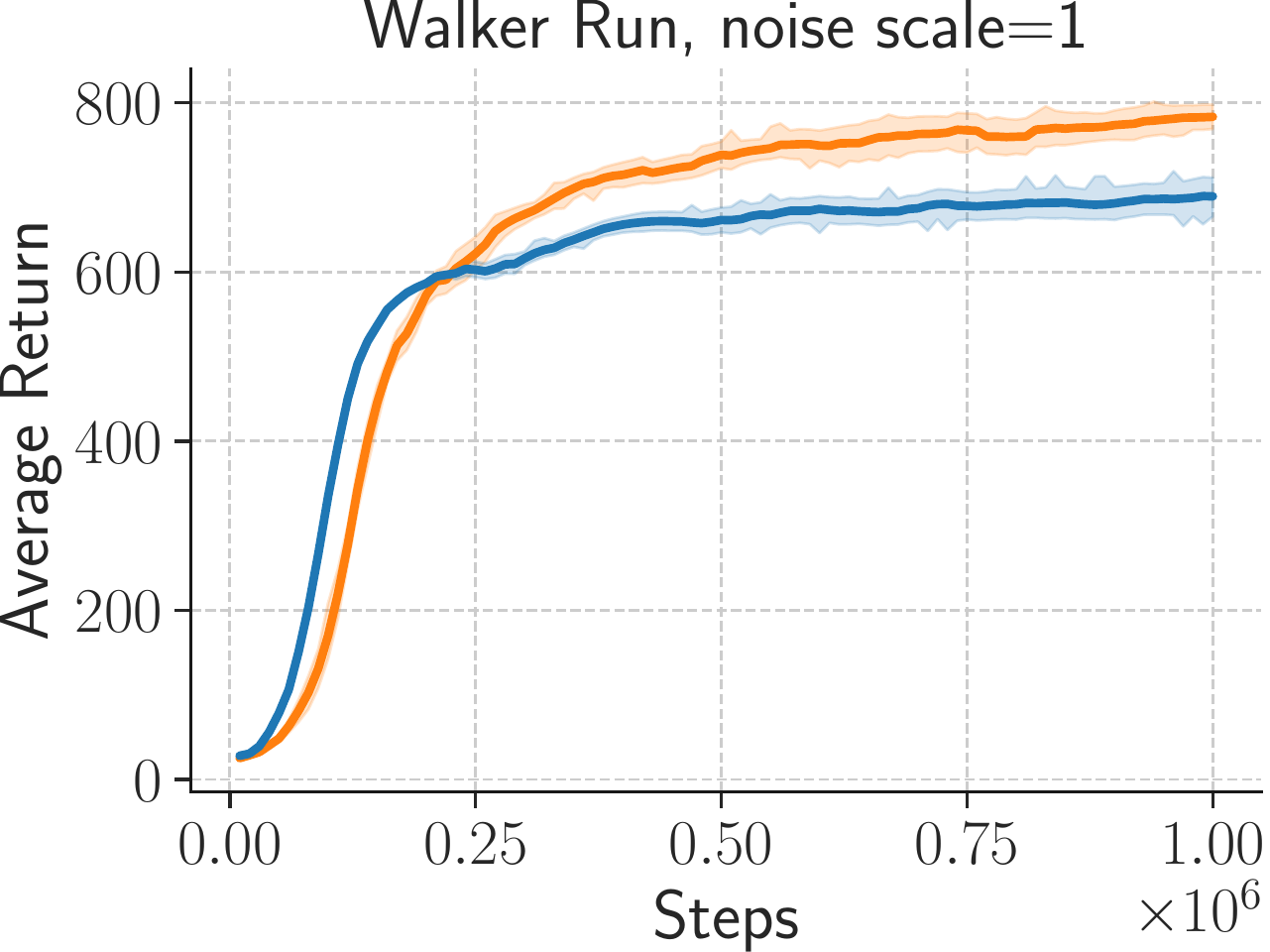} 
    \includegraphics[width=0.28\linewidth]{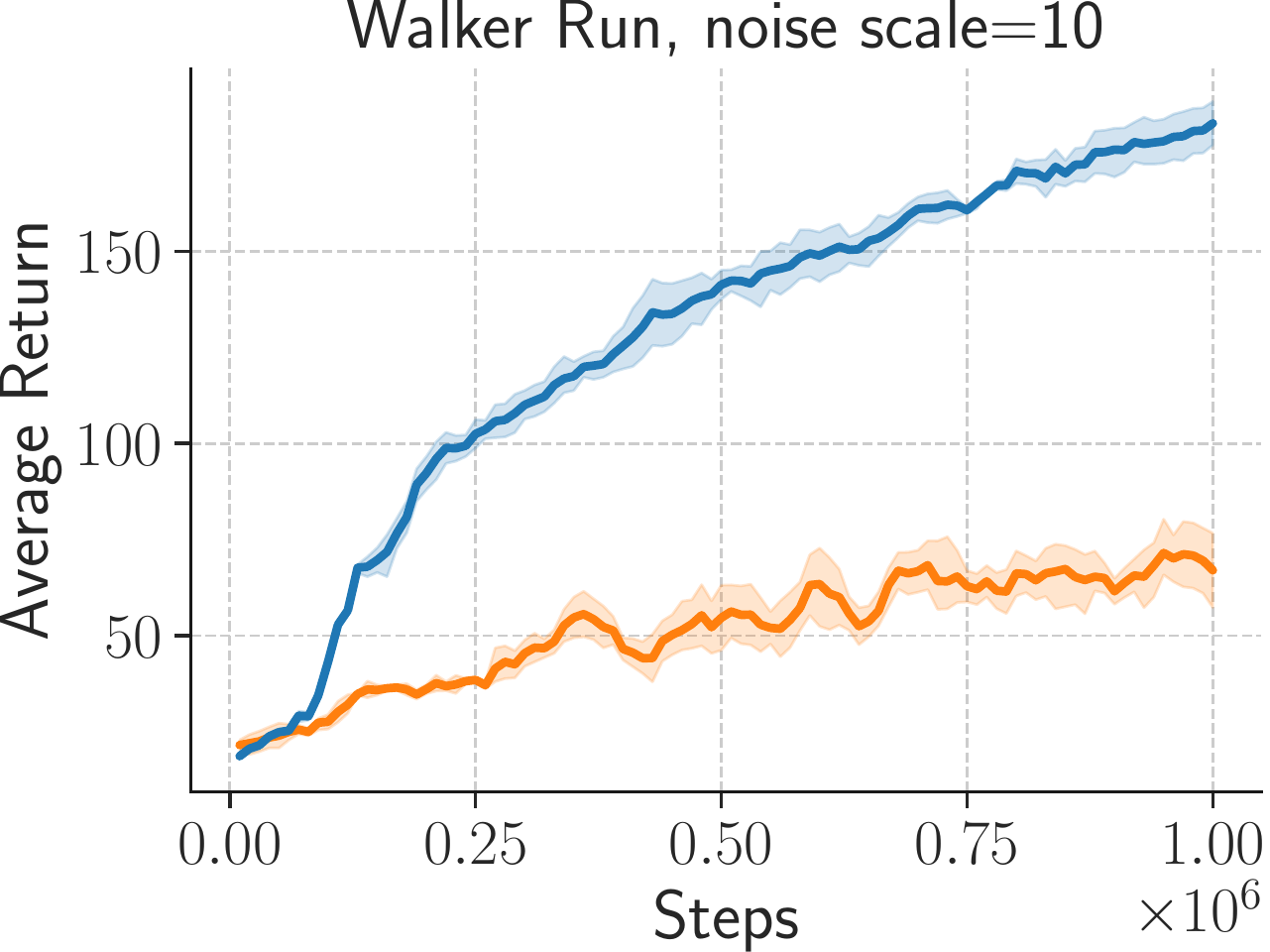} 
    \includegraphics[width=0.28\linewidth]{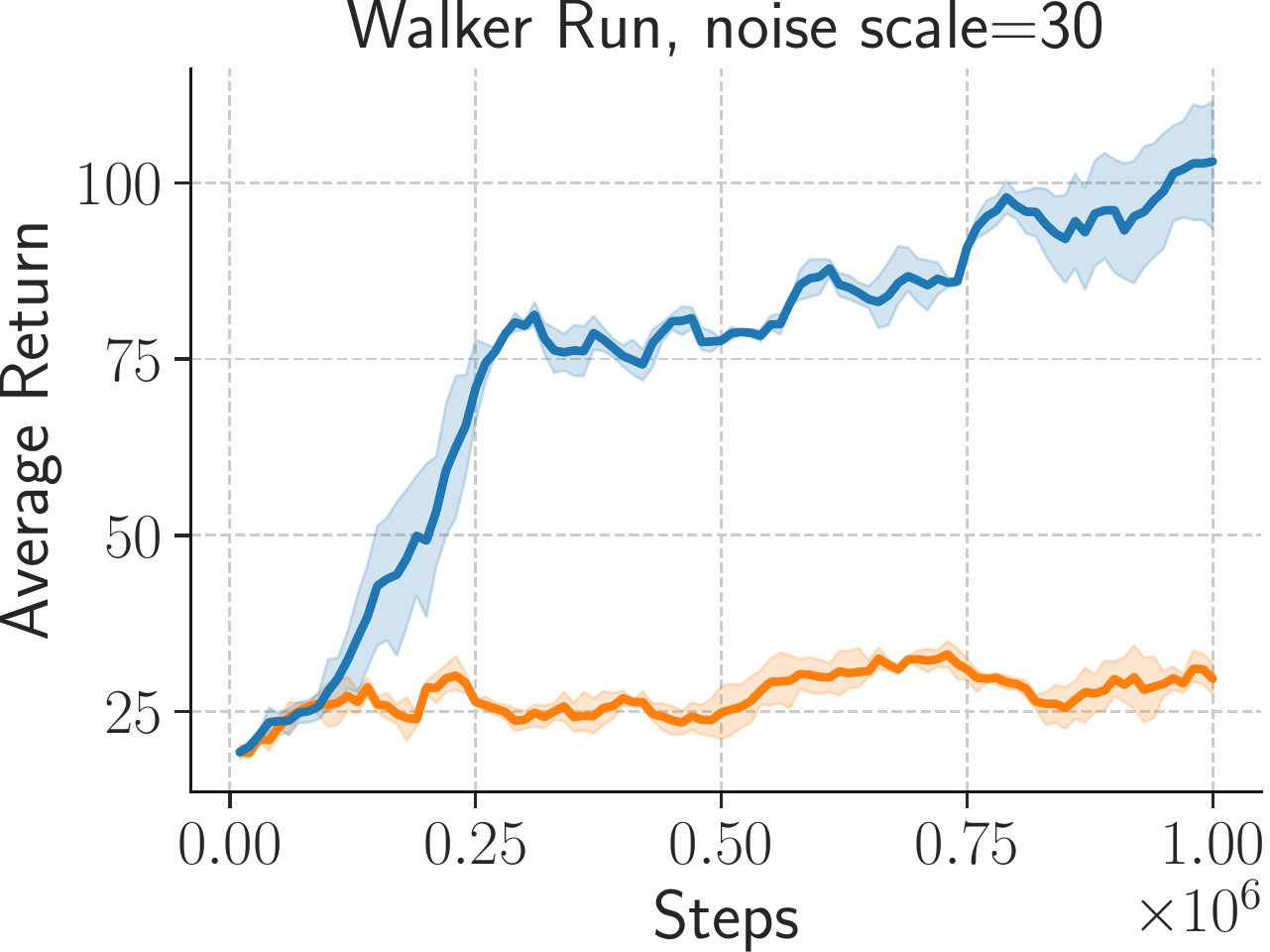}
    \caption{\textbf{State-based with Added Noise}: We add zero-mean Gaussian noise to the targets. As the standard deviation of the added noise increases, LFF maintains its performance better than MLPs.}
    \label{fig:added_noise}
\end{figure}

\section{Noise Amplification vs Implicit Underparameterization in Gridworld MDP}
\label{sec:noise_amplification}
\begin{figure}[H]
    \centering
    \begin{subfigure}{0.32\textwidth}
        \centering
        \includegraphics[width=\linewidth]{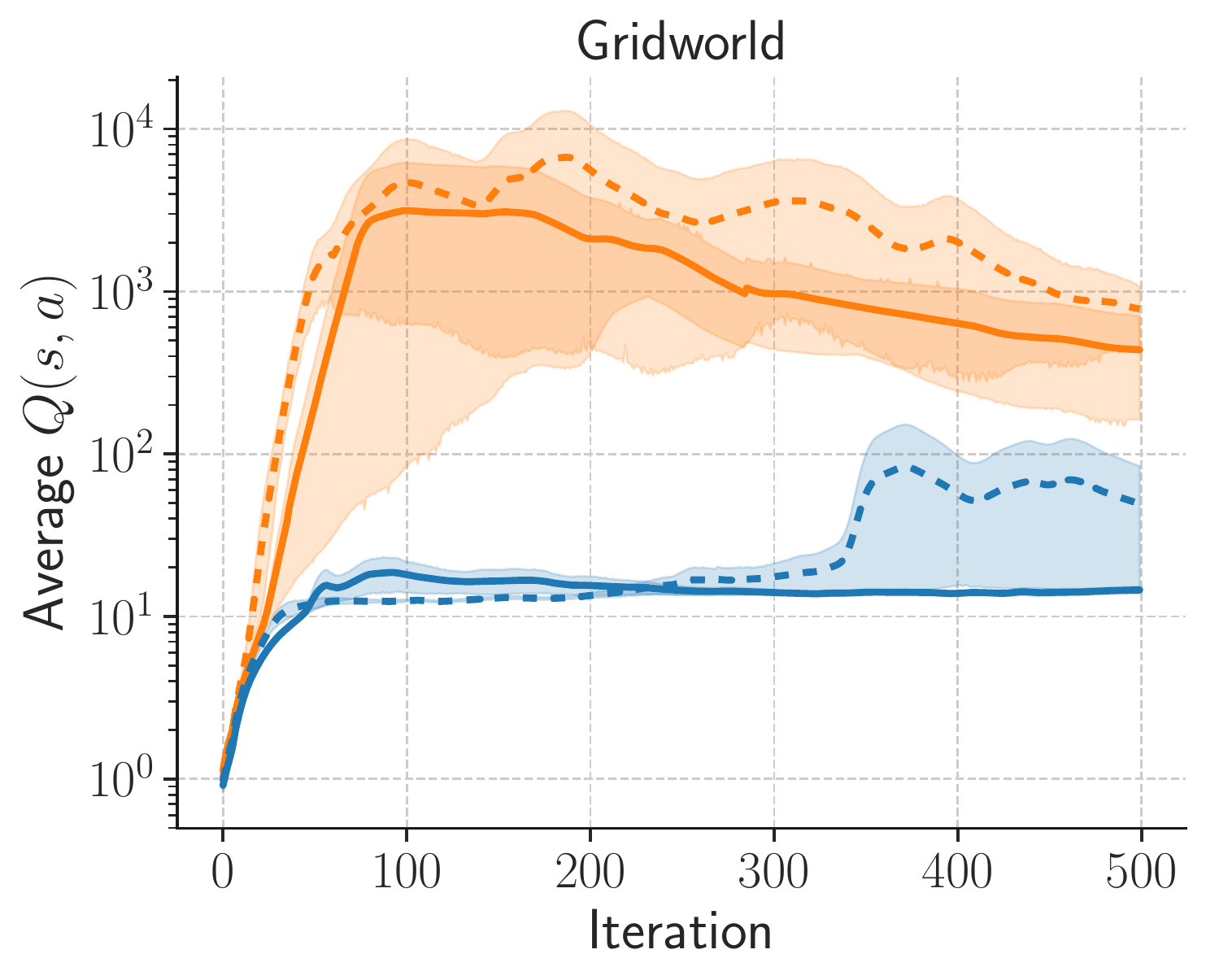}
        \caption{Average Q-network output over $(s, a)$ drawn from the replay buffer.}
    \end{subfigure}
    \hfill
    \begin{subfigure}{0.32\textwidth}
        \centering
        \includegraphics[width=\linewidth]{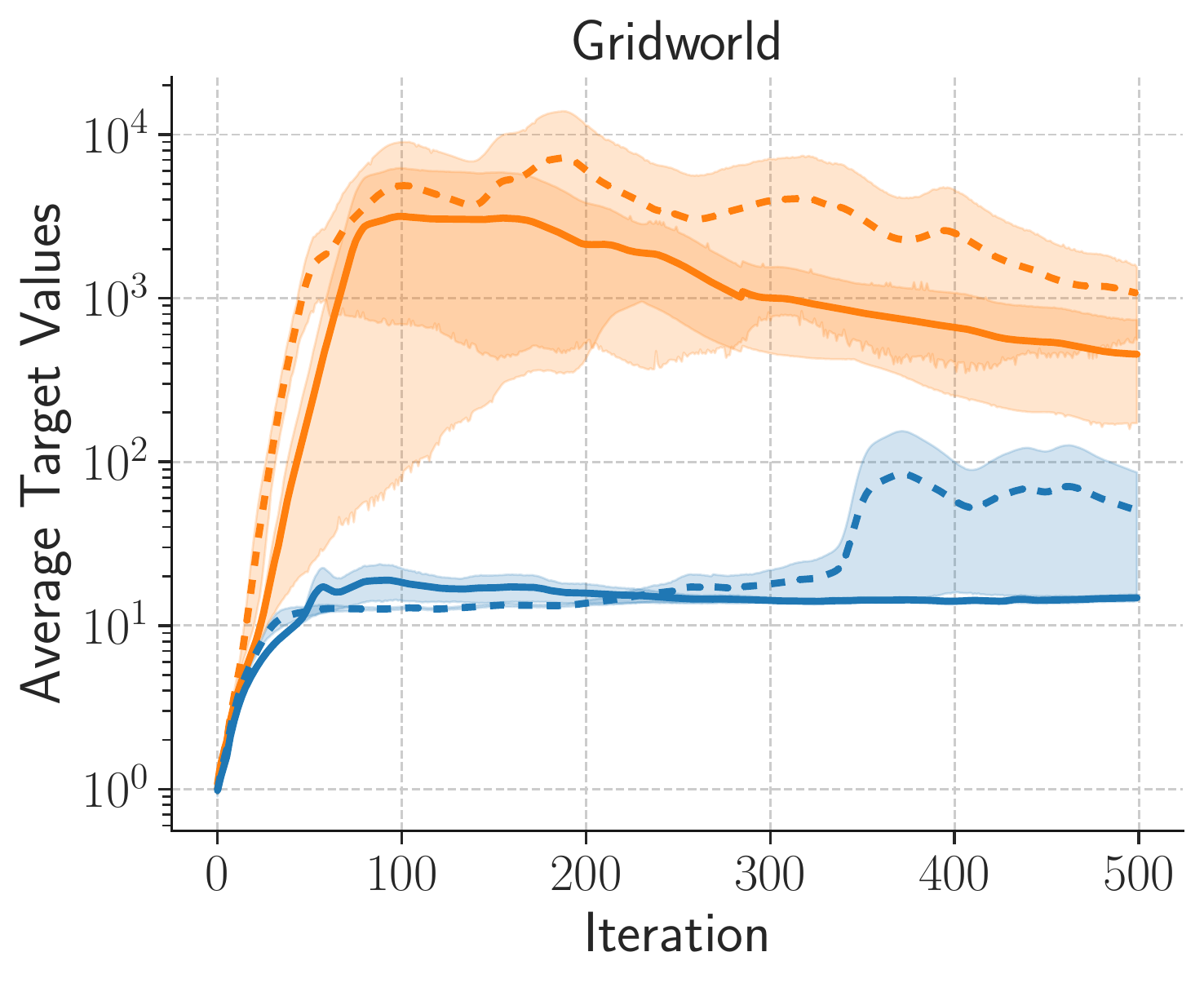}
        \caption{Average target values for transitions drawn from the replay buffer.}
    \end{subfigure}
    \hfill
    \begin{subfigure}{0.32\textwidth}
        \centering
        \includegraphics[width=\linewidth]{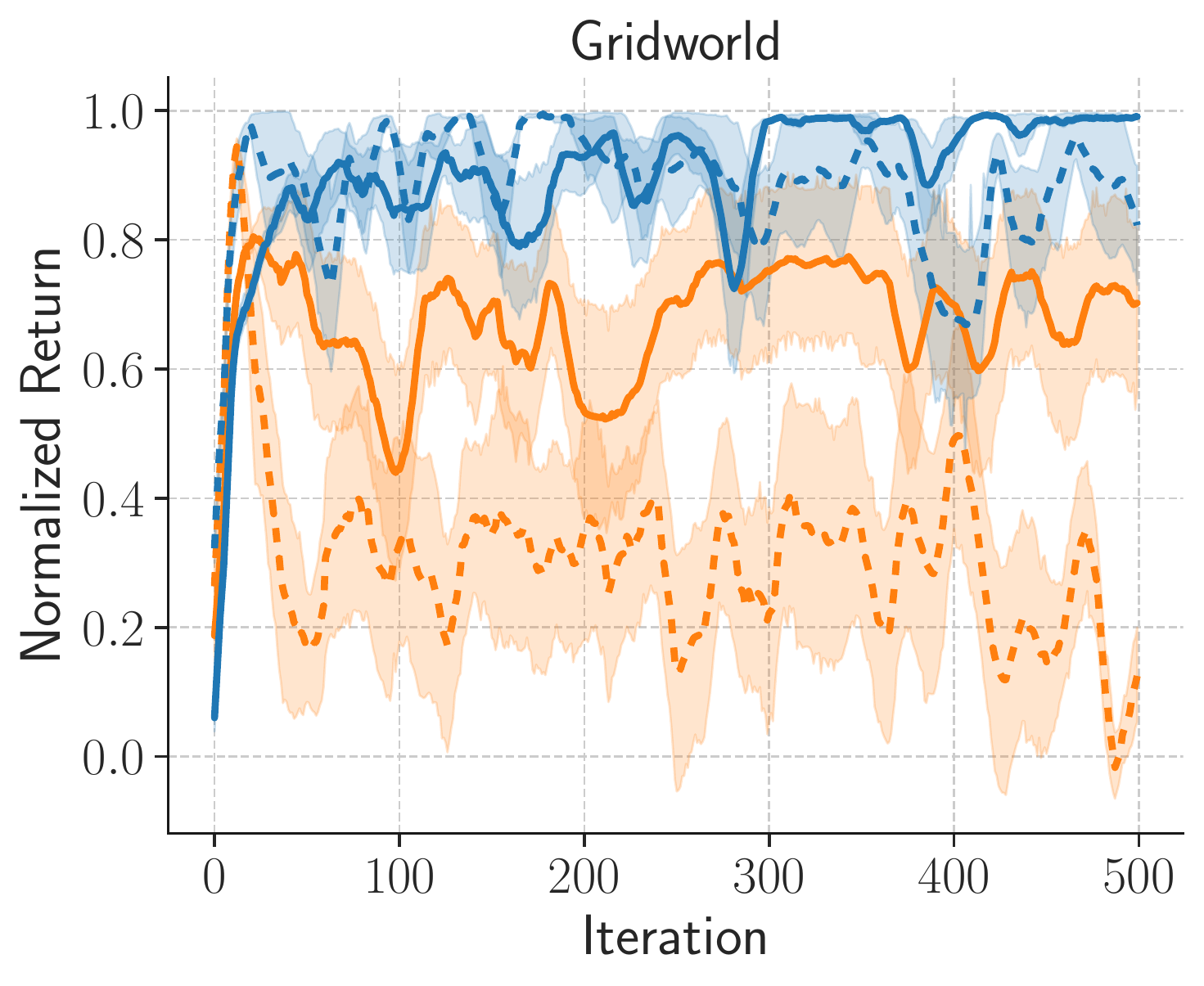}
        \caption{Normalized return achieved by policy $\pi(s) = \argmax_a Q_\theta(s, a)$.}
    \end{subfigure} \\
        \begin{subfigure}{0.32\textwidth}
        \centering
        \includegraphics[width=\linewidth]{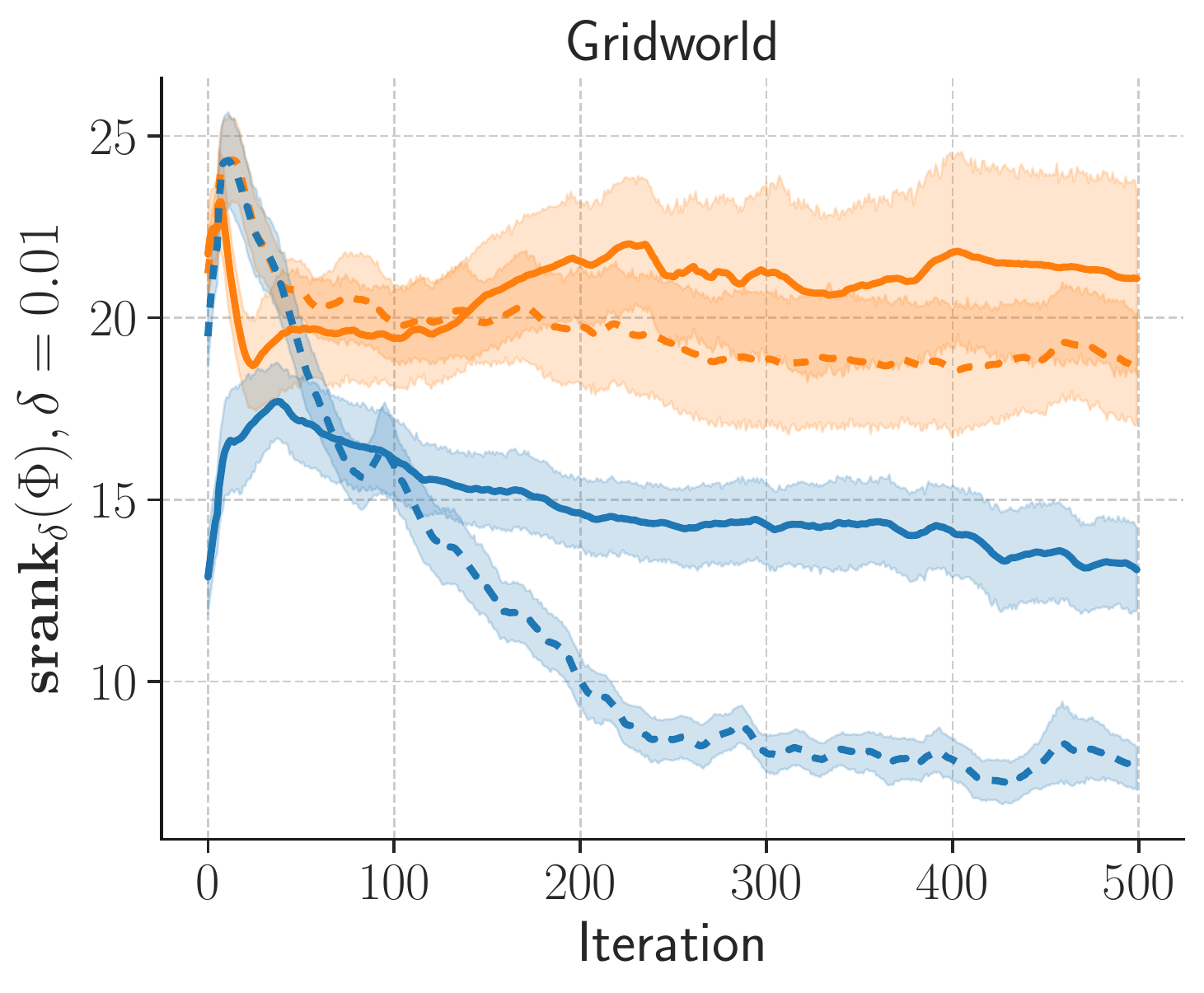}
        \caption{Effective rank of features over minibatch of states.}
    \end{subfigure} 
    \begin{subfigure}{0.32\textwidth}
        \centering
        \includegraphics[width=\linewidth]{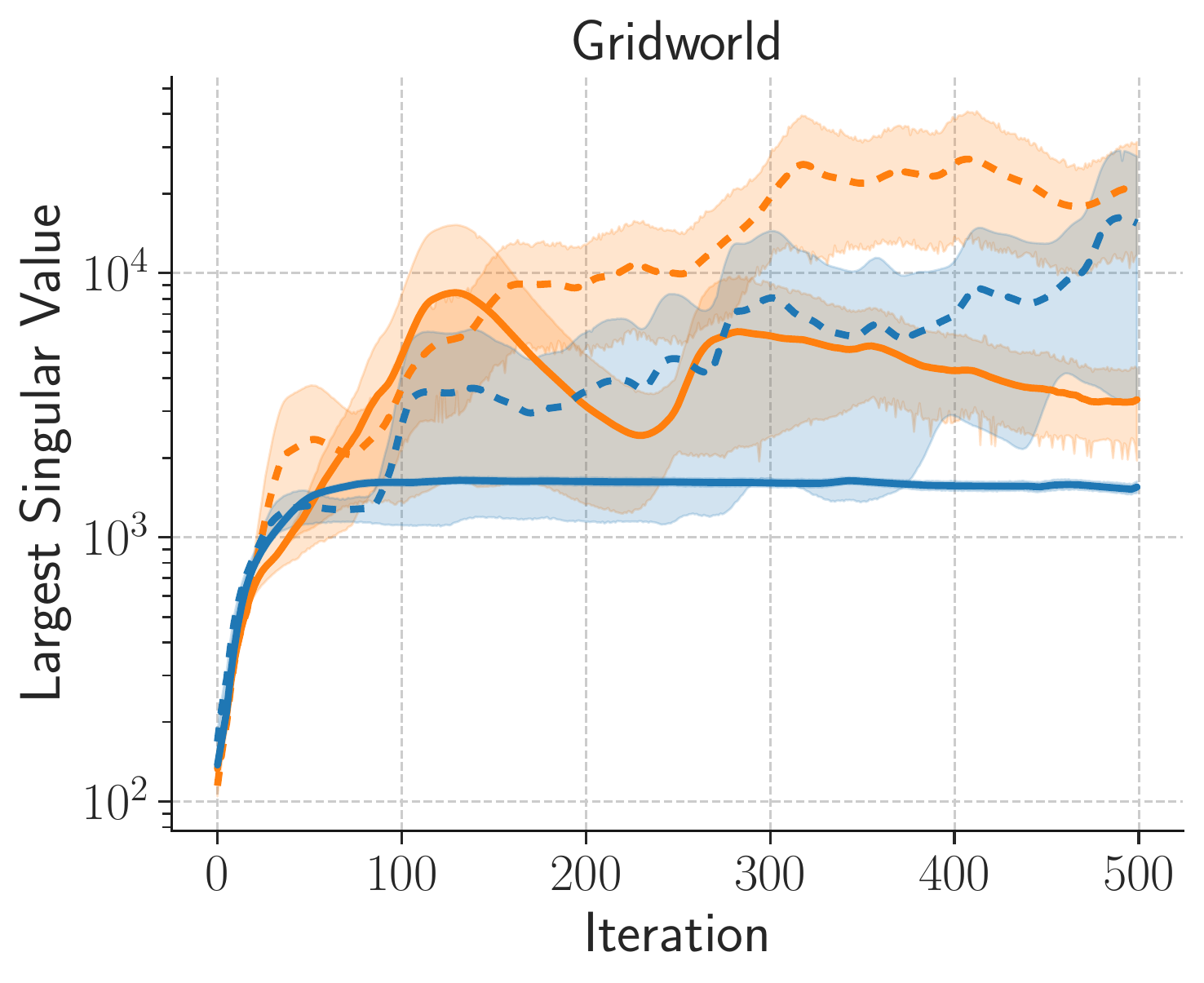}
        \caption{Largest singular value of learned features.}
    \end{subfigure}
    \hfill
    \begin{subfigure}{0.32\textwidth}
        \centering
        \includegraphics[width=\linewidth]{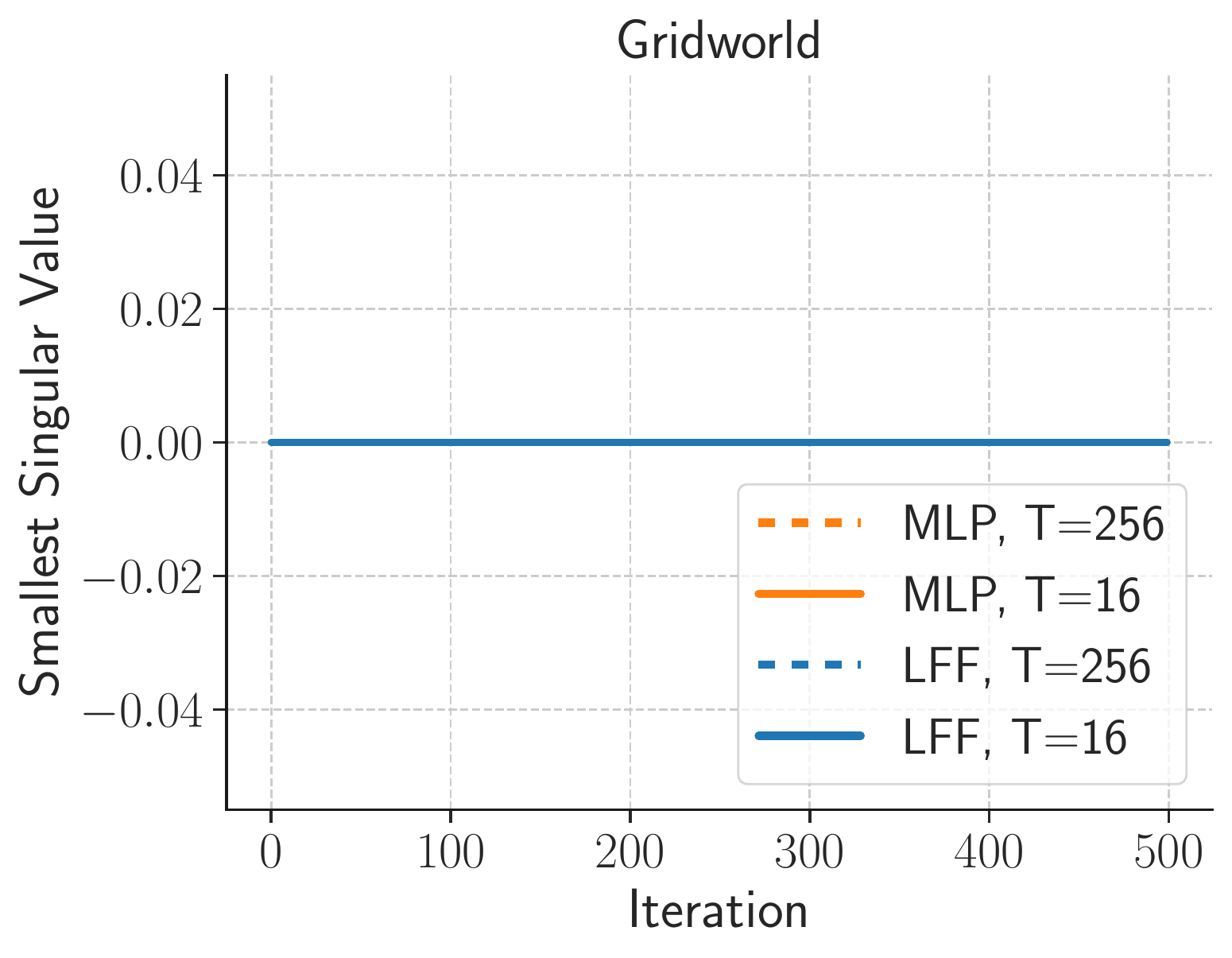}
        \caption{Smallest singular value of learned features. }
    \end{subfigure}
    \caption{Measuring noise amplification vs implicit underparameterization on \textsc{Grid16smoothobs}.
    }
    \label{fig:noise_amp}
\end{figure}
We want to examine whether noise amplification is indeed a problem that off-policy deep RL methods suffer from. Furthermore, if it is happening, \citet{kumar2020implicit} hypothesize that it is being caused by underfitting (``implicit underparameterization''), which would contradict our claim that LFF improves learning by regularizing the training dynamics. We test this in Figure \ref{fig:noise_amp} by performing Fitted Q-iteration (FQI) on the \textsc{Grid16smoothobs} environment from \citet{kumar2020implicit}.
\textsc{Grid16smoothobs} is a discrete environment with 256 states and 5 actions, so we can use Q-iteration to calculate the optimal $Q^*$ and compare that with our learned $Q_\theta$. $T$ denotes the number of gradient steps per FQI iteration; increasing the number of gradient steps is empirically more likely to cause divergence for MLP Q-functions. Note that we use a log scale for the y-axis in (a,b,e).

\paragraph{Noise amplification}
\textbf{(a)} shows that MLP-based Q-functions steadily blow up to orders of magnitude above the true $Q^*$, whose average value is around 15 in this environment. In contrast, LFF-based Q-functions either converge stably to the correct magnitude, or resist increasing as much. \textbf{(b)} shows that the MLP target values are in a positive feedback loop with the Q-values. \textbf{(c)} shows that divergence coincides with a drop in the returns.
Together, these results indicate that there can be a harmful feedback loop in the bootstrapping process, and that methods like LFF, which reduce fitting to noise, can help stabilize training.

\paragraph{Implicit Underparameterization}
We run Fitted Q-iteration and calculate the effective rank of the Q-network, which we parameterize using either a MLP or LFF network. We follow the procedure from \citet{kumar2020implicit}: sample 2048 states from the replay buffer and calculate the singular values $\sigma_i$ of the aggregated feature matrix $\Phi$. The effective rank is then defined as $\mathbf{srank}_\delta(\Phi) = \min\left\{k: \frac{\sum_{i=1}^k \sigma_i(\Phi)}{\sum_{j=1}^d \sigma_j(\Phi)} \geq 1 - \delta \right\}$. \textbf{(d)} shows that the MLP's effective rank does not actually drop over training. Furthermore, LFF is able to avoid diverging $Q$-values, even though it has signficantly lower srank than its MLP counterpart. 
Thus, noise amplification for MLPs in this setting is likely not related to any underfitting measured by the effective rank. 
\textbf{(e)} shows that the largest singular value blows up for MLPs, but stabilizes when training LFF for a reasonable number of gradient steps. \textbf{(f)} shows that the minimum singular value stays at zero over the course of training. In the context of \citet{kumar2020implicit}, this implies that their penalty $\sigma_{max}(\Phi)^2 - \sigma_{min}(\Phi)^2$ is exactly equivalent to penalizing only the maximum singular value $\sigma_{max}(\Phi)^2$ when using gradient descent. This is because the gradient of the second term is zero when the smallest singular value is zero. Thus, \citet{kumar2020implicit}'s penalty works by constraining the largest singular value and regularizing the magnitude of the feature matrix. Overall, these results support our hypothesis that noise amplification is a problem that is not caused by underfitting and can be ameliorated by regularization.

\section{Further Ablations}
\begin{figure}[H]
    \centering
    \includegraphics[width=0.24\textwidth]{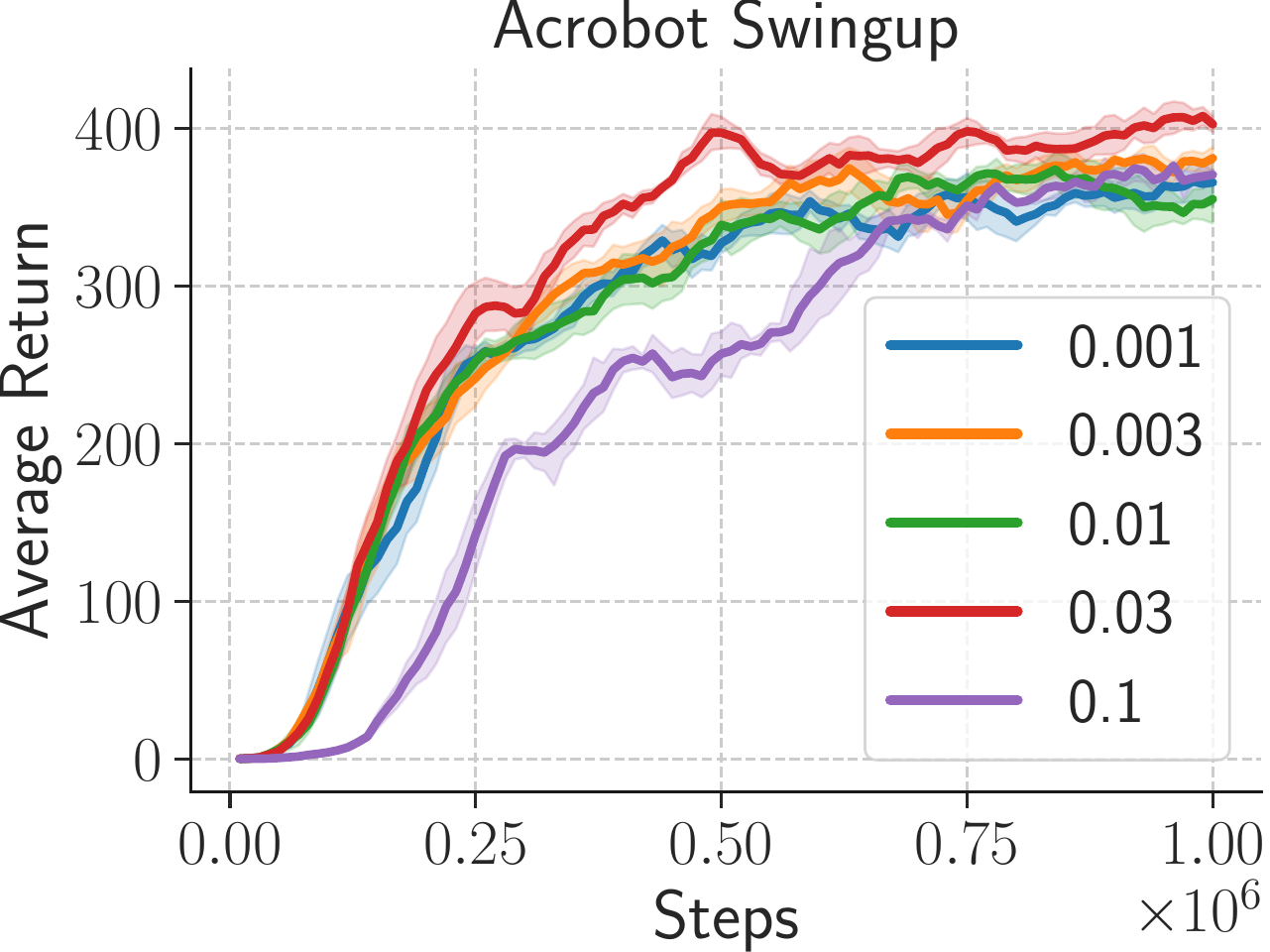}
    \includegraphics[width=0.24\textwidth]{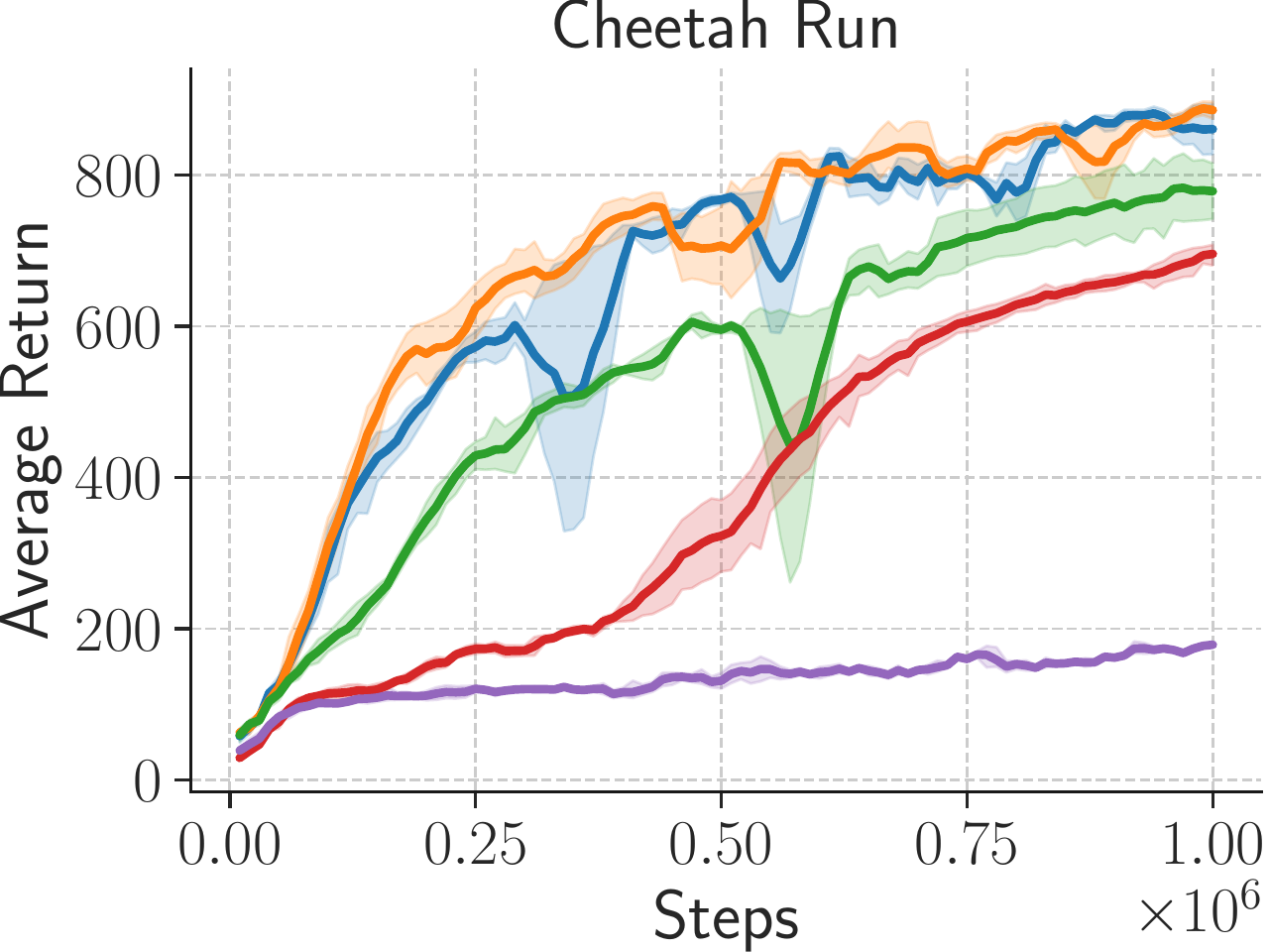}
    \includegraphics[width=0.24\textwidth]{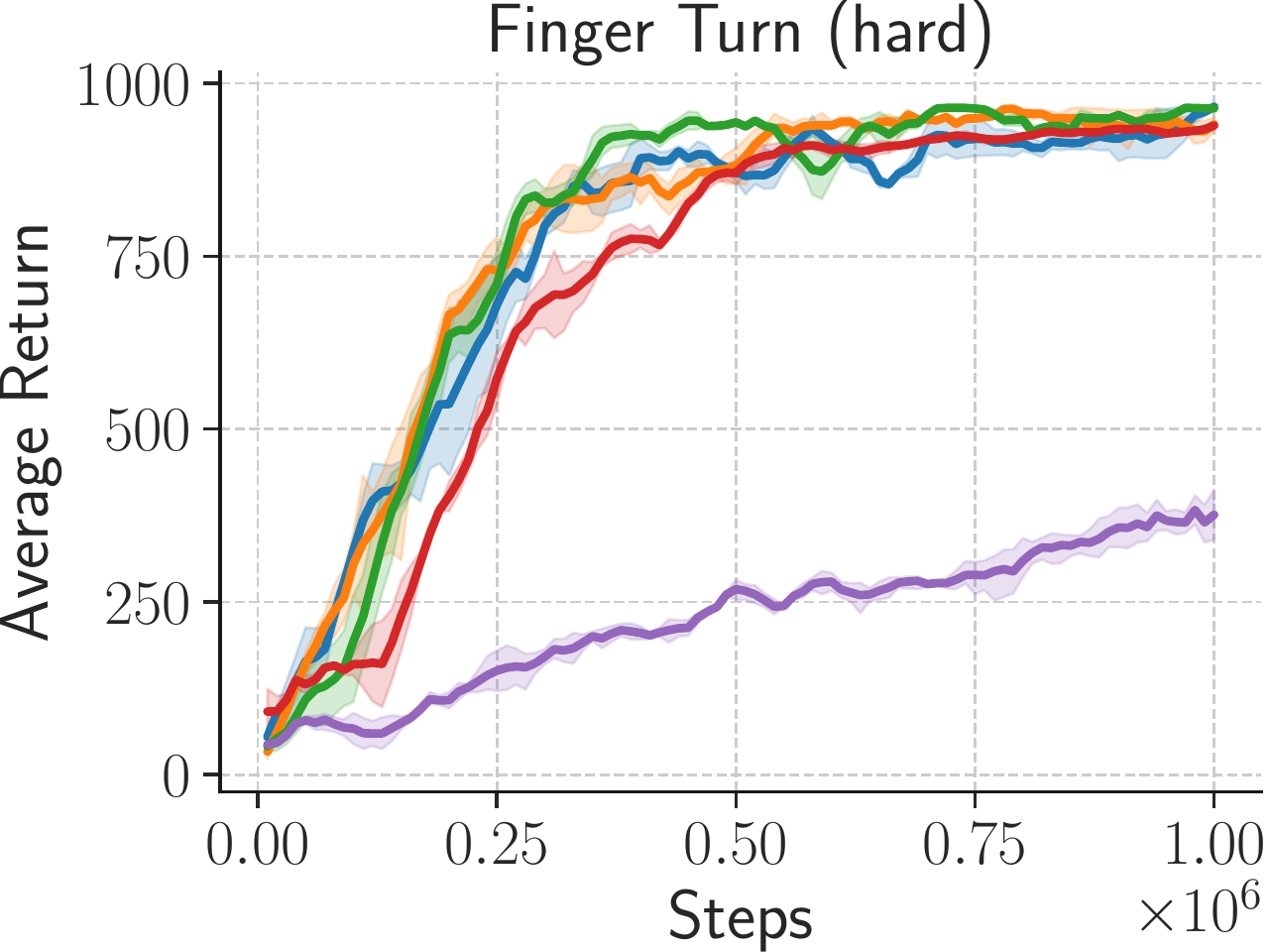}
    \includegraphics[width=0.24\textwidth]{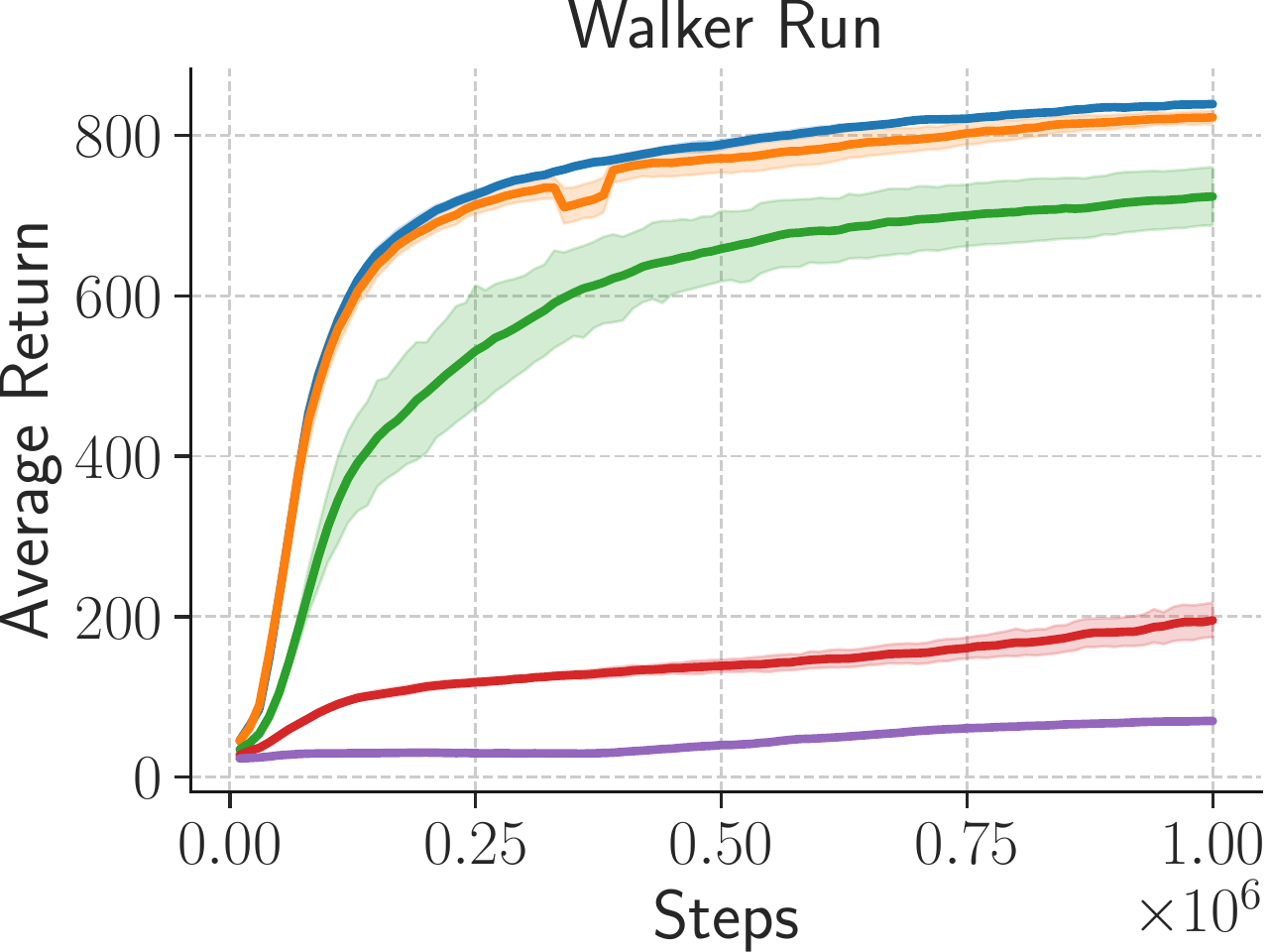} \\
    \includegraphics[width=0.24\textwidth]{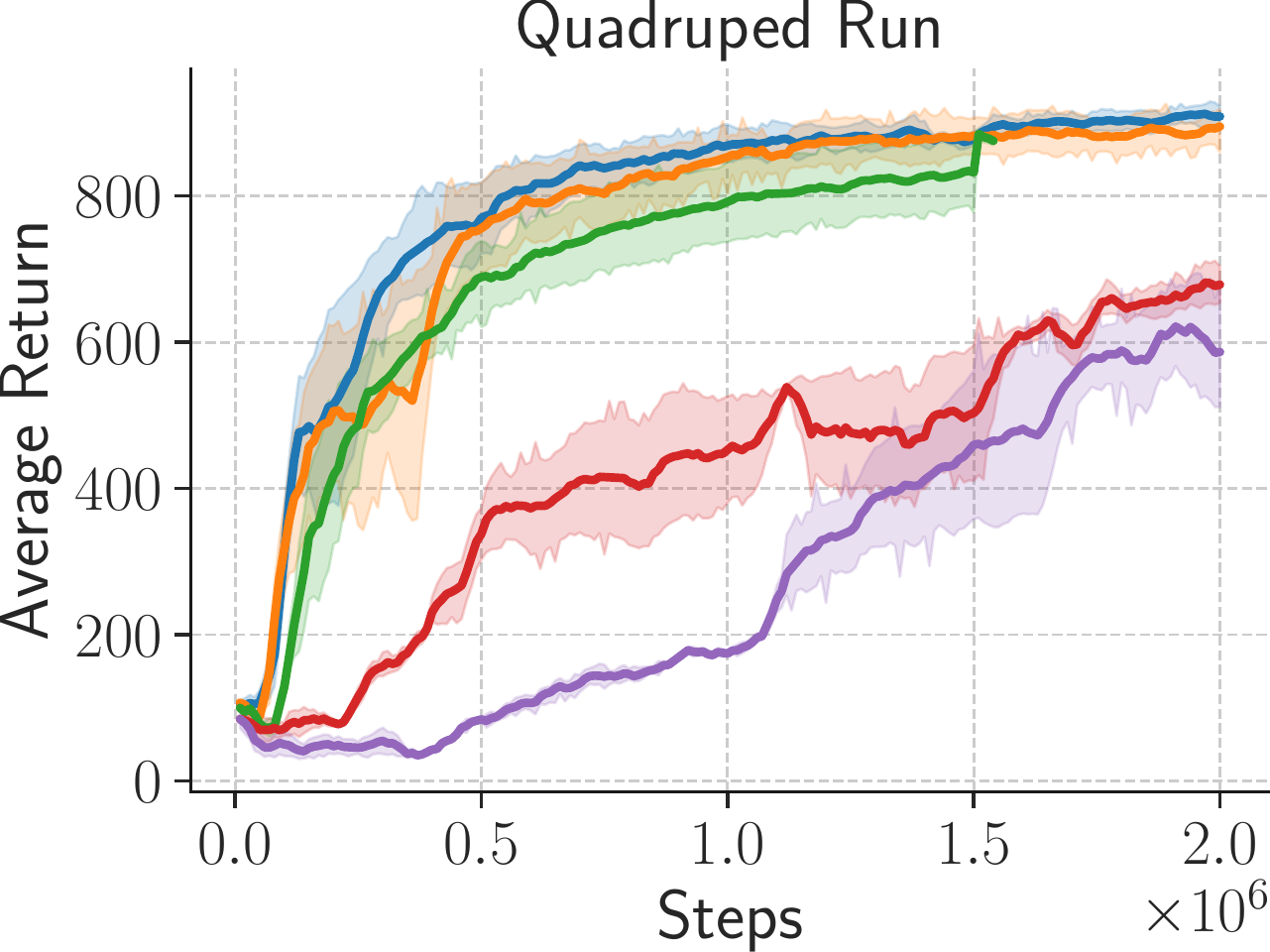}
    \includegraphics[width=0.24\textwidth]{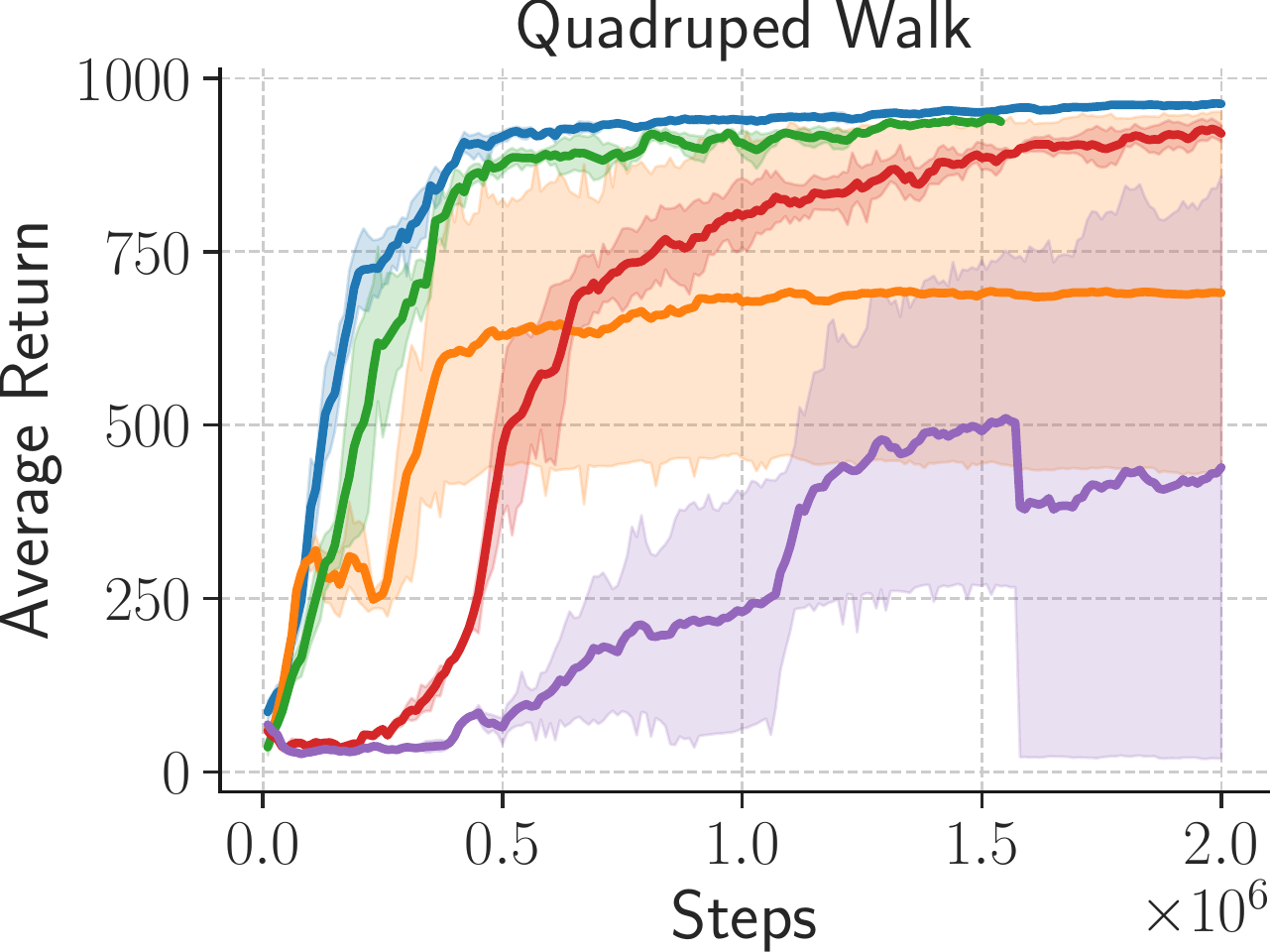}
    \includegraphics[width=0.24\textwidth]{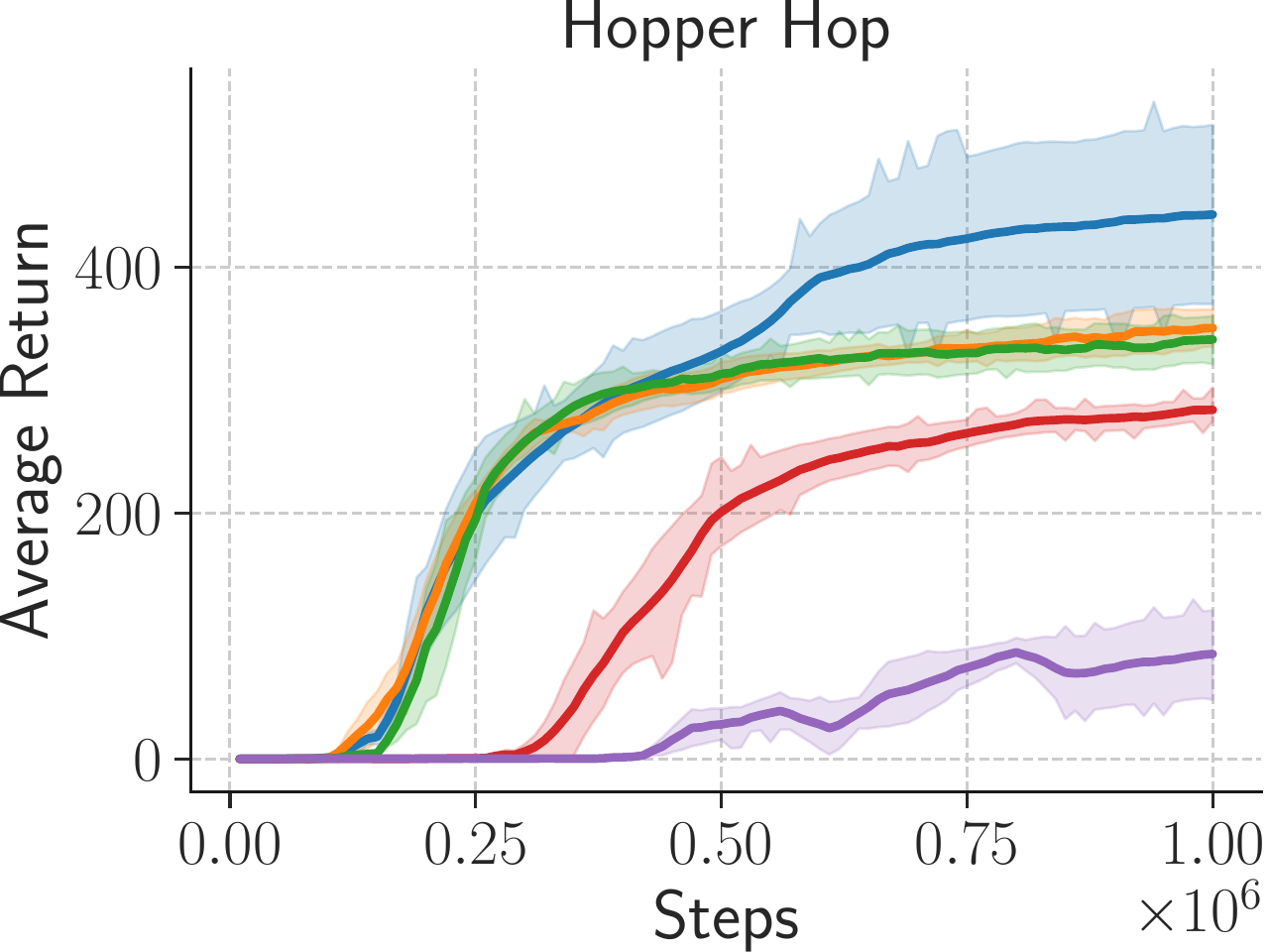}
    \includegraphics[width=0.24\textwidth]{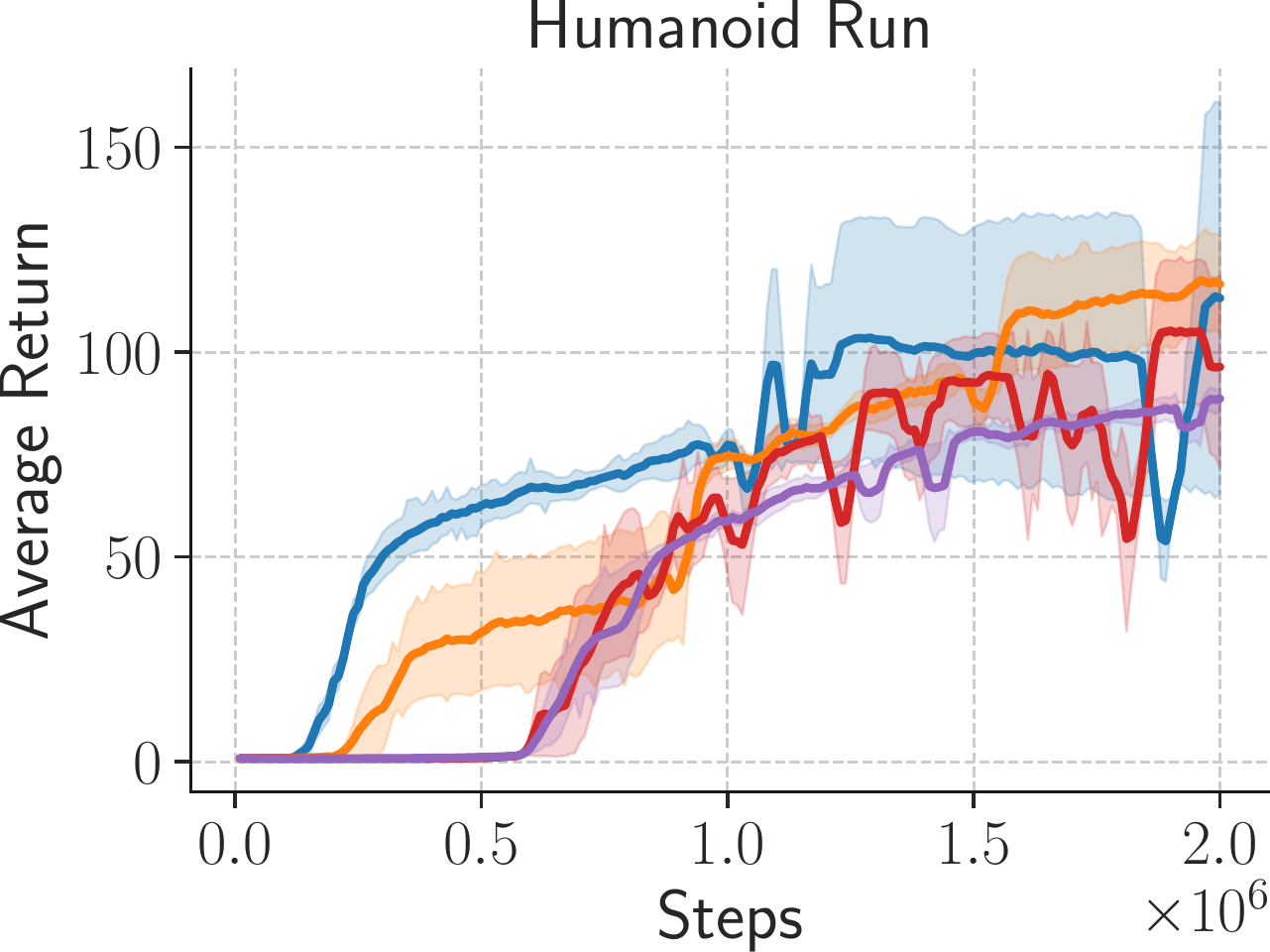}
    \caption{\textbf{Sensitivity to $\sigma$.} $\sigma=0.001$ is a good default across all of these state-based environments. Results are averaged over 5 seeds, using a Fourier dimension of 1024, and the shaded region denotes 1 standard error.}
    \label{fig:sigma_sensitivity}
\end{figure}

\begin{figure}[H]
    \centering
    \includegraphics[width=0.24\textwidth]{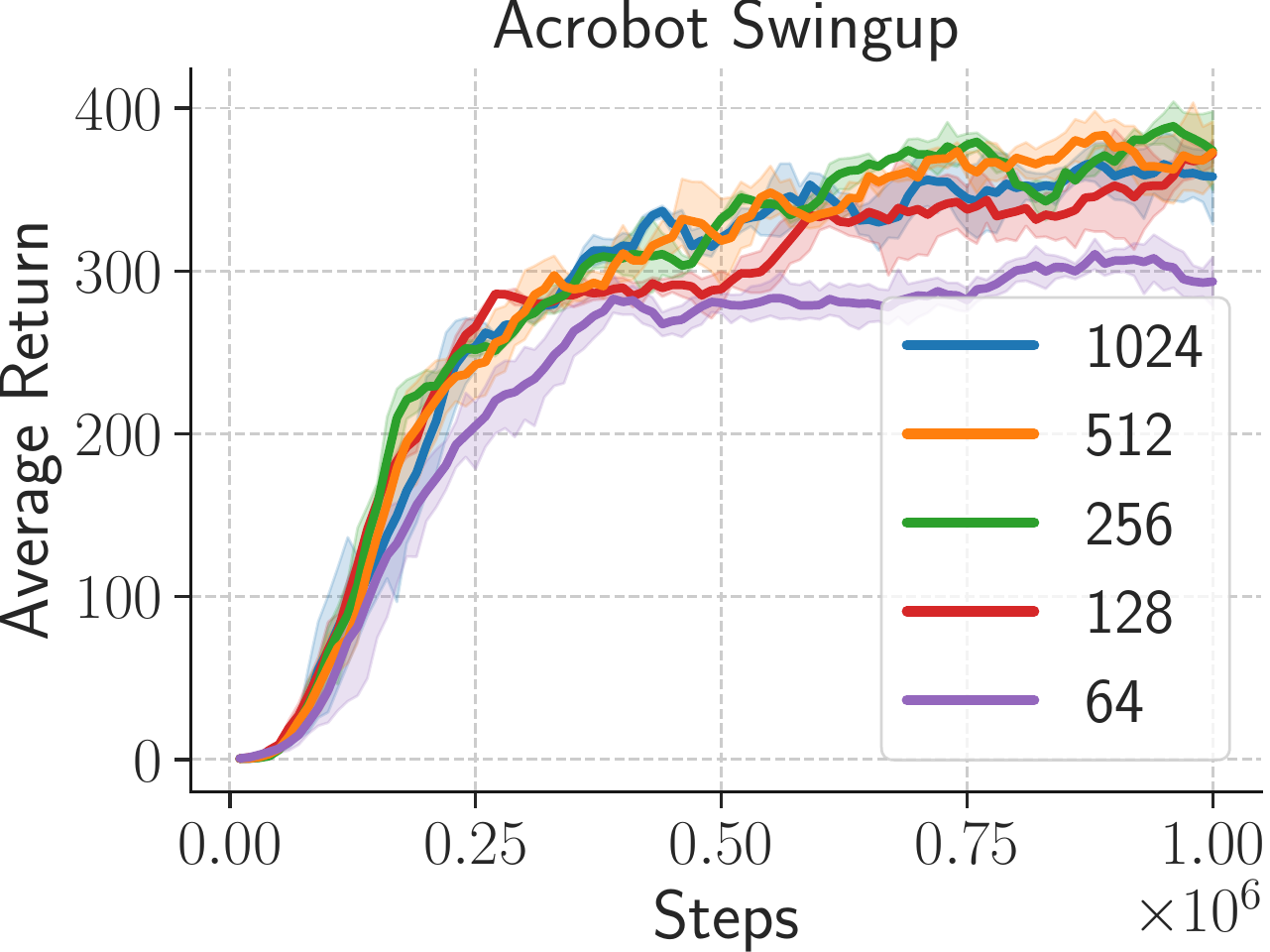}
    \includegraphics[width=0.24\textwidth]{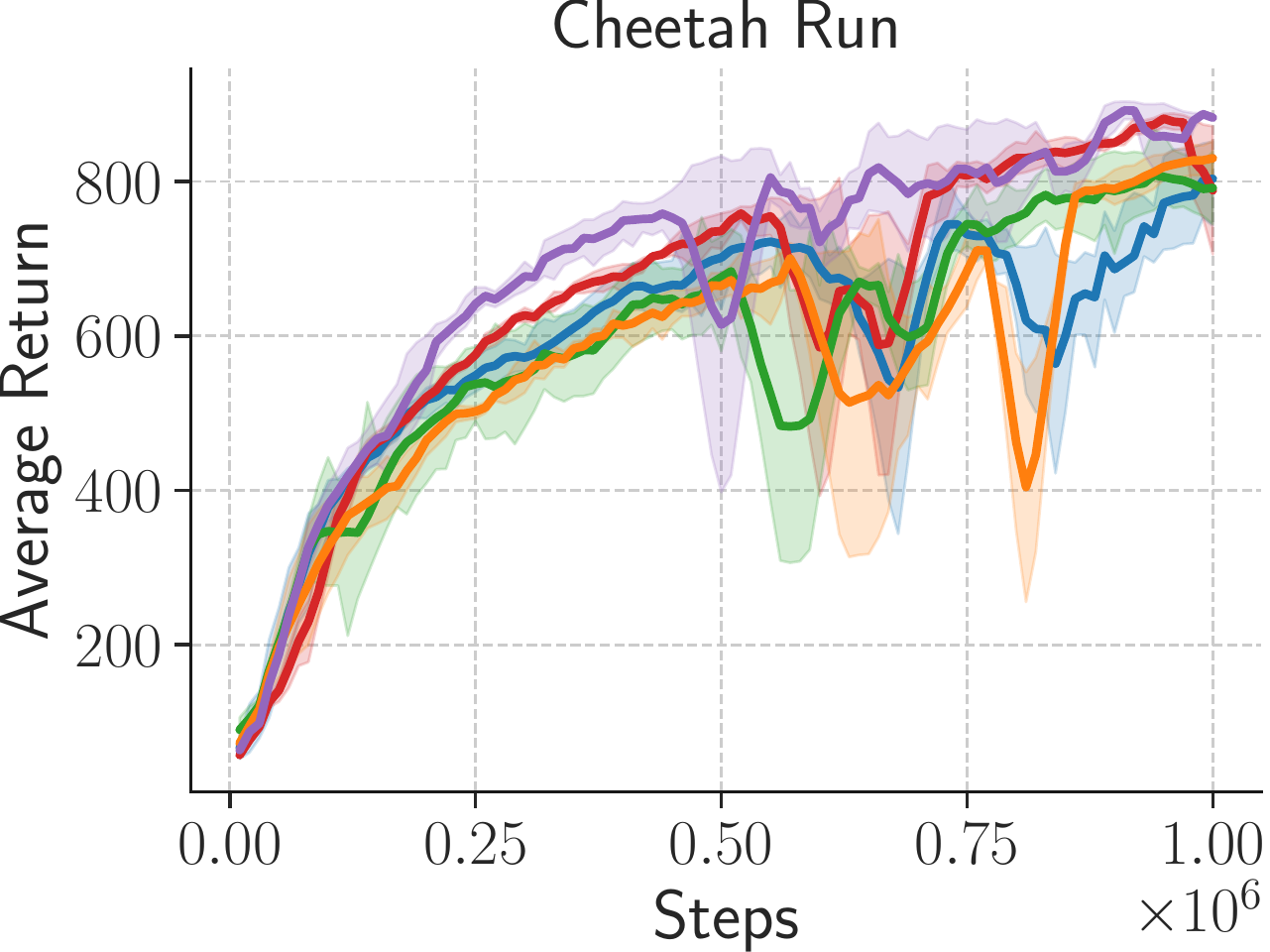}
    \includegraphics[width=0.24\textwidth]{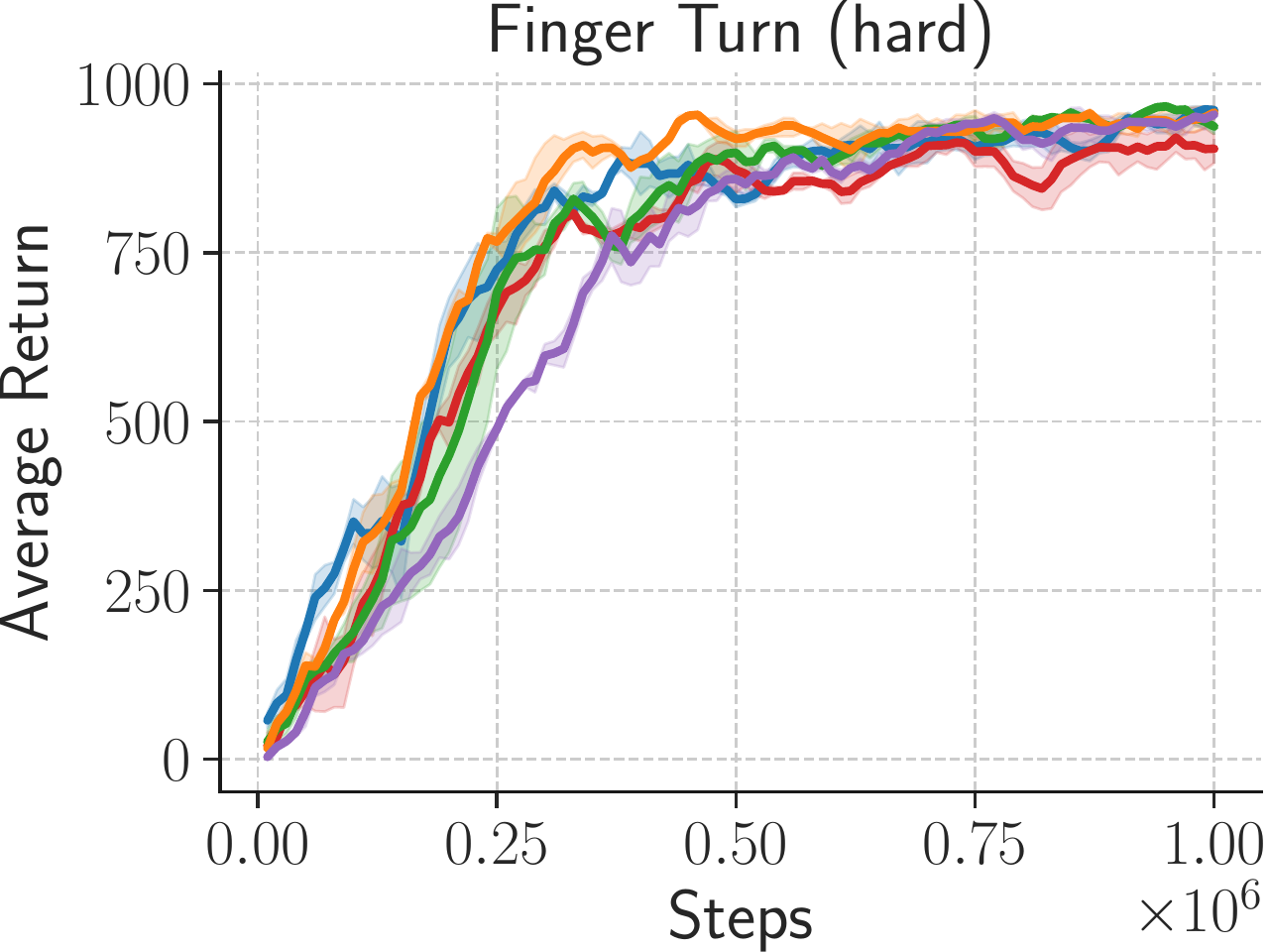}
    \includegraphics[width=0.24\textwidth]{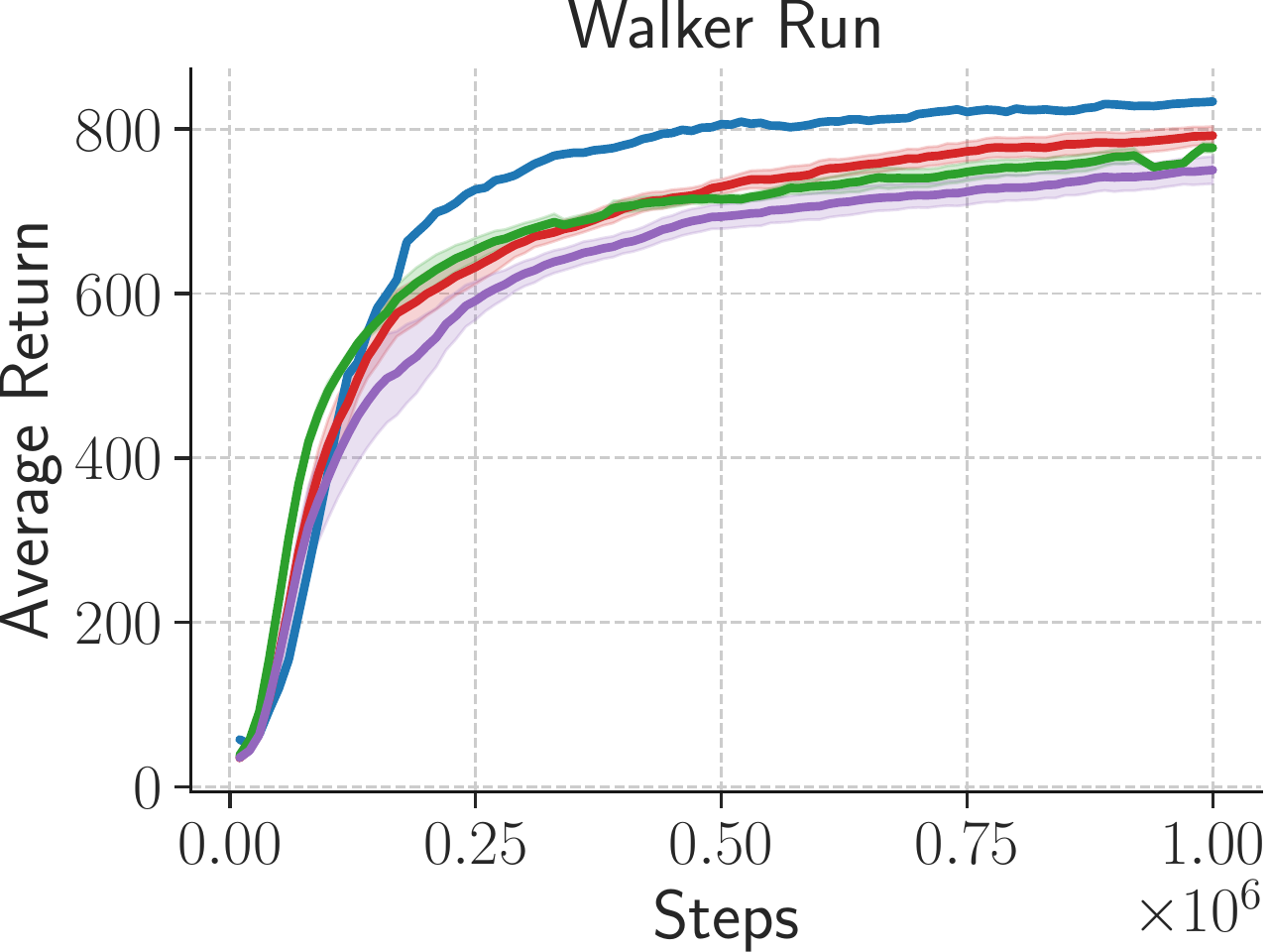} \\
    \includegraphics[width=0.24\textwidth]{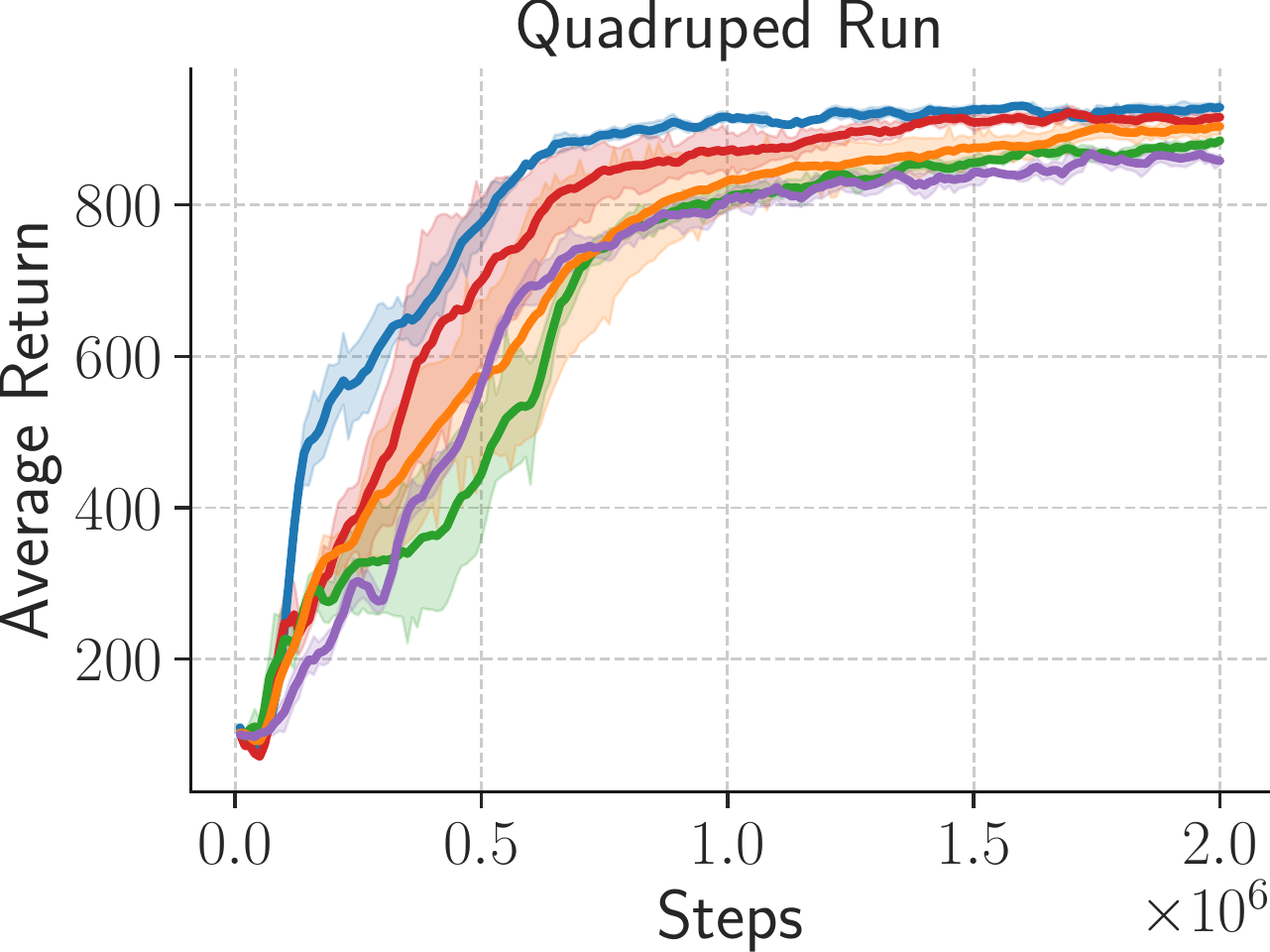}
    \includegraphics[width=0.24\textwidth]{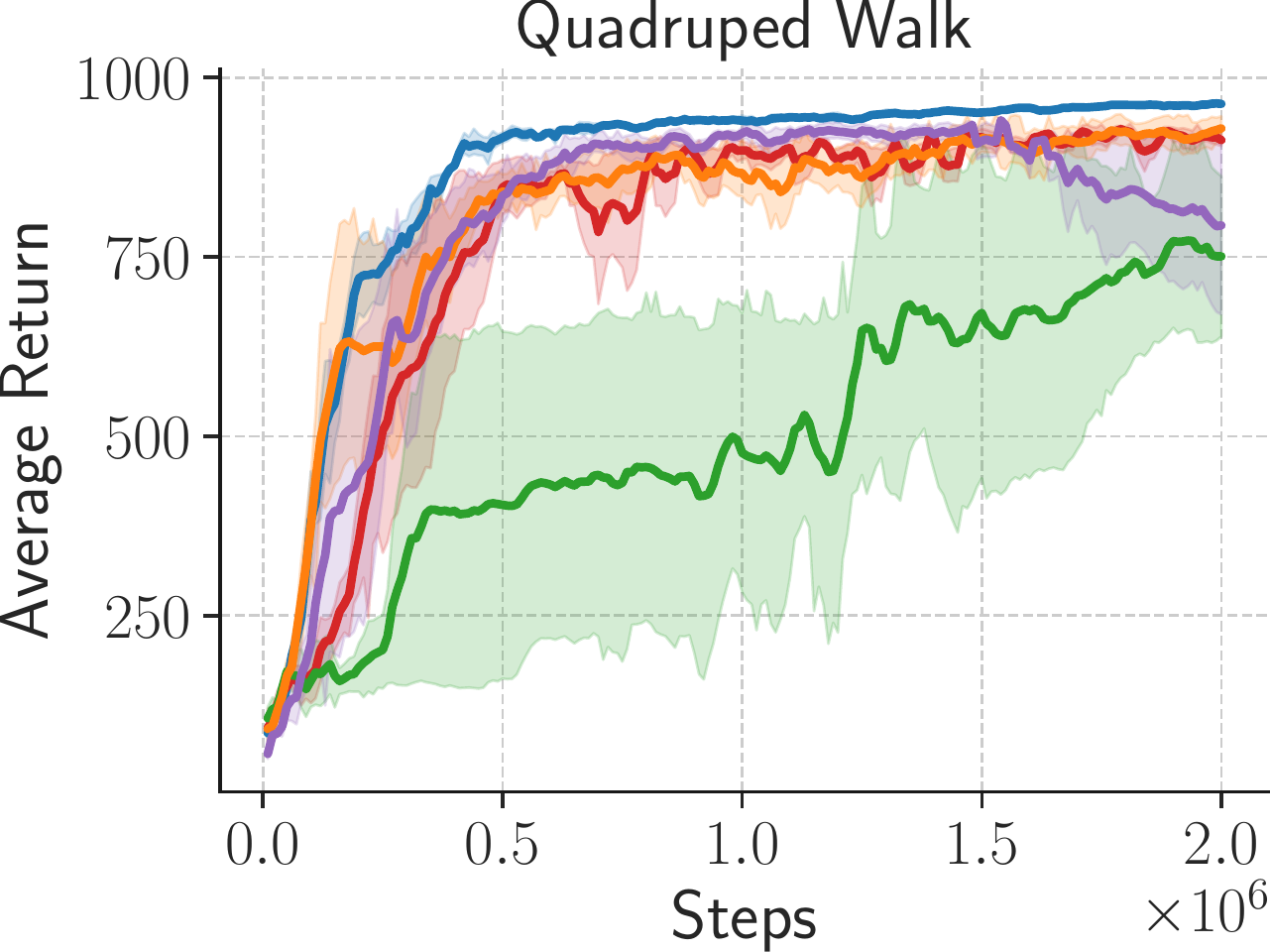}
    \includegraphics[width=0.24\textwidth]{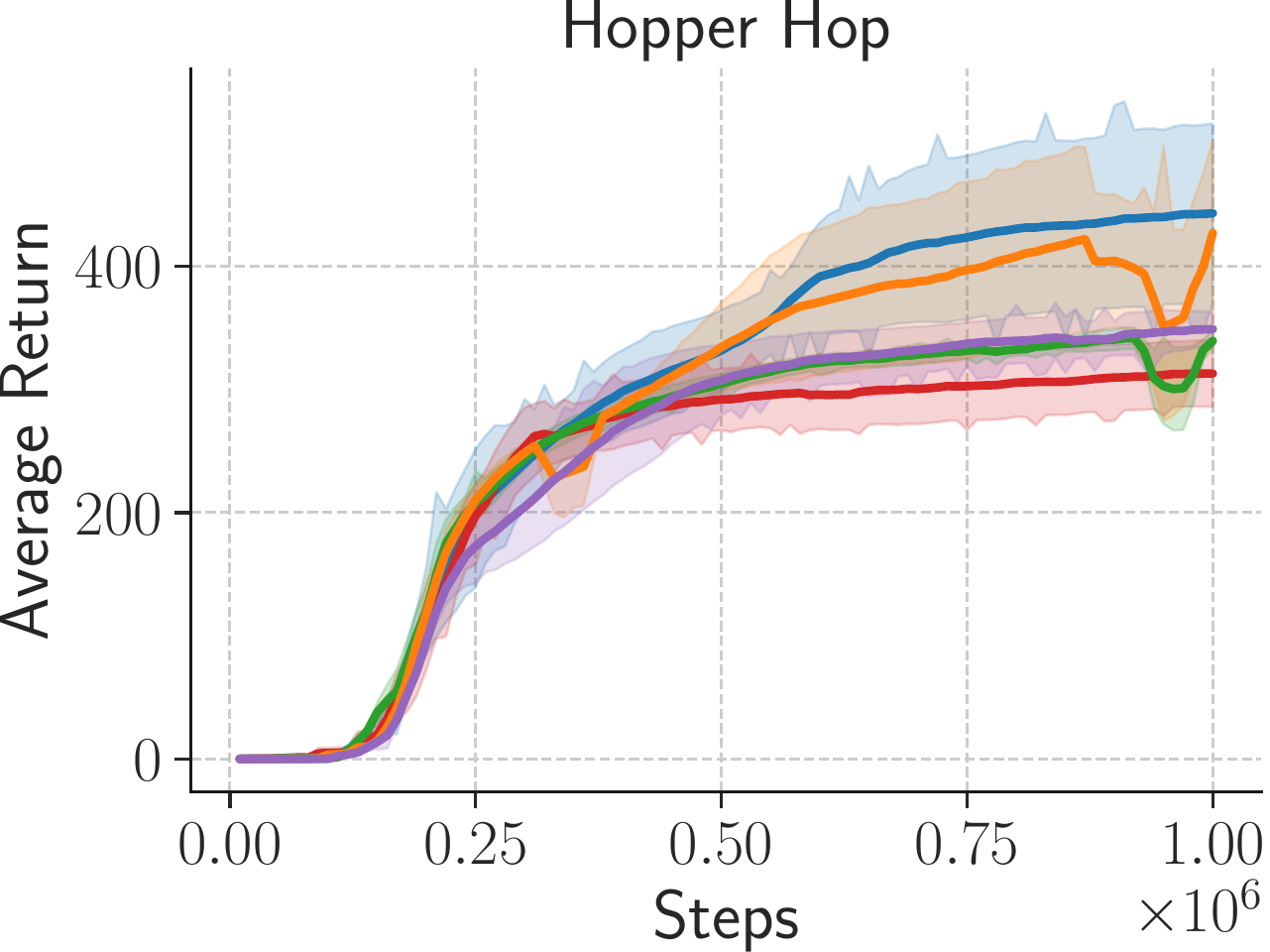}
    \includegraphics[width=0.24\textwidth]{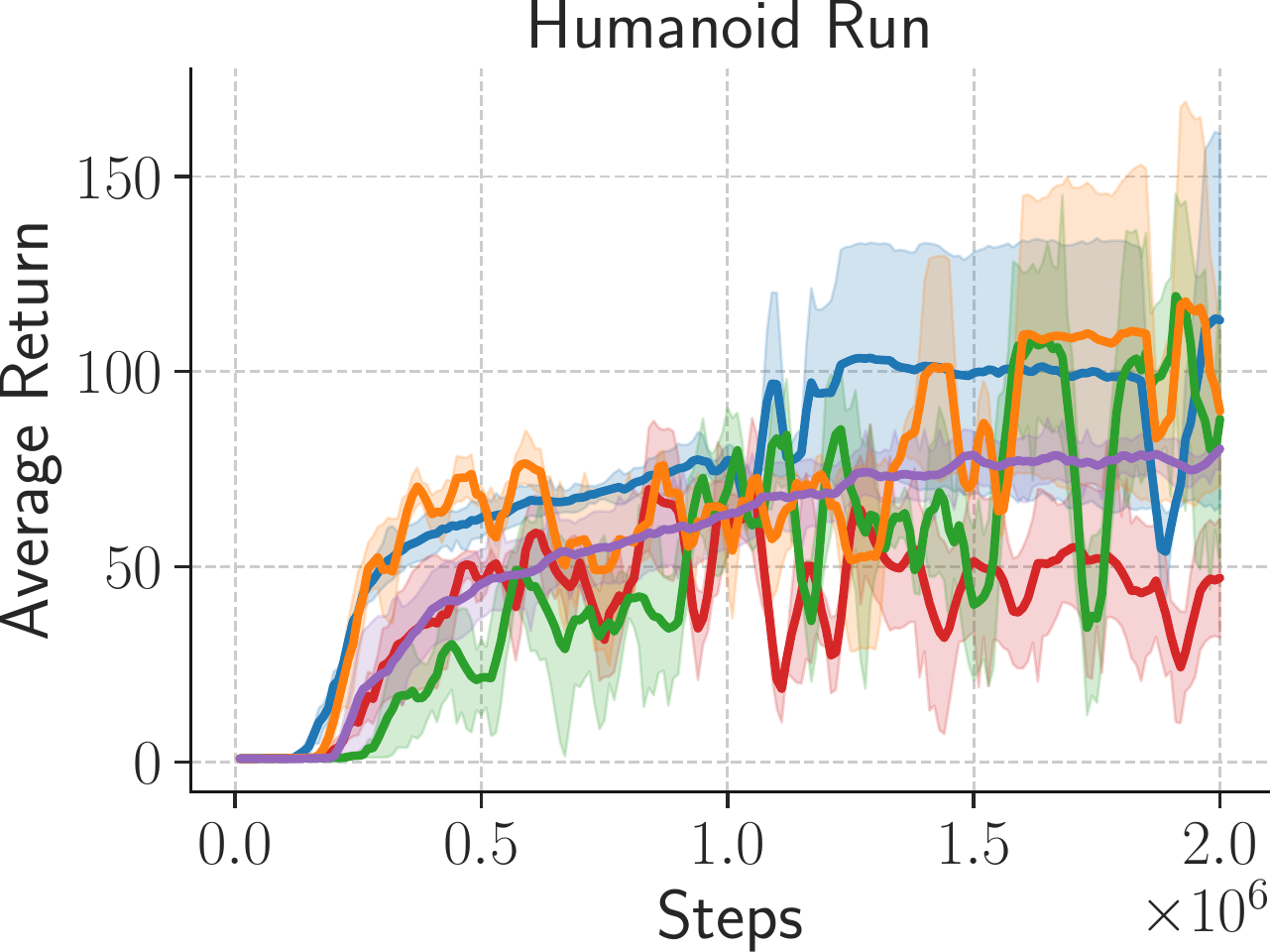}
    \caption{\textbf{Sensitivity to the Fourier dimension.} A Fourier dimension of 1024 is a good default across all of these state-based environments. Results are averaged over 5 seeds, using $\sigma = 0.001$, and the shaded region denotes 1 standard error.}
    \label{fig:fourier_dim_sensitivity}
\end{figure}

\newpage
\section{High Frequency Learning with Learned Fourier Features}
\label{sec:high_freq}
\begin{figure}[H]
    \centering
    \includegraphics[height=2.9cm]{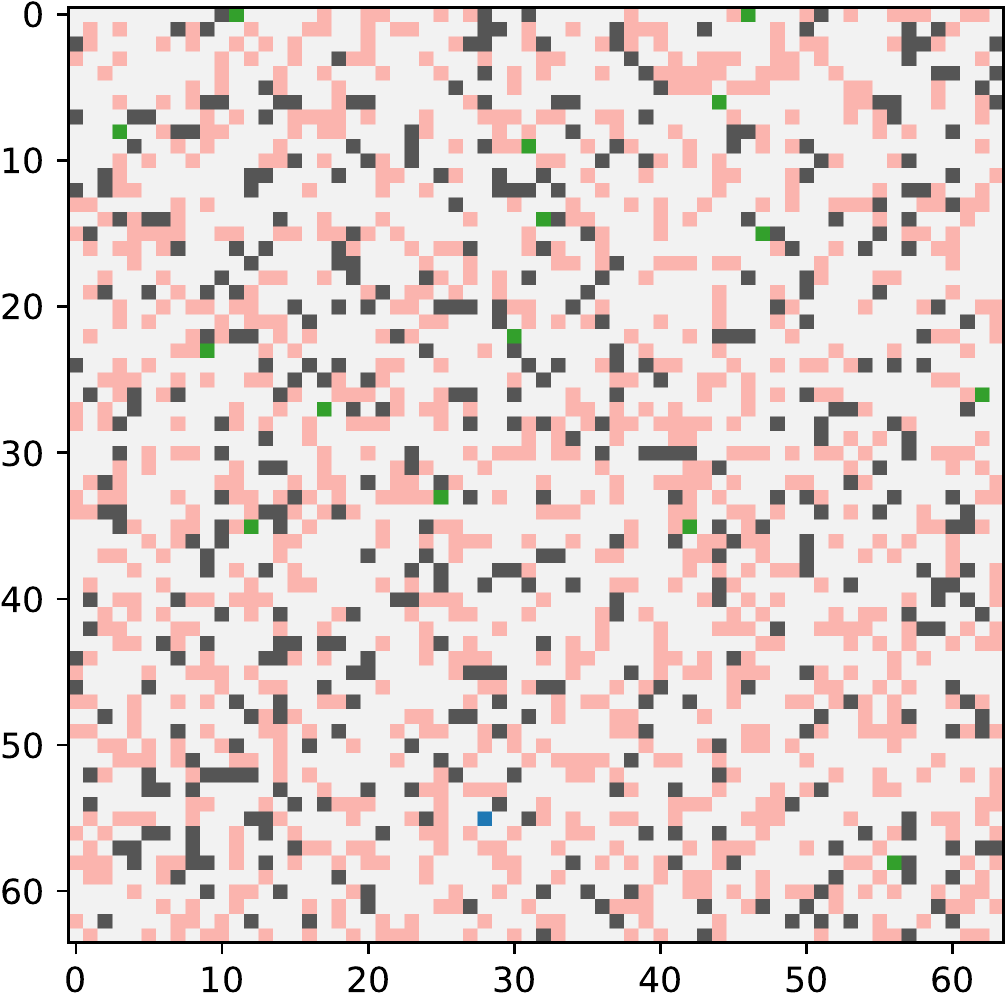}
    \includegraphics[height=2.9cm]{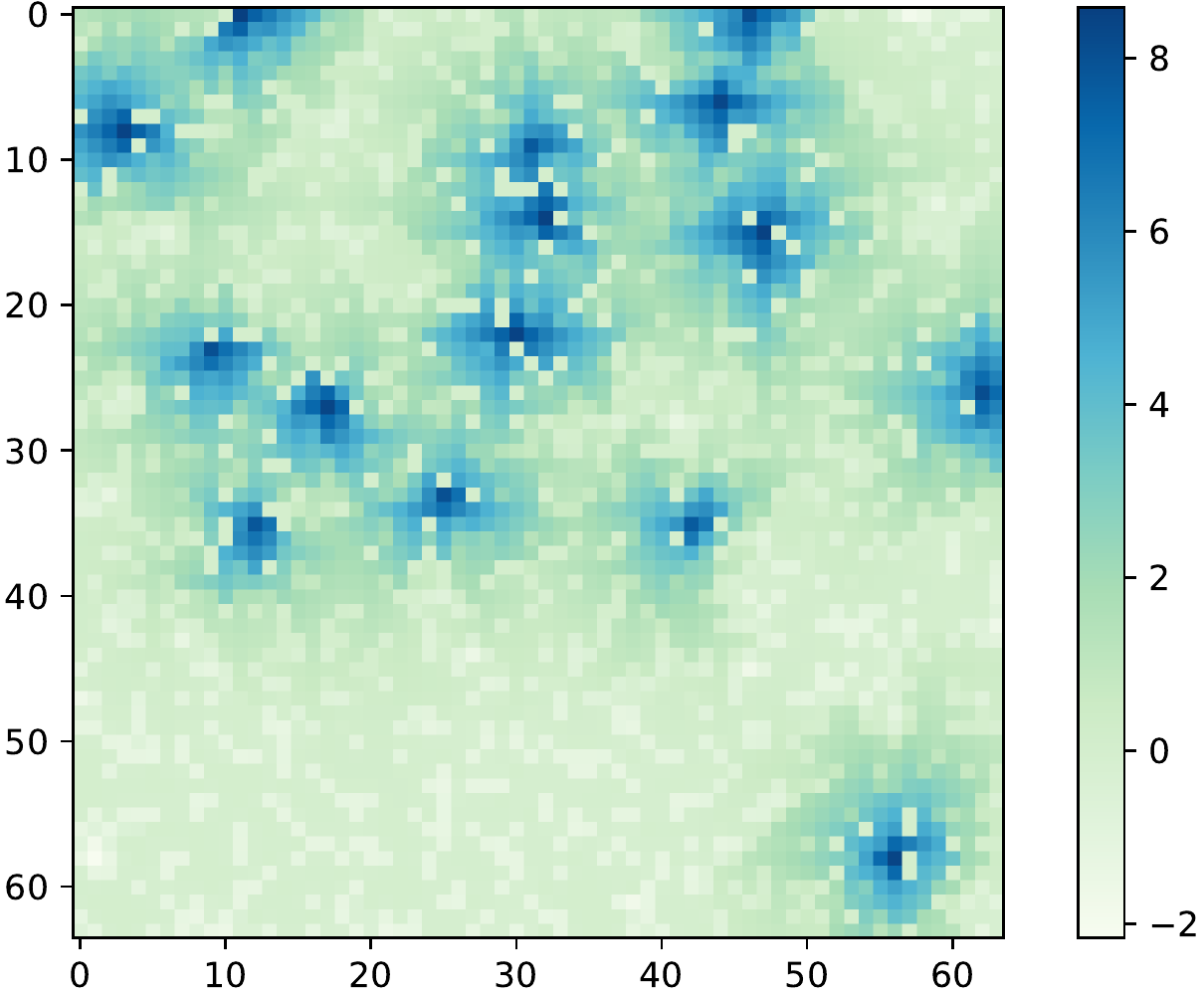}
    \includegraphics[height=2.9cm]{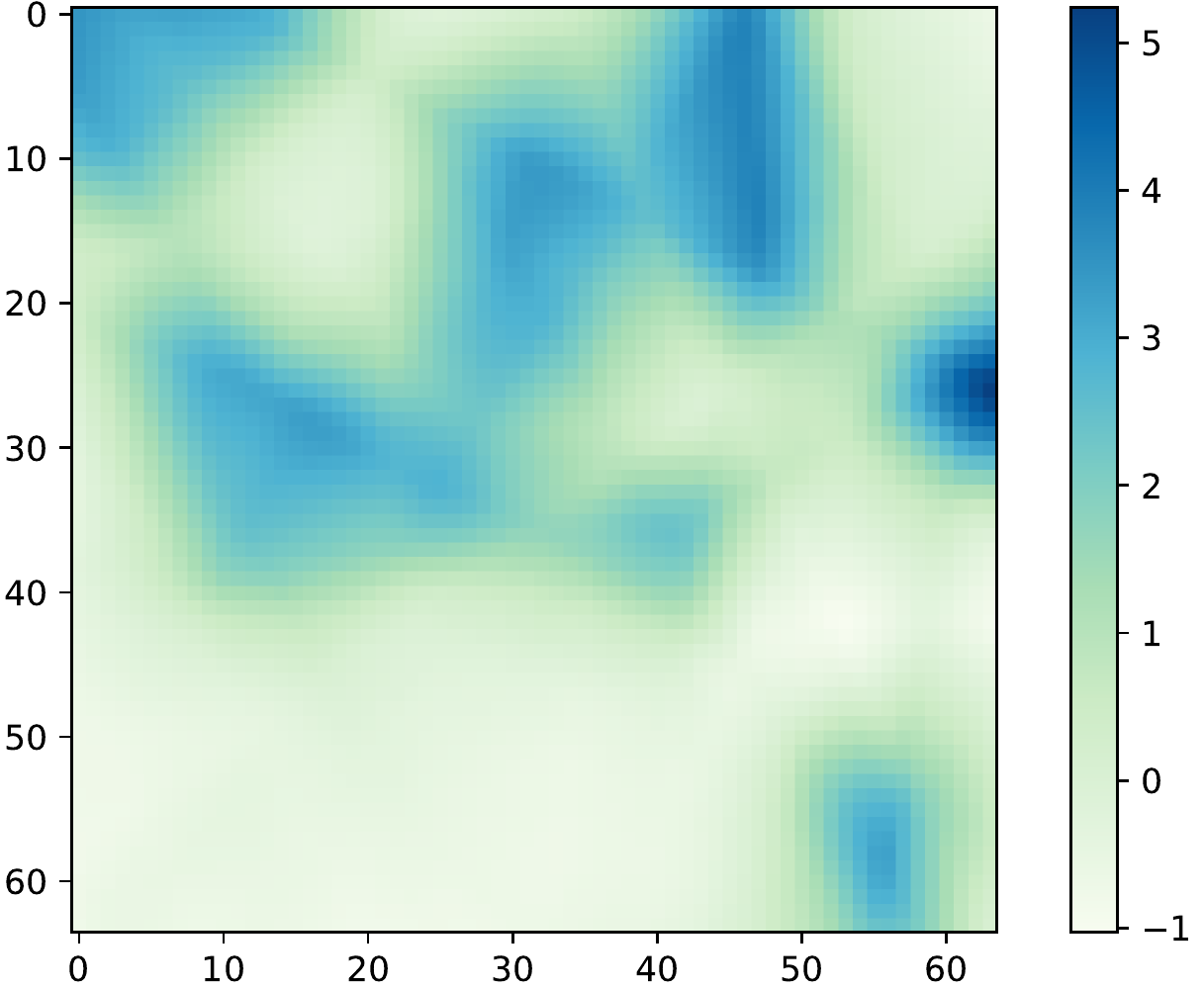}
    \includegraphics[height=2.9cm]{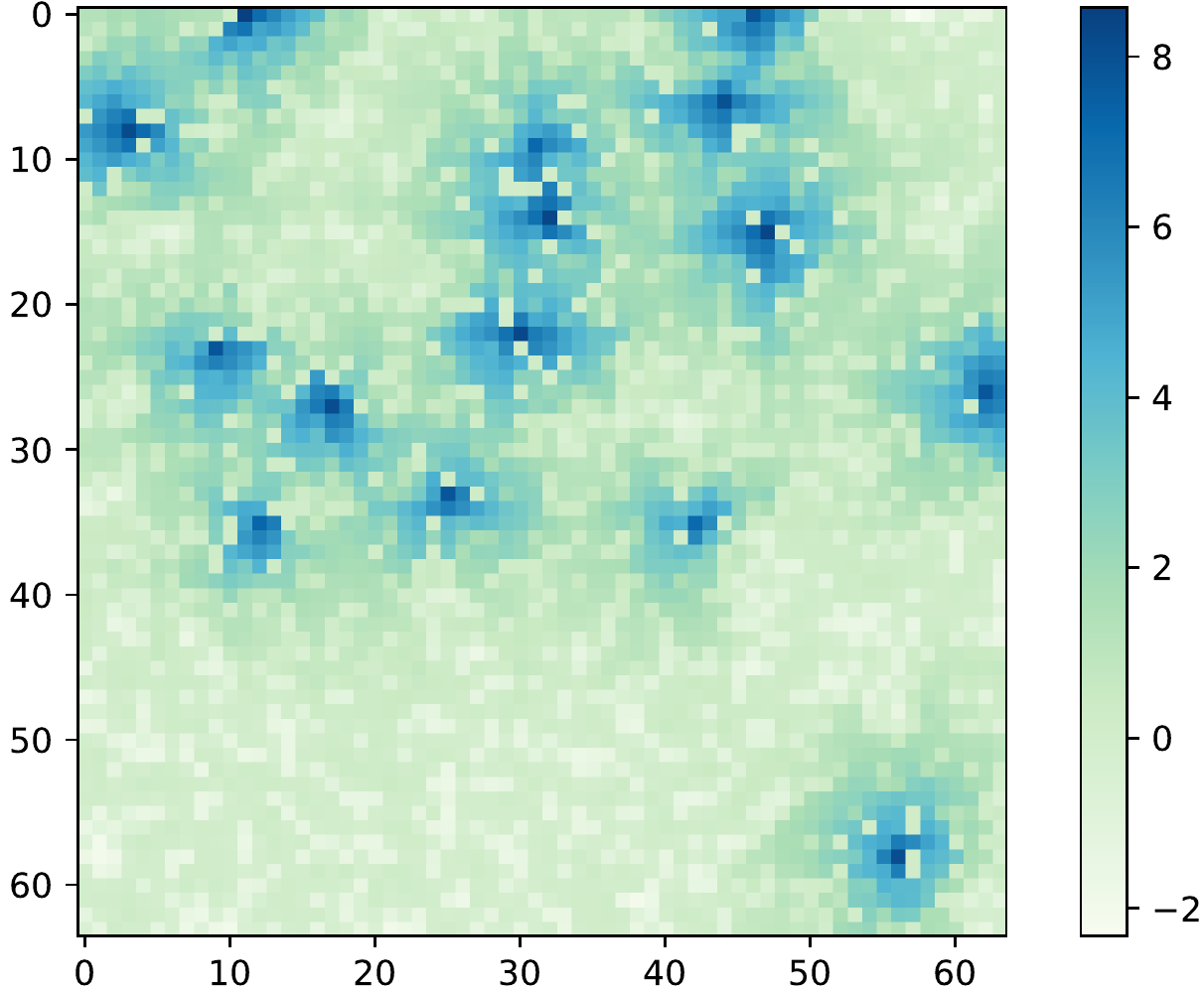}
    \caption{\textbf{Potential Underfitting in RL}. Left: gridworld structure. Red squares are lava, green are goals, gray are walls, and white squares are empty. Middle left: ground truth value function $V^*(s)$. Middle right: we fit $Q^*(s, a)$ with an MLP, then display $\max_a Q_\theta(s, a)$. The result is blurry and cannot properly distinguish between critical high and low value states. Right: we fit $Q^*(s, a)$ with our proposed learned Fourier feature architecture, then display $\max_a Q_\theta(s, a)$. It is able to exactly reproduce the ground truth, even with the same number of parameters and gradient steps. }
    \label{fig:gridworld}
\end{figure}
While the main paper focused on using LFF with small $\sigma$ to bias networks towards mainly learning low frequency functions, we can still use larger $\sigma$ to encourage faster high-frequency learning. Prior theoretical work \citep{basri2019convergence,rahaman2018spectral,xu2019frequency,farnia2018spectral} found that ReLU MLPs suffer from \textit{spectral bias} -- they can take an impractically long time to fit the highest frequencies. In the reinforcement learning setting, this can cause underfitting when fitting high frequencies is desirable. 

\subsection{Gridworld}
\label{sec:toy}
We demonstrate that underfitting is indeed happening in RL. We create a toy $64 \times 64$ gridworld task \citep{fu2019diagnosing} in Figure~\ref{fig:gridworld}, where each square can be one of five types: start state, goal state, empty, wall, or lava. The agent starts at the start state and receives +1 reward for every timestep at a goal state, -1 penalty for every timestep at a lava square, and 0 reward otherwise. It cannot enter cells with a wall. The agent has five actions available: up, down, left, right, or no-op. Whenever the agent takes a step, there is a 20\% chance of a random action being taken instead, so it is important for the agent to stay far from lava, lest it accidentally fall in. 
We use a discount of $\gamma = 0.9$. To create the environment, we randomly initialize it with 25\% lava squares and 10\% wall squares. We learn the optimal $Q^*$ using Q-iteration, then attempt to fit various neural network architectures with parameters $\theta$ to the ground truth $Q^*$ values through supervised learning: 
\begin{align}
    \theta^* = \argmin_\theta \sum_{s \in \mathcal S, a \in \mathcal A}(Q_\theta(s, a) - Q^*(s, a))^2
\end{align}
We then try to fit $Q^*$ using a standard MLP with 3 layers and 256 hidden units, and using our proposed architecture with an equivalent number of parameters and $\sigma=3$. Due to the challenge of visualizing $Q(s, a)$ with 5 actions, we instead show $V(s) = \max_a Q(s, a)$ in Figure~\ref{fig:gridworld}. 

Surprisingly, MLPs have extreme difficulty fitting $Q^*(s, a)$, even when doing supervised learning on this low-dimensional toy example. Even without the challenges of nonstationarity, bootstrapping, and exploration, the deep neural network has trouble learning the optimal Q-function. The learned Q-function blurs together nearby states, making it impossible for the agent to successfully navigate narrow corridors of lava. Prior work has described this problem as state aliasing~\citep{mccallum1996reinforcement}, where an agent conflates separate states in its representation space. We anticipate that this problem is worse in higher dimensional and continuous MDPs and is pervasive throughout reinforcement learning. 

In contrast, our LFF embedding with $\sigma=3$ is able to perfectly learn the ground-truth Q-function. This indicates that LFF with large $\sigma$ can help our Q-networks and policies fit key high-frequency details in certain RL settings. 
We believe that there are two promising applications for high-frequency learning with LFF: model-based RL and tabular problems. Model-based RL requires modeling the dynamics, which can have sharp changes, such as at contact points. Modeling transition dynamics is also supervised, so there are no bootstrap noise problems exacerbated by accelerating the rate at which high-frequencies are learned. Tabular problems are also suited for high-frequency learning, as they often have sharp changes in dynamics or rewards (e.g. gridworld squares with walls, cliffs, or lava). Capturing high-frequencies with LFF could improve both the sample-efficiency and asymptotic performance in this setting.

\section{Fourier Basis Variance After Training}
\label{sec:variance_after_training}
\begin{figure}[H]
    \centering
    \includegraphics[width=\textwidth]{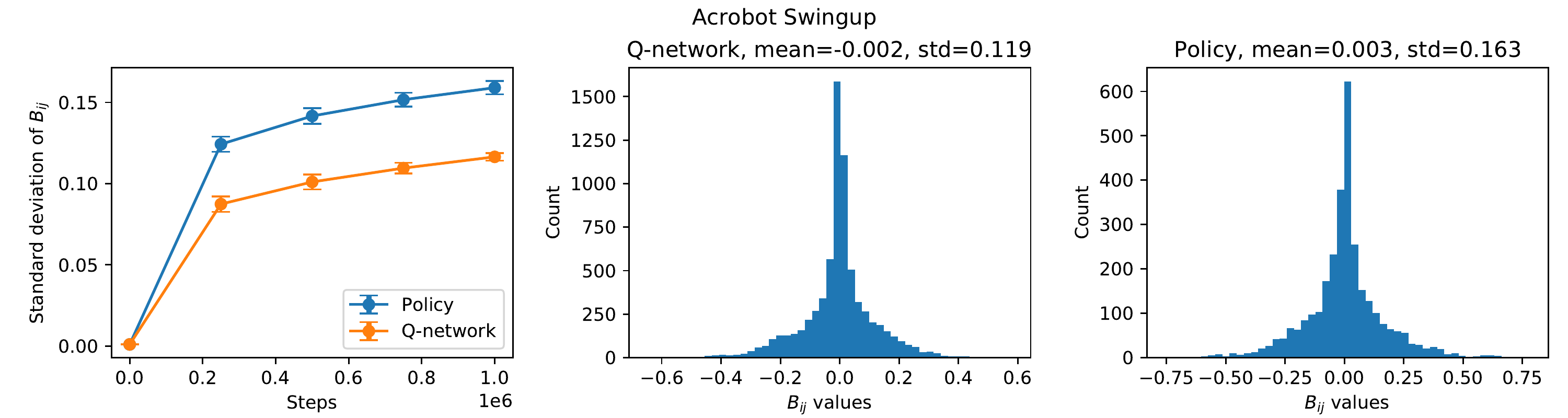} \\
    \includegraphics[width=\textwidth]{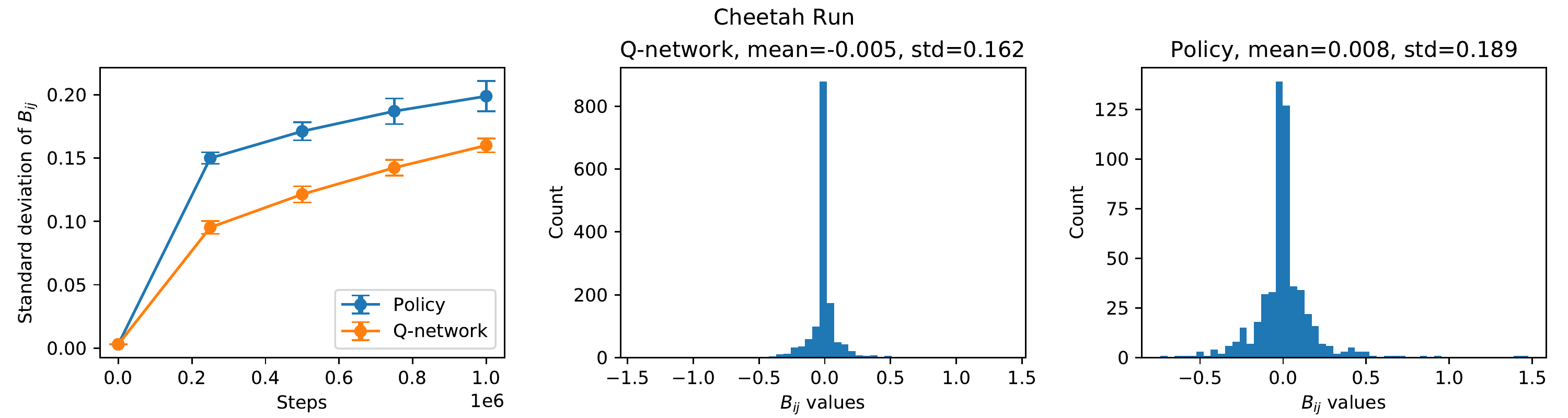} \\
    \includegraphics[width=\textwidth]{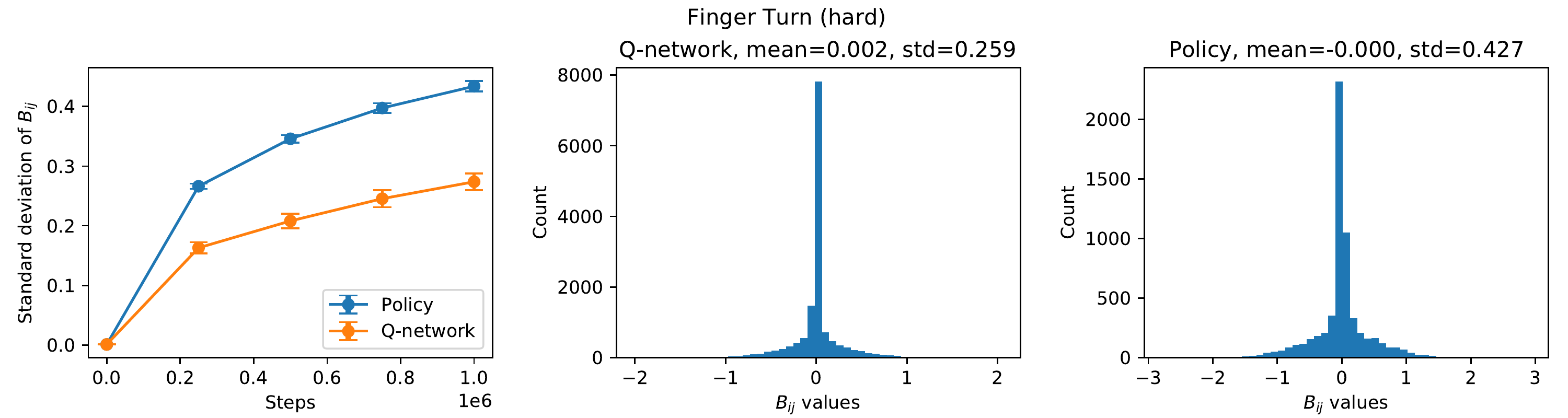} 
    \caption{\textbf{Change in the standard deviation of Fourier basis entries $B_{ij}$}. The curves on the left show how the standard deviation changes from $\sigma$ at initialization as state-based SAC training progresses. The confidence interval indicates the standard deviation of the standard deviation, measured across 5 seeds. The histograms in the middle and left show the distribution of entries within $B$ for the policy and Q-networks at the end of training. }
    \label{fig:sigma_change_a}
\end{figure}
In Section 4, we used the Neural Tangent Kernel (NTK) perspective to show that the initialization variance of the Fourier basis $B$ controls the per-frequency learning rate in infinite-width neural networks. However, the NTK's infinite-width assumption implies that the entries of the Fourier basis do not change over the course of training. Since we train $B$ with finite width, we examine how its variance evolves over training, which affects its per-frequency learning rate at each point in time. Figure~\ref{fig:sigma_change_a} and \ref{fig:sigma_change_b} show how the standard deviation of the Fourier basis change for the policy and Q-network. The standard deviation generally increases from $\sigma=0.001$ or $\sigma = 0.003$ to about 0.1, which should still be more biased towards low frequencies than vanilla MLPs are (see Figure \ref{fig:deeper_ntk}). Furthermore, the increase in standard deviation could actually be desirable. Having more data at the end of training could reduce the impact of bootstrap noise, so there may be less need for smoothing with small $\sigma$. Larger $\sigma$ could bias the network towards fitting medium-frequency signals that are important for achieving full asymptotic efficiency. This could explain the image-based results in Figure~\ref{fig:sigma_change_pixels}, where the standard deviation rises to about 0.5.
\begin{figure}
    \centering
    \includegraphics[width=\textwidth]{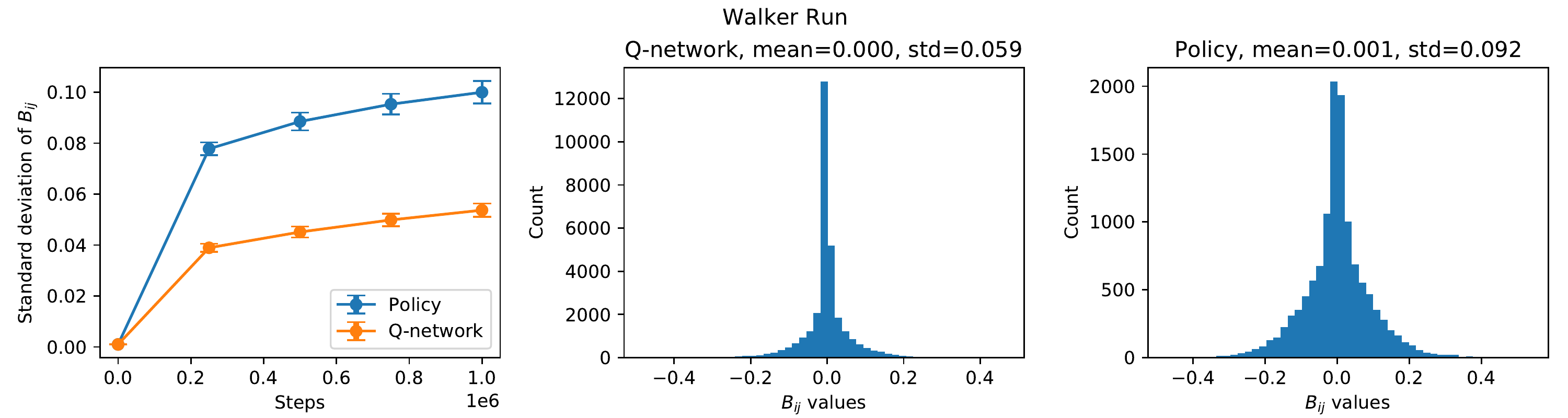} \\
     \includegraphics[width=\textwidth]{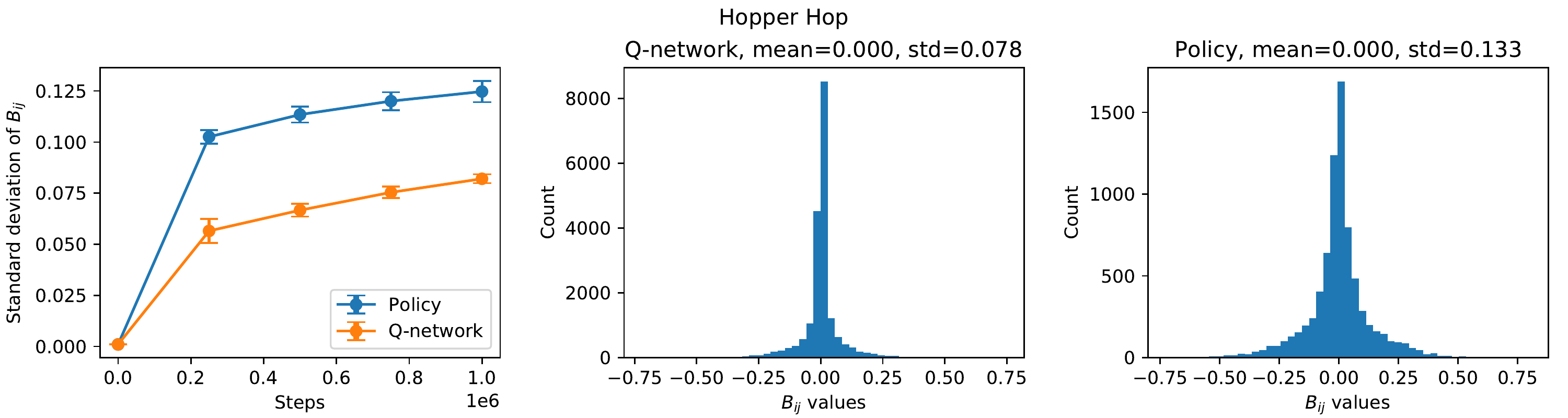} \\
    \includegraphics[width=\textwidth]{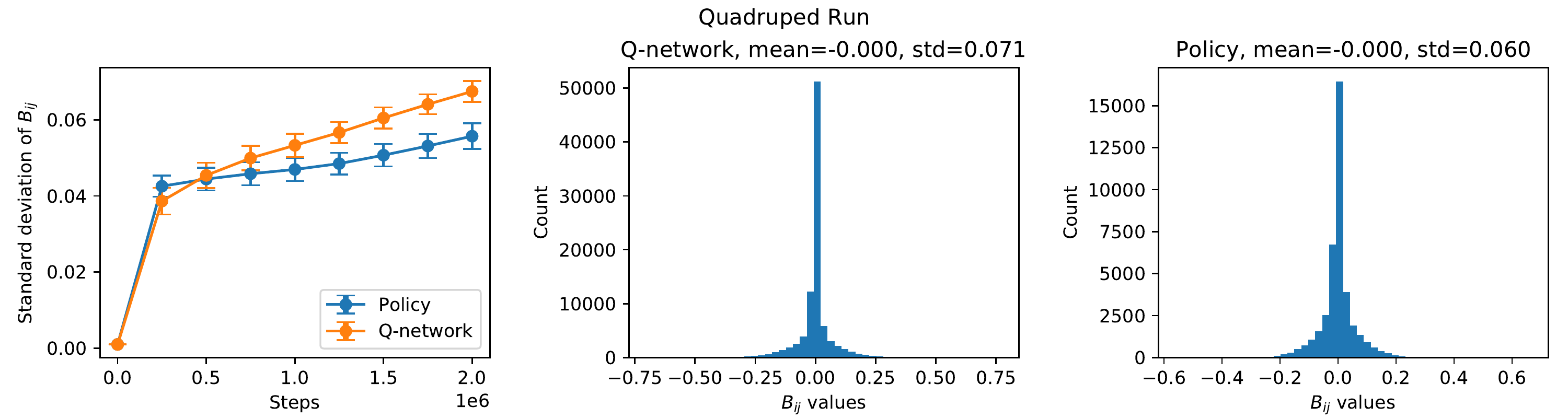} \\
    \includegraphics[width=\textwidth]{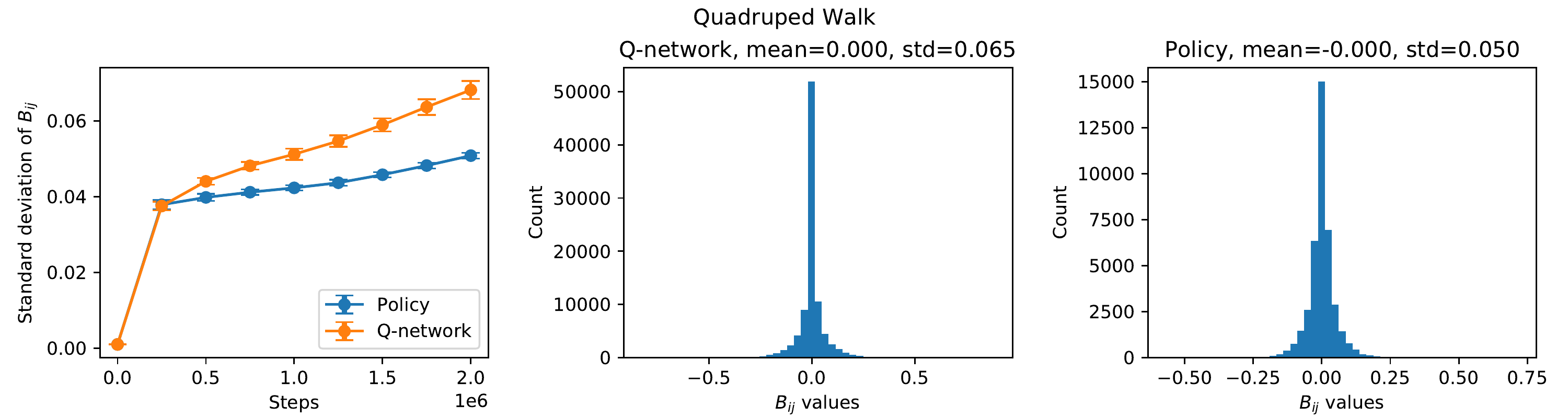} \\
    \includegraphics[width=\textwidth]{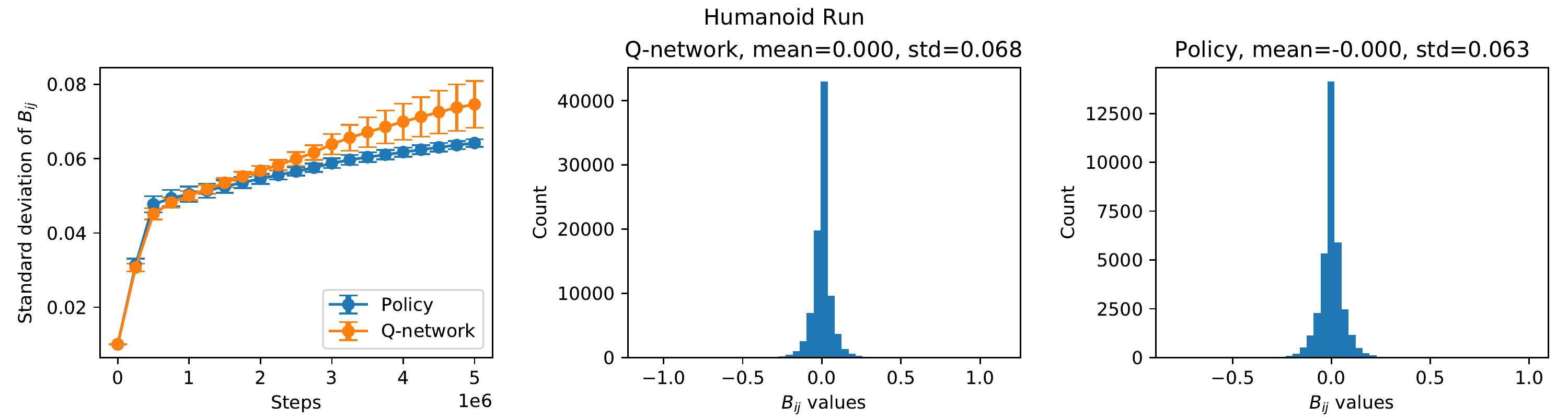} \\
    \caption{Continuation of Figure~\ref{fig:sigma_change_a}, which shows how the standard deviation of the Fourier basis changes over training.}
    \label{fig:sigma_change_b}
\end{figure}
\begin{figure}
    \centering
    \includegraphics[width=0.49\textwidth]{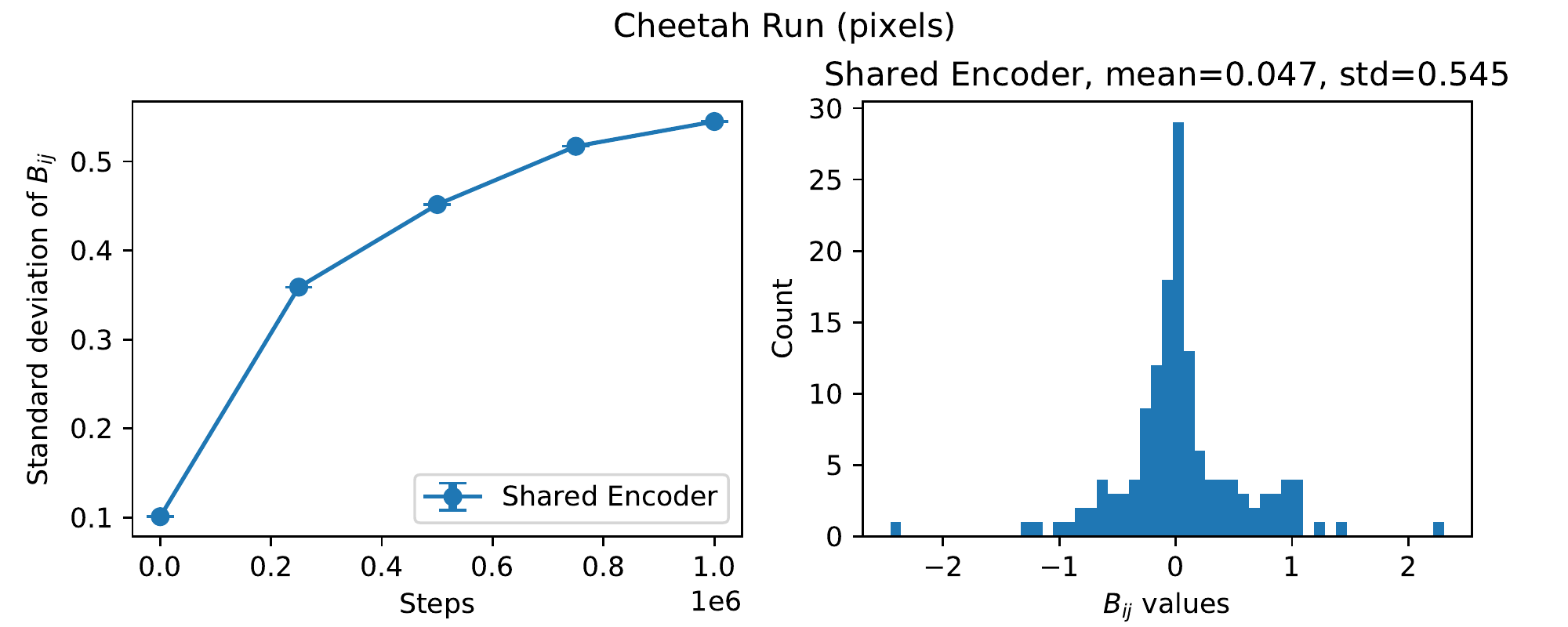}
    \includegraphics[width=0.49\textwidth]{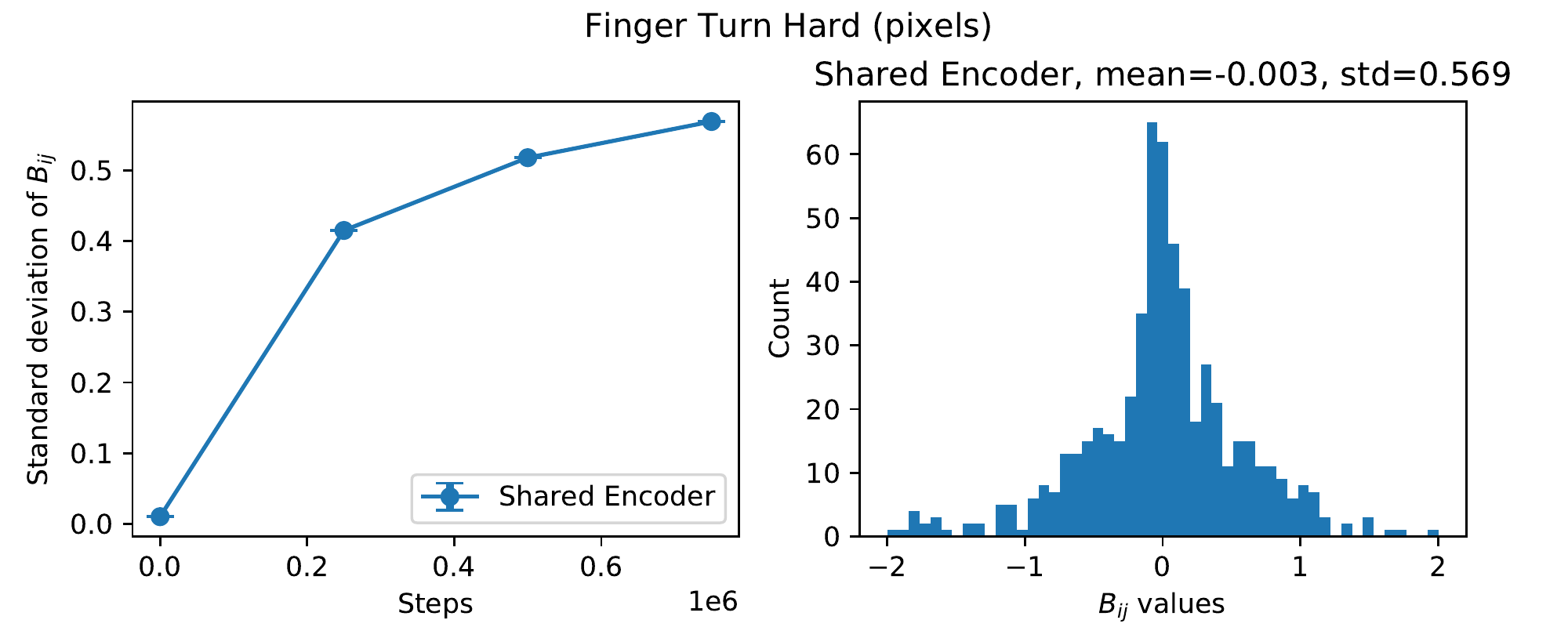} \\
    \includegraphics[width=0.49\textwidth]{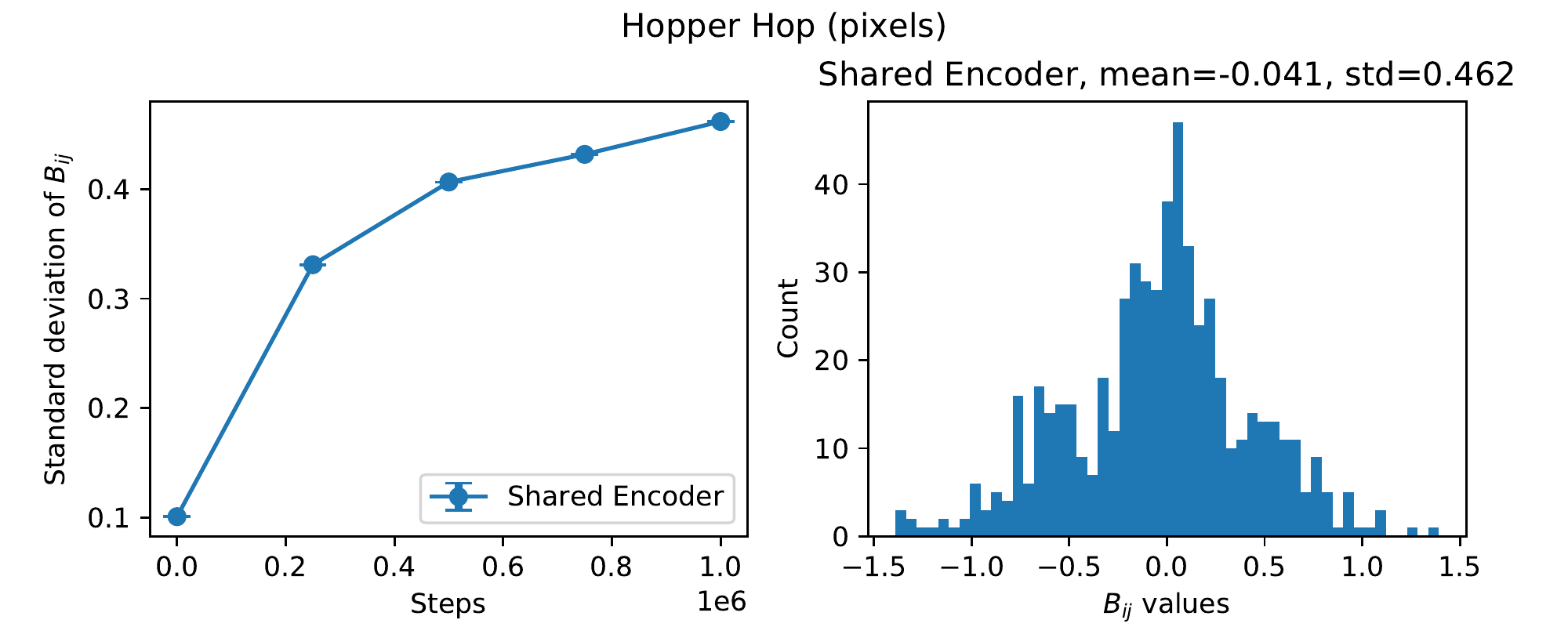}
    \includegraphics[width=0.49\textwidth]{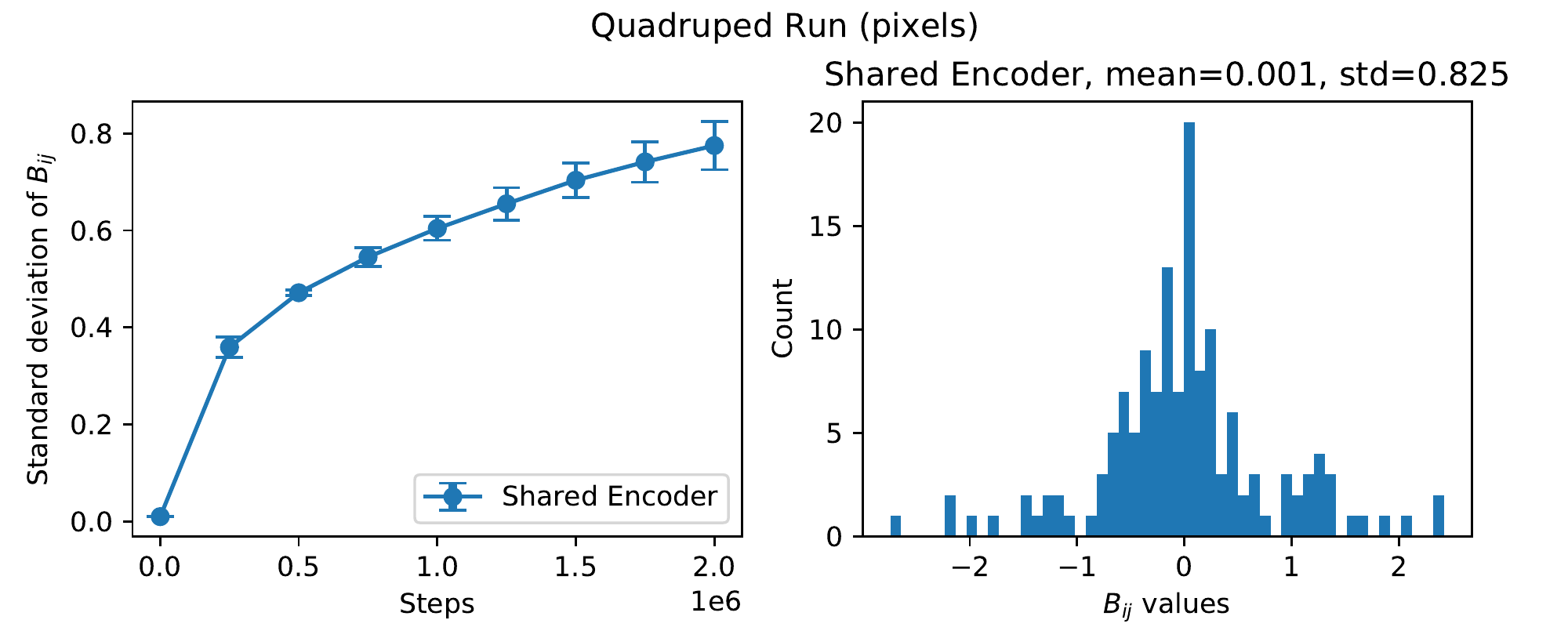} \\
    \caption{\textbf{Change in the standard deviation of Fourier basis entries $B_{ij}$ for image-based experiments}. For each of the four environments, we show how the standard deviation of the shared encoder's Fourier basis changes over training. The histogram shows the distribution of the entries of the shared encoder's $B$ at the end of training.}
    \label{fig:sigma_change_pixels}
\end{figure}

\newpage
\section{Hyperparameters}
\label{sec:hyperparameters}

We list hyperparameters for state-based and image-based SAC experiments in Table~\ref{tab:sac_hyperparams} and Table~\ref{tab:hyperparams_rad}, respectively. We also show hyperparameters for PPO experiments from Appendix~\ref{sec:on_policy} in Table~\ref{tab:hyperparams_ppo}.
\begin{table*}[!htbp]
  \begin{center}
    \begin{tabular}{ll}\toprule
      Parameter & Value\\
      \hline
      Algorithm & Soft Actor Critic \citep{haarnoja2018soft} \\
      Starting codebase & \citet{pytorch_sac} \\
      Optimizer & Adam \citep{kingma2014adam} \\
      Adam $(\beta_1, \beta_2)$ & (0.9, 0.999) \\
      Discount & 0.99 \\
      Batch size & 1024 \\
      Target smoothing coefficient  ($\tau$)  & 0.005\\
      Reward scale & Auto-tuned \citep{haarnoja2018softapplications} \\
      Actor learning rate & $10^{-4}$ \\
      Critic learning rate & $10^{-4}$ \\
      Reward scale learning rate & $10^{-4}$ \\
      Number of exploratory warmup steps & 5000 \\
      Number of hidden layers & 3 for MLP, 2 for LFF \\
      Hidden size & 1024 \\
      Hidden nonlinearity & ReLU \\
      Fourier dimension & 64 for Cheetah, 1024 otherwise \\
      Standard deviation $\sigma$ of Fourier basis initialization & 0.003 for Cheetah, 0.001 otherwise \\
    \bottomrule
    \end{tabular}
  \end{center}
\caption{Hyperparameters used for the state-basd SAC experiments.}
\label{tab:sac_hyperparams}
\end{table*}

\begin{table*}[!htbp]
  \begin{center}
    \begin{tabular}{ll}\toprule
      Parameter & Value\\
      \hline
      Algorithm & Soft Actor Critic \citep{haarnoja2018soft} + RAD \citep{laskin2020reinforcement}\\
      Starting codebase & \citet{laskin2020reinforcement} \\
      Augmentation & Translate: Cheetah. Crop: otherwise. \\
      Observation rendering & $(100, 100)$ \\
      Observation down/upsampling & Crop: $(84, 84)$. Translate: $(108, 108)$. \\
      Replay buffer size & 100000 \\
      Initial steps & 1000 \\
      Stacked frames & 3 \\
      Action repeat & 4 \\
      Optimizer & Adam \citep{kingma2014adam} \\
      Adam $(\beta_1, \beta_2)$ & $(0.5, 0.999)$ for entropy, (0.9, 0.999) otherwise \\
      Learning rate & $10^{-4}$ \\
      Batch size & 512 \\
      Encoder smoothing coefficient ($\tau$) & 0.05\\
      Q-network smoothing coefficient ($\tau$) & 0.01\\
      Critic target update freq & 2 \\
      Convolutional layers (excluding LFF embedding) & 4 for LFF, 5 otherwise \\
      Discount & 0.99 \\
      Fourier dimension & 128 for Hopper, 64 otherwise \\
      Initial standard deviation $\sigma$ of Fourier basis & 0.1 for Hopper, Cheetah; 0.01 for Finger, Quadruped \\
    \bottomrule
    \end{tabular}
  \end{center}
\caption{Hyperparameters used for the image-based SAC experiments.}
\label{tab:hyperparams_rad}
\end{table*}

\begin{table*}[!htbp]
  \begin{center}
    \begin{tabular}{ll}\toprule
      Parameter & Value\\
      \hline
      Algorithm & Proximal Policy Optimization \citep{schulman2017proximal} \\
      Learning rate & $3 \times 10^{-4}$ \\
      Learning rate decay & linear \\
      Entropy coefficient & 0 \\
      Value loss coefficient & 0.5 \\
      Clip parameter & 0.2 \\
      Environment steps per optimization loop & 2048 \\
      PPO epochs per optimization loop & 10 \\
      Batch size & 64 \\
      GAE $\lambda$ & 0.95 \\
      Discount & 0.99 \\
      Total timesteps & $10^{6}$ \\
      Number of hidden layers & 2 \\
      Hidden size & 256 \\
      Hidden nonlinearity & $\tanh$ \\
      Fourier dimension & 64 \\
      Variance of Fourier basis initialization $\tau$ & 0.01 \\
    \bottomrule
    \end{tabular}
  \end{center}
\caption{Hyperparameters used for the PPO experiments.}
\label{tab:hyperparams_ppo}
\end{table*}

\end{document}